\documentclass[journal,comsoc]{article}
\usepackage[T1]{fontenc}% optional T1 font encoding
\usepackage[a4paper,bindingoffset=0.2in,%
            left=1in,right=1in,top=1in,bottom=1in,%
            footskip=.25in]{geometry}
\usepackage{multirow}

\usepackage{algorithm,algorithmic,hyperref}

\usepackage{amssymb,bm,amsmath,amsthm,amsfonts,mathrsfs}
\newtheorem{thm}{Theorem}[section]
\newtheorem{lemma}[thm]{Lemma}
\newtheorem{definition}{Definition}

\newtheorem{cor}{Corollary}

\newcommand{\mbR}{\mathbb{R}}

\newcommand{\tr}{\rm{Tr}}

\newcommand{\mcN}{\mathcal{N}}

\newcommand{\mcS}{\mathcal{S}}

\newcommand{\card}{{\rm card}}

\newcommand{\mbE}{\mathbb{E}}

\usepackage{graphics,graphicx,subcaption,color}
\usepackage[cmintegrals]{newtxmath}
\usepackage{bm}

\begin{document}
\markboth{IEEE Transactions on Information Theory, ~Vol.~XXX, No.~XXX, April~2020}%
{Shell \MakeLowercase{\textit{et al.}}: Bare Demo of IEEEtran.cls for IEEE Communications Society Journals}
% The only time the second header will appear is for the odd numbered pages
% after the title page when using the twoside option.
%
% *** Note that you probably will NOT want to include the author's ***
% *** name in the headers of peer review papers.                   ***
% You can use \ifCLASSOPTIONpeerreview for conditional compilation here if
% you desire.

% If you want to put a publisher's ID mark on the page you can do it like
% this:
%\IEEEpubid{0000--0000/00\$00.00~\copyright~2015 IEEE}
% Remember, if you use this you must call \IEEEpubidadjcol in the second
% column for its text to clear the IEEEpubid mark.

% use for special paper notices
%\IEEEspecialpapernotice{(Invited Paper)}

\title{Mutual Information Learned Classifiers: an Information-theoretic Viewpoint of Training Deep Learning Classification Systems}
\author{Jirong Yi,~Qiaosheng Zhang,~Zhen Chen\\
	~Qiao Liu,~Wei Shao
	\thanks{Jirong Yi is with CAD Science Group, Hologic Inc, Santa Clara, CA 95054 and Department of Electrical and Computer Engineering, University of Iowa, Iowa City, IA 52242. Qiaosheng Zhang is with Department of Electrical and Computer Engineering at National University of Singapore, Singapore 119077. Zhen Chen is with Department of Electrical Engineering and Computer Science at University of California at Irvine, Irvine, CA 92697. Qiao Liu is with Department of Statistics at Stanford University, Stanford, CA 94305. Wei Shao is with Department of Radiology at Stanford University, Stanford, CA 94305. Emails: jirong.yi@hologic.com, ericzhang8951@gmail.com, zhenc4@uci.edu, liuqiao@stanford.edu, weishao@stanford.edu. Corresponding emails should be sent to: jirong.yi@hologic.com, jirong-yi@uiowa.edu.}}
%\date{\today}

% make the title area
\maketitle

% As a general rule, do not put math, special symbols or citations
% in the abstract or keywords.
\begin{abstract}
	Deep learning systems have been reported to acheive state-of-the-art performances in many applications, and one of the keys for achieving this is the existence of well trained classifiers on benchmark datasets which can be used as backbone feature extractors in downstream tasks. As a main-stream loss function for training deep neural network (DNN) classifiers, the cross entropy loss can easily lead us to find models which demonstrate severe overfitting behavior when no other techniques are used for alleviating it such as data augmentation. In this paper, we prove that the existing cross entropy loss minimization for training DNN classifiers essentially learns the conditional entropy of the underlying data distribution of the dataset, i.e., the information or uncertainty remained in the labels after revealing the input. In this paper, we propose a mutual information learning framework where we train DNN classifiers via learning the mutual information between the label and input. Theoretically, we give the population error probability lower bound in terms of the mutual information. In addition, we derive the mutual information lower and upper bounds for a concrete binary classification data model in $\mbR^n$, and also the error probability lower bound in this scenario.  Besides, we establish the sample complexity for accurately learning the mutual information from empirical data samples drawn from the underlying data distribution. Empirically, we conduct extensive experiments on several benchmark datasets to support our theory. Without whistles and bells, the proposed mutual information learned classifiers (MILCs) acheive far better generalization performances than the state-of-the-art classifiers with an improvement which can exceed more than 10\% in testing accuracy.
	%\footnote{To faciliate research, we will release our implementation at \underline{\href{https://github.com/JAMES-YI/MILC}{MILC}} in github.} 

\end{abstract}

%\newpage
%\tableofcontents
%\newpage

% For peer review papers, you can put extra information on the cover
% page as needed:
% \ifCLASSOPTIONpeerreview
% \begin{center} \bfseries EDICS Category: 3-BBND \end{center}
% \fi
%
% For peerreview papers, this IEEEtran command inserts a page break and
% creates the second title. It will be ignored for other modes.
\maketitle

\section{Introduction}\label{Sec:Introduction}

Ever since the breakthrough made by Krizhevsky et al. \cite{krizhevsky_imagenet_2012}, deep learning has been finding trenmdous applications in different areas such as computer vision, natural language process, and traditonal signal processing \cite{goodfellow_deep_2016,bora_compressed_2017,yi_outlier_2018}, and it achieved the state-of-the-art performances in almost all of them. In nearly all of these applications, a fundamental classification task is usually involved, i.e., determining the class to which a given input belongs. Examples from computer vision include the image-level classification in recognition tasks, the patch-level classification in object detection, and pixel-level classification in image segmentation tasks \cite{zheng_rethinking_2021,yi_trust_2019,ren_faster_2016}. For many other applications which do not directly or explicitly involve classifications, they still use model pretrained via classification tasks as a backbone for extracting useful and meaningful representation for the specific tasks \cite{liu_swin_2021,he_deep_2015}. In practice, such extracted representations have been reported to be beneficial for the downstream tasks \cite{liu_swin_2021,ren_faster_2016,he_deep_2015}.

To train such classifiers, the deep learning community has been mainly using the cross entropy loss or its variants as the objective function for guiding the search of a good set of model weights \cite{goodfellow_deep_2016}. However, the models trained this way can easily overfit the data and result in pretty bad generalization performance, and this motivates the proposal of many techniques for improved generalization. These techniques can be broadly divided into several categories. From the propsect of data, increasing the dataset size has been proving to be beneficial for better generalization, but collecting huge amount of data can be labor-consuming and costly. For example, in medical image analysis and diagnostics, collecting the dataset can cost millions of dollars, and the data size is usually very small for some rare disease such as cancer \cite{li_breast_2017}. The data augmentation is another commonly used technique to increase the diversity of dataset such as random cropping, flipping, and color jittering \cite{he_deep_2015,szegedy_going_2015,szegedy_rethinking_2016}. However, in situations where the dataset itself is scarce, the data augmentation may not be enough for training a well-performing classifier such as few shot learning \cite{boudiaf_information_2020}.

From the angle of models, traditional machine learning theory shows that decreasing the flexibility or complexity of the models can help alleviate the overfitting phenomenon \cite{shalev-shwartz_understanding_2014,bishop_pattern_2006,theodoridis_machine_2015,mohri_foundations_2018}. However, under the background of deep learning, this does not seem to be a feasible solution because the models which achieve the state-of-the-art (SOTA) performance are becoming increasingly more complex with parameters even over one trillion \cite{devlin_bert_2019,fedus_switch_2021,liu_swin_2021}. These huge models are motivated by the increasingly more challenging learning problems which require the strong capability of huge models to extract useful information and find meaningful patterns that can be used for solving them, and which the smaller models are incapable of \cite{fedus_switch_2021}.

Another line of works for improving the generalization performance comes from the regularization viewpoint, i.e., restricting the model space for searching during training to avoid overfitting \cite{goodfellow_deep_2016}. Examples includes the weight decay (or $\ell_2$ regularization), $\ell_1$ regularization, and label smoothing regularization \cite{szegedy_rethinking_2016}. The  major limitations of these approaches are that they require prior knowledge about the learning tasks. For example, the $\ell_1$ regularization usually requires the ground truth model to have sparse weights while the $\ell_2$ regularization requires the ground truth model to have small magnitude to achieve good generalization performance. Unfortunately, such prior knowledge is not always available in practice. What makes things worse is that it is recently reported that such prior knowledge may make the learned models adversarially vulnerable such that adversarial attacks can be easily acheived, and this is because the model can be underfitted to those unseen adversarial examples \cite{xu_adversarial_2019,xu_information-theoretic_2017,zhang_towards_2020}.

\subsection{Ignored Conditional Label Entropy}

In this paper, we show that the existing cross entropy loss minimization for training deep neural network classifiers essentially learns the conditional entropy of the underlying data distribution of the input and the label. We argue that this can be the fundamental reason which accounts for the severe overfitting of models trained by cross entropy loss minimization, and the extremely small training loss in practice implies that the learned model completely ignores the conditional entropy about the label distribution. The reasons for such ignorance of the conditional entropy include that the marginal input entropy is very big when compared with the conditional label entropy, and that the annotating of input samples in the data collection process and the encoding of the labels during the training both ignore the conditional entropy of label. 

To see these, we consider the MNIST image recognition task from computer vision where we want to train a model to predict which class from 10 classes a given digit image belongs to\footnote{http://yann.lecun.com/exdb/mnist/}. In this example, the label has only 10 choices, and the maximum label entropy is $\log_2(10)\approx3.32$ bits. However, since the image input is in $\mbR^{784}$, the maximum entropy of the image distribution can be very big, i.e., $\log_2\left( 256^{784} \right)=6272$ bits if we assume each pixel to take value in $\{0,1,\cdots,255\}$ uniformly. The gap between the information contained in the input image distribution and that contained in the label distribution is so big that when the model is trained to learn the conditional entropy of the label distribution after revealing the image input, it has a strong tendency to simply ignore such remaining information and directly treat it as zero. This is indeed the case in many machine learning applications. In the MNIST classification task, when the image of the hand-written digital is given, we are usually 100\% sure which class the image belongs to. In Figure \ref{Fig:MNISTImgsSure1}, we show one of such images, and there is no doubt that the digit is 1, thus the label entropy is 0 when this image is given. However, this is not always case because we can have image samples whose classes cannot determined with complete certainty. Some such examples are also presented in Figure \ref{Fig:MNISTImgs}, the digit in Figure \ref{Fig:MNISTImgsAmbug1} has a truth label 1, but it looks like 2. Similarly, the digit in Figure \ref{Fig:MNISTImgsAmbug4} has a truth label 4 but looks like 9, and Figure \ref{Fig:MNISTImgsAmbug9} has a truth label 9 but looks like 4. Different people can have different labels for these image samples, but their ground truth annotations or labels are at the discretion of the creator of them. In more complex image classification tasks such as ImageNet classification, a single image itself can contain multiple objects, and thus belong to multiple classes. However, it has only single annotation or label which depends on the discretion of the human annotators \cite{deng_imagenet:_2009}. In Figure \ref{Fig:ImageNetImgs}, we show image examples from the ImageNet-1K classification task where the goal is to classify a given image into 1000 classes. Though Figure \ref{Fig:ImageNetImgsMultiple} contains also a pencil, the human annotator only labeled it with cauliflower label. In deep learning practice, since we usually use the one-hot encoding of the label for a given image, i.e., assigning all the probability mass to the annotated class while zero to all the other classes, this further encourages the model to ignore the conditional information of the label \cite{he_deep_2015,szegedy_rethinking_2016,yi_trust_2019}. The ignorance of label entropy allows the classifiers to give over-confident label predictions, resulting in unsatisfactory generalization performance.

\begin{figure*}[!htb]
	\centering
	\begin{subfigure}[b]{0.24\textwidth}
		\centering
		\includegraphics[width=\linewidth]{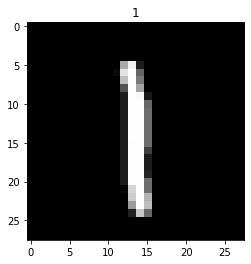}
		\caption{Truth label is 1}\label{Fig:MNISTImgsSure1}
	\end{subfigure}
~
	\begin{subfigure}[b]{0.24\textwidth}
	\centering
	\includegraphics[width=\linewidth]{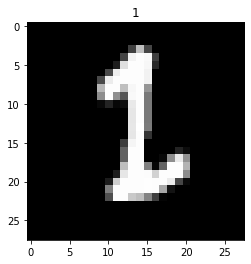}
	\caption{Truth label is 1}\label{Fig:MNISTImgsAmbug1}
\end{subfigure}%
	\begin{subfigure}[b]{0.24\textwidth}
	\centering
	\includegraphics[width=\linewidth]{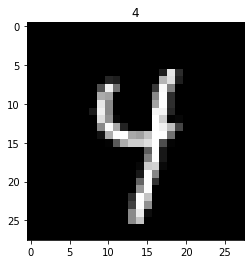}
	\caption{Truth label is 4}\label{Fig:MNISTImgsAmbug4}
\end{subfigure}%
~
\begin{subfigure}[b]{0.24\textwidth}
	\centering
	\includegraphics[width=\linewidth]{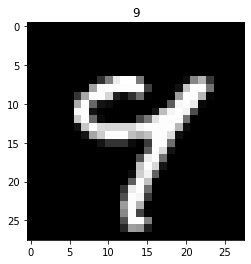}
	\caption{Truth label is 9}\label{Fig:MNISTImgsAmbug9}
\end{subfigure}
	\caption{Image examples from MNIST dataset.}\label{Fig:MNISTImgs}
\end{figure*}

\begin{figure*}[!htb]
	\centering
	\begin{subfigure}[b]{0.45\textwidth}
		\centering
		\includegraphics[width=\linewidth]{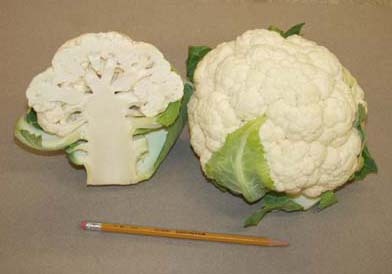}
		\caption{Truth label is cauliflower}\label{Fig:ImageNetImgsMultiple}
	\end{subfigure}%
	~
	\begin{subfigure}[b]{0.45\textwidth}
		\centering
		\includegraphics[width=\linewidth]{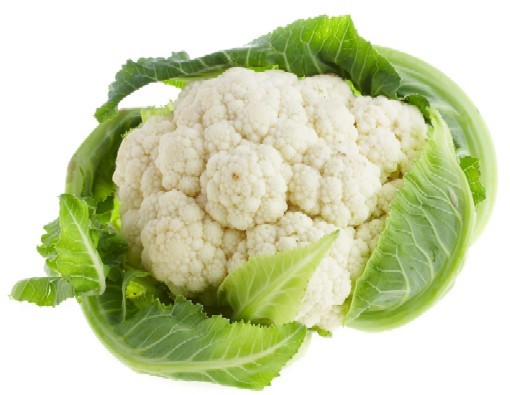}
		\caption{Truth label is cauliflower}\label{Fig:ImageNetImgsSingle}
	\end{subfigure}
	\caption{Image examples from ImageNet dataset. Though Figure \ref{Fig:ImageNetImgsMultiple} contains also a pencil, the human annotator only labeled it with cauliflower. In deep learning practice, since we usually use the one-hot encoding of the label for a given image, i.e., assigning all the probability mass to the annotated class while zero to all the other classes, this further encourages the model to ignore the conditional entropy of the label \cite{he_deep_2015,szegedy_rethinking_2016,yi_trust_2019}.}\label{Fig:ImageNetImgs}
\end{figure*}

Based on the above observations, a naive way for improving the generalization performance can be getting back the conditional entropy of label, e.g., giving multiple annotations for a single image if it contains multiple objects during the data generation process, and using other types of label encoding instead of one-hot encoding during the training process. The deep learning community seems to also realize the limitations of dataset with single annotation for images containing multiple objects, thus created the ImageNet ReaL benchmark in 2020 which is more than 10 years after the construction of the original ImageNet dataset \cite{deng_imagenet:_2009,beyer_are_2020}. However, these re-assessed labels only partially fixed the conditional information loss of labels, and can still suffer the conditional information loss when an object itself has uncertainty as we show in Figure \ref{Fig:MNISTImgs}. Besides, annotating each object in an image can be extremely labor-consuming, or even impossible in cases where some objects are so small that they can hardly be perceivable \cite{lee_interactive_2022}. 

As for using different label encoding methods instead of the one-hot encoding, all of them assume implicitly that the annotations are reasonably good \cite{szegedy_rethinking_2016,he_deep_2015,meister_generalized_2020}. Examples of other label encodings include the label-smoothing regularization (LSR), generalized entropy regularization (GER), and so on \cite{szegedy_rethinking_2016,meister_generalized_2020}. In LSR, Szegedy et al. replaced the one-hot encoding of labels in the cross entropy loss minimization with a mixture of the original one-hot distribution and a uniform distribution, and this mixture encoding is obtained by taking a small probability mass from the annotated class and then evenly spreading it over all the other classes \cite{szegedy_rethinking_2016}. In GER, Meister et al. proposed to use a skew-Jensen divergence to encourage the learned conditional distribution to approximate the mixture encoding of label, and they add this extra divergence term as a regularization to the original cross entropy loss minimization \cite{meister_generalized_2020}. However, these efforts can still have severe limitations. First of all, the construction of mixture encoding of label can be quite biased due to the similar reasons accounting for the conditional entropy loss in dataset construction process. Secondly, their focus is still on improving the conditional entropy, and this can be very challenging since the gap between the conditional entropy of the label and the differential entropy of the input is so big that after we reveal the the input, the conditional entropy of the label distribution can be very small and hard to learn. 

\subsection{Mutual Information Learned Classifiers}

In this paper, we propose a new learning framework, i.e., mutual information learning (MIL) where we train classifiers via learning the mutual information of the dataset, i.e., the dependency between the input and the label, and this is motivated by several observations.

First of all, we show that the existing cross entropy loss minimization for training DNN classifiers actualy learns the conditional entropy of the label when the input is given. From an information theoretic viewpoint, the mutual information between the input and the label quantifies the information shared by them while the conditional entropy quantifies the information remained in the label after revealing the input. Compared with the conditional entropy, the magnitude of the mutual information can be much larger, and it may not be easily ignored by the model during training, thus possible to alleviate the overfitting phenomenon. An illustration of the relation among different information quantities involved in a dataset is shown in Figure \ref{Fig:DatasetInformationAmount}. 

In addition, in 2020, there are several works which apply information theoretic tools to investigate DNN models \cite{wang_robust_2021,yi_derivation_2020,yi_towards_2021}. In \cite{yi_derivation_2020}, Yi et al. investigated the adversarial attack problems from an information theoretic viewpoint, and they proposed to acheive adversarial attacks by minimizing the mutual information between the input and the label. Yi et al. also established theoretical results for characterizing what the best an adversary can do for attacking machine learning models \cite{yi_derivation_2020}. Concurrently, in \cite{wang_robust_2021}, Wang et al. used information theoretic tools to derive interesting relations between the existing adversarial training formulations and the conditional entropy optimizations. These works imply that there are intrinsinc connections between the properties of the model learned from the dataset and the information contained in the dataset. Last but not least, some recent works on training DNN reported that increasing entropy about the labels can improve the generalization performance, and even make the model more adversarially robust \cite{szegedy_rethinking_2016,meister_generalized_2020,papernot_distillation_2016}. Though these results imply that the overfitting can be due to the severe ignorance of the conditional entropy of the label during training, we argue that a more appropriate quantity for guiding the learning and training of classification systems can be the mutual information (MI) which better characterizes the dependency between the input and the label. Besides, the MI usually has larger magnitude than the conditional label entropy in classification tasks, making it less to be ignored by the model and affected by the numerical precisions. To see this, we can consider an extreme case where the model gives uniform distribution for the label when an arbitrary example is fed to the model. In this case, the conditional entropy of the label acheives the maximum, but what the model learned can be meaningless since it cannot accurately characterize the dependency between the input and the label. 

\begin{figure}[!htb]
	\centering
	\includegraphics[width=0.5\linewidth]{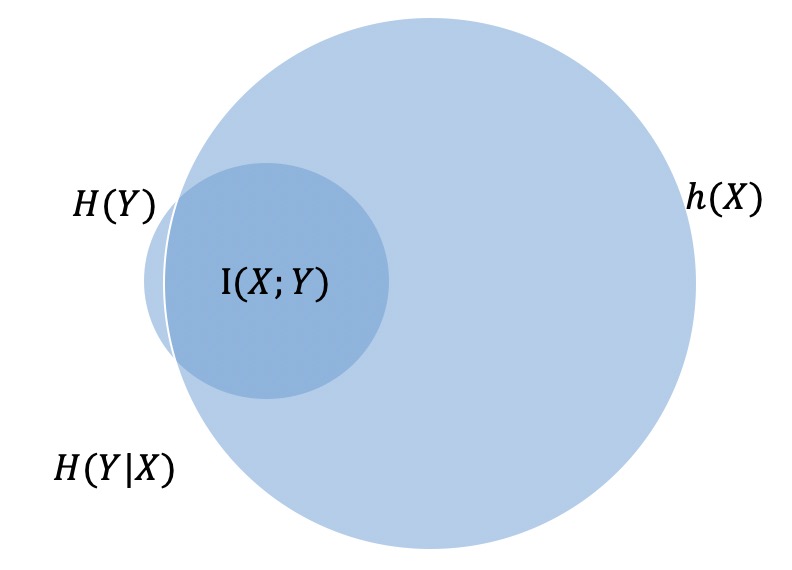}
	\caption{Information quantities involved in joint data distribution $p_{X,Y}$. The big circle represents the differential entropy $h(X)$ or information contained in input random variable $X$, and the small circle represents the entropy $H(Y)$ or information contained in the label random variable $Y$. The overlapped area corresponds to the mutual information or information shared by the $X$ and the $Y$, and the area in the small circle after excluding the overlapped area corresponds to the information remained in the $Y$ after revealing the $X$.}\label{Fig:DatasetInformationAmount}
\end{figure}

Under our mutual information learning (MIL) framework, we design a new loss for training the DNN classifiers, and the loss itself originates from a representation of mutual information of the dataset generating distribution, and we propose new pipelines associated with the proposed framework. We will refer to this loss as mutual information learning loss (milLoss), and the traditional cross entropy loss as conditional entropy learning loss (celLoss) since it essentially learning the conditional entropy of the dataset generating distribution. When reformulated as a regularized form of the celLoss, the milLoss can be interpreted as a weighted sum of the conditional label entropy loss and the label entropy loss. In the regularized form, the milLoss encourages the model not only to accurately learn the conditional entropy of the label when an input is given, but also to precisely learn the entropy of the label. This is distinctly different from the label smoothing regularization (LSR), confidence penalty (CP), label correction (LC) etc which consider the conditional entropy of the label \cite{szegedy_going_2015,wang_proselflc_2021,meister_generalized_2020}.

For the proposed MIL framework, we establish an error probability lower bound for arbitrary classification models in terms of the mutual information (MI) associated with the data generating distribution by using Fano's inequality \cite{cover_elements_2012} and an upper bound of error probability entropy developed by Yi et al. \cite{yi_derivation_2020}. These bounds explicitly characterize how the performance of the classification models trained from a dataset is connected to the mutual information contained in it, i.e., the dependency between the input and the output. Compared to Fano's inequality,  our bound is tighter due to a carefully designed relaxation. Our error probability bound is applicable for arbitrary distribution and arbitrary learning algorithms. We also consider a concrete binary classification problem in $\mbR^n$, and derive both lower and an upper bounds of the mutual information associated with the data distribution. Besides, we derive an error probability bound for this binary classification data model. We also establish theoretical guarantees for training models to accurately learn the mutual information associated with arbitrary dataset generating distribution, and we give the sample complexity for achieving this in practice. The keys for establishing these are the universal approximation properties of neural networks and the concentration of measure phenomenon from statistics \cite{boucheron_concentration_2013,hornik_multilyaer_1989,belghazi_mine_2018,boucheron_concentration_2013,shalev-shwartz_understanding_2014,mohri_foundations_2018}. We conduct extensive experiments to validate our theoretical analysis by using classification tasks on benchmark dataset such as MNIST and CIFAR-10\footnote{https://www.cs.toronto.edu/~kriz/cifar.html}. The empirical results show that the proposed MIL framework can achieve far superior generalization performance than the existing conditional entropy learning approach and its variants \cite{szegedy_going_2015,pereyra_regularizing_2017,wang_proselflc_2021}. 

\subsection{Related Works}

Our work is highly related to the following several works, but there are distinct differentiations between our work and them \cite{yi_trust_2019,yi_derivation_2020,yi_towards_2021}. First of all, in 2019, Yi et al. formulated the classification problem under the encoding-encoding paradigm by assuming there is an observation synthesis process which can generate observations or inputs for a given label, and the classification task is simply about inferring the label from the observation. Under this framework, they give theoretical characterizations of the robustness of machine learning models to different types of perturbations \cite{yi_trust_2019}. They also characterized the limit of an arbitrary adversarial attacking algorithm for an arbitrary machine learning system for answering the question of what is the best attack that an adversary can acheive and what the optimal adversarial attacks look like \cite{yi_derivation_2020}. We continue to investigate the classification tasks using encoding-decoding paradigm. Though the works by Yi et al., are the major motivations for this work, the goal of this work is completely different.  We investigate the learning of  classification models without presence of adversaries, and the connection between the models' generalization performance and the mutual information of the dataset generating distribution \cite{yi_trust_2019,yi_derivation_2020,yi_towards_2021}. 

Our work is also highly related to that by \cite{mcallester_formal_2020} where the authors proposed a difference-of-entropy (doe) formulation for estimating the mutual information of a distribution from empirical observations sampled from it \cite{mcallester_formal_2020}. In their formulation, two different neural networks are trained jointly to learn the conditional entropy and the entropy, respectively. In this paper, our goal is to train DNN classifiers with good generalization performances rather than estimating the mutual information. We use a formulation similar to the doe in \cite{mcallester_formal_2020}, but we also consider the scenario where we use a single neural network to learn both of them, and the new scenario follows the weight sharing ideas in deep learning practice \cite{goodfellow_deep_2016}. Besides, we prove that the existing cross entropy loss minimization approach is essentially learning the conditional entropy, and establish error probability lower bound in terms of the mutual information. In addition, we give sample complexity for the accurate learning of the mutual information from empirical samples, and derived the lower bound and upper bounds of the mutual information in a binary classification data model. 

Our work is also related to \cite{xu_information-theoretic_2017} where Xu and Raginsky investigated the generalization performance from an information-theoretic viewpoint, and they derived upper bounds for generalization error (essentially equivalent to error probability) in terms of mutual information between the dataset and the model set (the model is also assumed to follow a distribution). Though we also derived the error porbability lower bound in terms of the mutual information, the mutual information we consider in this paper is associated with the dataset only, i.e., the mutual information between the data input and the data label while the mutual information considered by Xu and Raginsky is from the joint distribution of data input, data output, and the model itself \cite{xu_information-theoretic_2017}.

\begin{table}[h]
	\centering
	\small
	\begin{tabular}{|l | c |}
		\hline 
		Loss obj. & Formula \\
		\hline 
		celLoss & $H(P_{Y|X}, Q_{Y|X})$\\
		\hline 
		celLoss$+$LSR & $(1-\epsilon)H(P_{Y|X}, Q_{Y|X}) + \epsilon H(U_{Y|X}, Q_{Y|X})$\\
		\hline celLoss$+$CP & $(1-\epsilon)H(P_{Y|X}, Q_{Y|X}) - \epsilon H(Q_{Y|X}, Q_{Y|X})$\\
		\hline 
		celLoss$+$LC & $(1-\epsilon)H(P_{Y|X}, Q_{Y|X}) + \epsilon H(Q_{Y|X}, Q_{Y|X})$\\
		\hline 
		milLoss (proposed) & $H(P_{Y|X}, Q_{Y|X}) + \lambda_{ent} H(P_{Y}, Q_{Y})$\\
		\hline 
	\end{tabular}
	\caption{Comparisons among celLoss, LSR+celLoss, CP+celLoss, LC+celLoss, and regularized form of milLoss. The $P_{Y|X}$ and $P_Y$ are the conditional and marginal label distribution, respectively. The $Q_{Y|X}$ and the $U_{Y|X}$ are the predicted conditional label distribution and the uniform conditional label distribution.}\label{Tab:LossFunctionComparisons}
\end{table}

Another line of works which is highly related to this paper includes \cite{szegedy_rethinking_2016,meister_generalized_2020,pereyra_regularizing_2017,wang_proselflc_2021} where the regularized forms of celLoss are considered such as the LSR, CP, and LC. See the difference between these loss functions and the regularized form of our proposed one in Table \ref{Tab:LossFunctionComparisons}. The key assumption of the LSR and the CP is that the one-hot label is too confident, and a less confident prediction should be preferred. This is achieved by encouraging the prediction to be also close to a uniform distribution in LSR, or to have high entropy in CP \cite{szegedy_going_2015,pereyra_regularizing_2017}. The LC assumes the model will fit to the data distribution before overfitting to the noise during training, and the model should trust its prediction after certain stages during training. This is achieved by encouraging the model to have low-entropy or high-confidence prediction \cite{wang_proselflc_2021}. In both \cite{szegedy_rethinking_2016,pereyra_regularizing_2017,wang_proselflc_2021} and most of other related works, the regularizations still look at the conditional label distribution only while the regularization term in our formulation looks at the marginal label distribution.

\subsection{Contributions}

The contributions of this work are summarized as follows.
\begin{itemize}
	\item We show that the existing cross entropy loss minimization approach for training DNN essentially learns the conditional entropy, and we point out some of the fundamental limitations of this approach. These limitations motivate us to propose a new training paradigm via mutual information learning.
	
	\item For the proposed mutual information learning (MIL) framework, we give theoretical anaysis to answer several fundamental questions, i.e., how the error probability over the distribution is connected to the mutual information between the data input and the data label, and what the sample complexity is for accurately learning the mutual information and thus a classifiers with excellent generalization performance. For the formal, we derive a lower bound for the error probability in terms of the mutual information. For the later, we derive the sample complexity for learning the MI from empirical risk minimization. These results are applicable for arbitrary data distributions. To better appreciate the MIL framework, we consider a concrete binary classification data model, and derive bounds for the mutual information error probability associated with the data distribution.
	
	\item As a proof of concept, we conduct extensive experiments with training DNN classifiers on several benchmark datasets to validate our theory, and the empirical results show that the proposed MIL can improve greatly the generalization performance of DNN classifiers.
	
\end{itemize}

This paper is organized as follows. In Section \ref{Sec:Preliminaries}, we present necessary definitions and relations which will be used in later sections. In Section \ref{Sec:CrossEntropyLossTrainingAsConditionalEntropyEstimation}, we show that the existing cross entropy loss minimization is equivalent to learning the conditional entropy of the label, and we also give upper bound of estimating entropy or conditional entropy from empirical data samples. We present the mutual information learning (MIL) framework in Section \ref{Sec:TrainingClassifiersViaMutualInformationEstimation}, derive the error probability lower bound in terms of the mutual information in Section \ref{Sec:ErrorProbLBound}, and establish the guarantees for learning the mutual information from empirical data samples in Section \ref{Sec:MILGuarantees}. In Section \ref{Sec:BinaryClassificationDataModel}, we consider a binary classification data model, and derive bounds of the mutual information and the error probability of the data distribution. We present experimental results in Section \ref{Sec:ExperimentalResults}, and conclude this paper in Section \ref{Sec:Conclusions}.

{\bf Notations:} We use $X$ to represent a random variable or vector, and its dimensions should be determined in the specific context.  We denote by $P(X)$ or $P_X$ the probability mass function of $X$ if $X$ is a discrete random variable or vector, and by $p(X)$ or $p_X$ the probability density function of $X$ if $X$ is continuous. Without loss of generality, we will refer to both as probability distribution. The joint distribution of a continuou radnom variable $X$ and discrete random variable $Y$ will be denoted by $p_{X,Y}$ or $p(X,Y)$. For the distribution $P_X$ of a discrete random variable $X$ with $N$ realizations, we will alternatively use it as a vector representation $P\in[0,1]^N$. Similarly, for the joint distribution $P_{X,Y}$ of discrete random variable $X$ with $N$ realizations and discrete random variable $Y$ with $C$ realizations, the $P$ will be alternatively denoted as a matrix in $[0,1]^{N\times C}$. We denote by $p_{X|U}$ or $P_{X|U}$ conditional distribution of $X$ given $U$. The entropy (or the differential entropy) of a discrete (or continuous) random variable $X$ is denoted by $H(X)$ (or $h(X)$). Unless specified, all the entropy (or differential entropy), and mutual information quantities are in nats. When it is necessary, we also use subscript to emphasize the distribution with respect to which these quantities are computed or simply for avoiding confusions, e.g., $I_{p_{X,Y}}(X;Y)$ means that the mutual information between $X,Y$ is calculated under distribution $p_{X,Y}$. For a distribution $Q_{Y}$ parameterized by $\theta$, we will denote it alternatively by $Q_{Y;\theta}$ and $Q_{Y}(\cdot;\theta)$. The probability mass (or the probability density) at a realization $y$ of discrete (continuous) random variable $Y$ will be denoted by $P_Y(Y)$ and $P_Y(Y=y)$ (or $p_Y(y)$ and $p_Y(Y=y)$) alternatively.

We use $[B]$ where $B$ is a positive integer to denote a set $\{1,2,\cdots,B\}$. We denote by $\mcS:=\{(x_i,y_i)\}_{i=1}^N \subset\mbR^n\times[C]$ a set of data samples where $x_i\in\mbR^n$ is the input (or feature) and $y_i$ is the corresponding output (or label, or prediction, or target), and by $\mcS_x:=\{x_i\}_{i=1}^N$ a set of features or input by dropping the labels. In this paper, we assume $\|x_i - x_j\| >0, \forall i\neq j$. The $\mcS|_x$ denotes a subset of $\mcS$ with the input being $x$, i.e., $\mcS|_x :=\{(x_i,y_i)\in\mcS: x_i=x\}$. We denote by $\bm{0}$ a vector or a matrix whose elements are all zero, and by $\bm{1}$ a vector or a matrix whose elements are 1. We use $I_n\in\mbR^{n\times n}$ to denote an identity matrix. The $|Q|$ denotes the determinant of a square matrix $Q$, and the cardinality of a set $\mcS$ is denoted by $\card(\mcS)$ or $|\mcS|$. All the proofs can be found in the Appendix.

\section{Preliminaries}\label{Sec:Preliminaries}

We consider the classification tasks in machine learning, i.e., given a dataset $\mcS:=\{(x_i,y_i)\}_{i=1}^N $ drawn according to a joint data distribution $p_{X,Y}$ where $(x_i,y_i)\in\mbR^n\times[C]$ with $C$ being a positive integer, we want to learn a mapping $M:\mbR^n\to[C]$ from $\mcS$ such that $M$ can classify an unseen sample $x'\sim p_X$ in $\mbR^n$ to the correct class. The mutual information $I(X,Y)$ of the input $X$ and the label $Y$ under the joint distribution $p_{X,Y}$ in this setup can then be defined as
\begin{align}\label{Eq:MutualInfoDefn}
	I(X;Y)
	&:= \int_{x\in\mbR^n} \sum_{y\in [C]} p(x,y) \log\left( \frac{p(x,y)}{p(x) P(y)} \right) dx,
\end{align} 
where we also define $p(x,y):=p(x)P(y|x)$ and $p(x,y):=P(y)p(x|y), \forall x\in\mbR^n, y\in[C]$. Similar to the Shannon information theory framework, we define several other information-theoretic quantities as in Definition \ref{Definition:Ent}, \ref{Definition:ConditionalEnt}, \ref{Definition:CrossEntropy}, and \ref{Definition:ConditionalCrossEntropy}. The definitions of entropy, differential entropy, and cross entropy are exactly the same as those in Shannon information theory, and we present them for self-containedness.

\begin{definition}\label{Definition:Ent}
	(Differential Entropy and Entropy) For a continuous random vector $X\in\mbR^n$ with distribution $p_X$, we define its differential entropy $h(X)$
	as
	\begin{align}\label{Defn:DifferentialEntropy}
		h(X): = - \int_{x\in\mbR^n} p(x) \log(p(x)) dx.
	\end{align}
For a discrete random variable $Y\in[C]$ with distribution $P_Y$, we define its entropy as
\begin{align}\label{Defn:Entropy}
	H(Y):= -\sum_{y\in[C]} P(y) \log(P(y)).
\end{align}
\end{definition}

\begin{definition}\label{Definition:ConditionalEnt}
	(Conditional Differential Entropy and Conditional Entropy) For a joint distribution $p_{X,Y}$ of a continuous random vector $X\in\mbR^n$ and discrete random variable $Y\in[C]$, we define the conditional differential entropy $h(X|Y)$ as 
	\begin{align}\label{Defn:ConditionalDifferentialEnt}
		h(X|Y):=  \sum_{y\in[C]} P(y) \int_{x\in\mbR^n} p(x|y) \log\left( \frac{1}{p(x|y)} \right)dx,
	\end{align}
and the instance conditional differential entropy at realization $y$ for $Y$ as
\begin{align}\label{Defn:InstcConditionalDifferentialEnt}
	h(X|y):= \int_{\mbR^n} p(x|y) \log\left( \frac{1}{p(x|y)} \right)dx.
\end{align}

We define the conditional entropy $H(Y|X)$ as
\begin{align}\label{Defn:ConditionalEnt}
	H(Y|X):=  \int_{x\in\mbR^n}  p(x)   \sum_{y\in[C]} P(y|x) \log\left( \frac{1}{P(y|x)} \right)dx, 
\end{align}
and the instance conditional entropy at realization $x$ of $X$ as 
\begin{align}\label{Defn:InstcConditionalEnt}
	H(Y|x):= \sum_{y\in[C]} P(y|x) \log\left( \frac{1}{P(y|x)} \right).
\end{align}
\end{definition}

\begin{definition}\label{Definition:CrossEntropy}
	(Cross Entropy) We define the cross entropy between two continuous distributions $p_X, q_X$ over the same continuous support set $\Omega$ as
	 \begin{align} 
		h(p_X,q_X):=\int_{x\in\Omega} p_X(x) \log\left( \frac{1}{q_X(x)} \right) dx.
	\end{align} 
We define the cross entropy between discrete distributions $P_Y, Q_Y$ over the same discrete support set $\Omega$ as 
\begin{align} 
	H(P_Y,Q_Y):=\sum_{y\in\Omega} P_Y(y) \log\left( \frac{1}{Q_Y(y)} \right).
\end{align}
\end{definition}

\begin{definition}\label{Definition:ConditionalCrossEntropy}
	(Conditional Cross Entropy) For two joint distributions $p_{X,Y}$ and $q_{X,Y}$ of a continuous random vector $X\in\mbR^n$ and discrete random variable $Y\in[C]$, we define the conditional cross entropy $H(P_{Y|X}, Q_{Y|X})$ as 
	\begin{align}\label{Defn:ConditionalCrossEnt}
		H(P_{Y|X},Q_{Y|X}):=
		\int_{x\in\mbR^n} p_X(x) \sum_{y\in[C]} P_{Y|X}(y|x) \log\left( \frac{1}{Q_{Y|X}(y|x)} \right)dx,
	\end{align}
and the conditional cross entropy $H(P_{X|Y}, Q_{X|Y})$ as 
\begin{align}\label{Defn:ConditionalCrossEnt2}
	h(p_{X|Y}, q_{X|Y}):= \sum_{y\in[C]} P_Y(y)  \int_{x\in\mbR^n} p_{X|Y}(x|y) \log\left( \frac{1}{q_{X|Y(x|y)}} \right)dx.
\end{align}
\end{definition}

The proposed concept of conditional cross entropy will be used to derive a new formulation for training classifiers. The connections among these information-theoretic quantities are presented in Theorem \ref{Thm:ITQuantitiesConnections}. Theorem 2.1 shows that for the joint distribution of a continous random variable and a discrete random variable, the relations among the mutual information, the entropy, and the conditional entropy are are exactly the same as those in the case where the joint distribution is over two continuous random variables or over two discrete random variables. The proof of Theorem \ref{Thm:ITQuantitiesConnections} can be found in the Appendix.

\begin{thm}\label{Thm:ITQuantitiesConnections}
	(Connections among different information-theoretic quantities) For a joint distribution $p_{X,Y}$ over a continous random vector $X\in\mbR^n$ and a discrete random variable $Y\in[C]$ where $C$ is a positive integer constant, with the definition of mutual information in \eqref{Eq:MutualInfoDefn} and Defintion \ref{Definition:Ent}-\ref{Definition:ConditionalCrossEntropy}, we have
	\begin{align}\label{Eq:ITQuantitiesConnections}
		I(X,Y) = H(Y) - H(Y|X), \nonumber\\
		I(X,Y) = h(X) - h(X|Y).
	\end{align}
\end{thm}

\section{Cross Entropy Loss Minimization as Conditional Entropy Learning}\label{Sec:CrossEntropyLossTrainingAsConditionalEntropyEstimation}

The common practice in the machine/deep learning commnunity separates the training or learning process and the decision or inference process, i.e., by first learning a conditional probability $Q_{Y|X}(y|x; \theta_{Y|X})\in[0,1]$of $Y$ given a realization $x$ of $X$, and then making decisions about the labels via checking which class achieves the highest probability. The $\theta_{Y|X}$ denotes the parameters associated with the function $Q_{Y|X}(\cdot;\theta_{Y|X})$. The former process is achieved by minimizing the cross entropy loss between an empirical conditional distribution $\hat{P}_{Y|X}$ from data samples and the estimated conditional distribution $Q_{Y|X}(\cdot;\theta_{Y|X})$, i.e.,
\begin{align}\label{Eq:CondEntLoss}
	\min_{\theta_{Y|X}} \frac{1}{N} \times \sum_{(x,y)\in\mcS} \left(\sum_{ c\in[C]} 
	\left( \hat{P}_{Y|X}(c|x) \log\left( \frac{1}{ Q_{Y|X}(c|x; \theta_{Y|X})} \right) \right) \right),
\end{align}
while the inference process is then achieved via 
\begin{align}\label{Eq:DecideLabelFromConditionalProbability}
	\hat{y}_i = \arg\max_{c\in[C]} \left[Q_{Y|X}(x_i; \theta_{Y|X})\right]_c.
\end{align} 
Thus, the mapping $M:\mbR^n\to[C]$ is a composite function of the probability prediction function $Q_{Y|X}(\cdot;\theta_{Y|X})$ and the maximum probability inference function $\arg\max$.

The empirical distribution $\hat{P}_{Y|X}$ is usually affected by the data collection process and the encoding methods for labels. For example, in image recognition tasks, a single image can have multiple objects, but it is at the human annotators' discretion about which label we want to use. Even for this particular label, different label encoding methods can give different label representations. When the one-hot representation is used for encoding labels, the empirical distribution $\hat{P}_{Y|X}$ will be 
\begin{align}\label{Eq:EmpiricalLabelConditionalDistribution}
	\hat{P}_{Y|X}(y|x_i) = 
	\begin{cases}
		1, {\rm if\ } y=y_i,\\
		0, {\rm o.w.}
	\end{cases}
\end{align}
which means that only the target class gets all the probability mass while all the other classes have zero probability mass. When the label-smoothing regularization is used, the $\hat{P}_{Y|X}$ is defined as
\begin{align}\label{Eq:LabSmthReg}
	\hat{P}_{Y|X}(y|x_i) =
	\begin{cases}
		1-\epsilon, \text{ if } y=y_i,\\
		\frac{\epsilon}{C-1}, \text{o.w.}
	\end{cases}
\end{align}
where $\epsilon>0$ is a constant \cite{szegedy_rethinking_2016}. With the one-hot encoding for labels, the cross entropy minimization \eqref{Eq:CondEntLoss} can be simplified as
\begin{align}\label{Eq:CondEntLoss_OneHot}
	\min_{\theta_{Y|X}} \frac{1}{N} \times \sum_{(x,y)\in\mcS} \left( 
	\log\left( \frac{1}{ Q_{Y|X}(y|x; \theta_{Y|X})}  \right) \right)
\end{align}
The $Q_{Y|X}(\cdot; \theta_{Y|X})$ can be interpreted as estimated conditional distribution of $Y$ when the realization $x$ of the continuous random variable $X$ is given.

The $\sum_{ c\in[C]} 
\hat{P}_{Y|X}(c|x) \log\left( \frac{1}{ Q_{Y|X}(c|x; \theta_{Y|X})} \right)$ in \eqref{Eq:CondEntLoss} can be interpreted as an estimate of the instance conditional entropy of the truth data distribution $P_{Y|X}(Y|x)$ conditioning on the realization $x$ of $X$, i.e., 
\begin{align}\label{Eq:InstanceConditionalEntropy}
	H(Y|x) \approx \sum_{ c\in[C]} 
	\hat{P}_{Y|X}(c|x) \log\left( \frac{1}{ Q_{Y|X}(c|x;\theta_{Y|X})} \right).
\end{align}
In Theorem \ref{Thm:EntEstViaCrossEnt}, we will show that this is indeed the case under certain conditions, and the cross entropy minimization in \eqref{Eq:CondEntLoss} learns the conditional entropy of the truth data distribution $P_{Y|X}$. Thus, we will refer to \eqref{Eq:CondEntLoss} as conditional entropy learning (CEL) and the corresponding loss function as CEL loss (celLoss). In the later sections, we will propose a mutual information learning (MIL) framework, and refer to the corresponding loss as MIL loss (milLoss). 

Similar to Shannon information theory, we define the $\hat{P}_X$ associated with the input data $\mcS_x:=\{x_i\}_{i=1}^N$ of the data set $\mcS:=\{(x_i,y_i)\}_{i=1}^N$ drawn from distribution $p_{X,Y}$ as 
\begin{align}\label{Eq:XType}
	\hat{P}(X=x_i) = \frac{\left|\{x \in \mcS_x: x = x_i\} \right|}{N} \in[0,1], \forall i \in[N],
\end{align}
where we define $[N]:=\{1,2,\cdots, N\}$. Since we assume $\|x_i-x_j\|>0, \forall i\neq j$, we have $\hat{P}(X=x_i)=\frac{1}{N}$. The objective function in \eqref{Eq:CondEntLoss} becomes an estimate of the conditional entropy $H(Y|X)$ with respect to $P_{Y|X}$, i.e., 
\begin{align}
	 H(Y|X)
	 \approx \sum_{(x,y)\in\mcS} 
	\left( \hat{P}_X(x) \hat{P}_{Y|X}(y|x) \log\left( \frac{1}{ \left[Q_{Y|X}(x;\theta^*_{Y|X})\right]_y} \right) \right),
\end{align}
where $Q_{Y|X}(\cdot;\theta^*_{Y|X})$ is an optimal conditional distribution determined by $\theta_{Y|X}$. In Theorem \ref{Thm:EntEstViaCrossEnt} and its implications, we show that the cross entropy minimization in \eqref{Eq:CondEntLoss} for training classifiers learns the conditional entropy when $Q_{Y|X}(\cdot|\theta_{Y|X})$ is the solution to an optimization problem.

\begin{thm}\label{Thm:EntEstViaCrossEnt}
	(Cross Entropy Minimization as Entropy Learning) For an arbitrary discrete distribution $P_{Y}$ in $[C]$, we have 
	\begin{align}
		H(Y) \leq \inf_{Q_Y} H(P_Y,Q_Y),
	\end{align}
	where $Q_Y$ is a distribution of $Y$, and the equality holds if and only if $P_Y=Q_Y$. When a set of $N$ data points $\mcS:=\{y_i\}_{i=1}^N$ drawn independently from $P_{Y}$ is given, by defining $R(y):= \frac{P_Y(y)}{\hat{P}_Y(y)}$ where $\hat{P}_Y$ is the empirical distribution associated with $\{y_i\}_{i=1}^N$, we have 
	\begin{align}
		H(Y) \leq \inf_{Q^g_Y} H(\hat{P}_Y^g, Q_Y^g),
	\end{align}
	where $\hat{P}_Y^g$ is defined as
	\begin{align}
		\hat{P}_Y^g(y) : = \hat{P}_Y(y) R(y), \forall {y\in[C]},
	\end{align}
	and $Q_Y^g$ is defined as 
	\begin{align}
		Q_Y^g(y) = Q_Y(y)R(y), \forall {y\in[C]},
	\end{align}
	with $Q_Y$ being a distribution of $Y$. The inequality holds if and only if $P_Y = \hat{P}_Y = Q_Y$.
\end{thm}

The proof of Theorem \ref{Thm:EntEstViaCrossEnt} can be found in the Appendix. Theorem \ref{Thm:EntEstViaCrossEnt} shows that the entropy $H(Y)$ is upper bounded by the cross entropy $H(P_Y,Q_Y)$, and the calculation of $H(Y)$ can be achieved by finding $Q_Y$ to minimize $H(P_Y,Q_Y)$. Theorem \ref{Thm:EntEstViaCrossEnt} also tells us that the entropy $H(Y)$ can actually be estimated from the empirical distribution $\hat{P}_Y$ over the sample set by minimizing a generalized cross entropy $H(\hat{P}_Y^g,{Q}_Y^g)$. We want to emphasize that though we call $H(\hat{P}_Y^g,{Q}_Y^g)$ generalized cross entropy, it may not actually be an entropy since the $\hat{P}_Y^g,{Q}_Y^g$ may not be valid distributions. In practice, when we assume $R(y)=1, \forall y\in[C]$, and this gives 
\begin{align}\label{Eq:EntViaCrossEntCont2}
	\hat{H}(Y)
	:= \inf_{Q_Y} H(\hat{P}_Y,{Q}_Y)
\end{align}
which is equal to $H(Y)$ if and only if $P_Y(y) = \hat{P}_Y(y) = Q_Y(y)$. Theorem \ref{Thm:EntEstViaCrossEnt} essentially implies the possibility of learning entropy from empirical samples. 

In the classification tasks as we discussed in Section \ref{Sec:Preliminaries}, we can have similar formulation for $H(Y|X)$ 
\begin{align}\label{Eq:EntViaCrossEntCont3}
	{H}(Y|X)
	& \approx \inf_{Q_{Y|X}} H(\hat{P}_{Y|X},{Q}_{Y|X}) \nonumber \\
	& = \inf_{Q_{Y|X}} \sum_{(x,y)\in \mcS} \hat{p}_{X,Y}(x,y) \log\left( \frac{1}{Q_{Y|X}(y|x)} \right) \nonumber\\
	& = \inf_{Q_{Y|X}} \sum_{(x,y)\in \mcS} \hat{P}_{X}(x) \hat{P}_{Y|X}(y|x) \log\left( \frac{1}{Q_{Y|X}(y|x)} \right)\\
\end{align}
where $\hat{P}_{Y|X}$ is the empirical conditional distribution of $Y$ given $X$.
In \eqref{Eq:CondEntLoss}, the objective function in the cross entropy minimization is essentially $H(\hat{P}_{Y|X},{Q}_{Y|X})$ with ${Q}_{Y|X}$ parameterized by $\theta_{Y|X}$. This implies that if the $\hat{P}_{Y|X}$ and $\hat{P}_X$ are the same as the ground truth distributions, then the optimal objective function value in \eqref{Eq:CondEntLoss} will be the conditional entropy $H(Y|X)$, i.e., the commonly used cross entropy loss minimization is essentially learning the conditional entropy $H(Y|X)$. We will refer to these classifiers as conditional entropy learning classifiers (CELC, /selk/).

From Theorem \ref{Thm:ITQuantitiesConnections},  the mutual information (MI) can be learned via
\begin{align}\label{Eq:MutualInfoViaDoE_YEntMinusYCondEnt}
	I(X;Y)
	& = H(Y) - H(Y|X) \nonumber \\
	& \approx \inf_{Q_Y} H(\hat{P}_Y,Q_Y) - \inf_{Q_{Y|X}} H(\hat{P}_{Y|X},Q_{Y|X}),
\end{align}
where $\hat{P}_Y$ is the empirical distribution of $Y$ associated with the label components $\mcS_y:=\{y_i\}_{i=1}^N$ in $\mcS$, i.e.
\begin{align}\label{Eq:UType}
	\hat{P}_Y(Y=c) := \frac{\left| \{y\in\mcS_y: y=c\} \right|}{N}, c\in[C].
\end{align}

We want to point out that in \cite{mcallester_formal_2020}, McAllester and Stratos also proposed a formula similar to \eqref{Eq:MutualInfoViaDoE_YEntMinusYCondEnt} as an estimation for MI, i.e., 
estimating $h(X) = \inf_{q_X}H(p_X,q_X)$ via
\begin{align}\label{Eq:CrossEntEst}
	\hat{H}^N(\hat{P}_X,Q_X) = \inf_{Q_X} \left(- \frac{1}{N} \sum_{i=1}^N \log(Q_X(x_i)) \right),
\end{align}
but they did not quantify the relation between the cross entropy computed via the truth distribution and that computed via the empirical distribution. Besides, there are other fundamental differentiations between their formulation and the one we use for establishing the mutual information learning framework in later sections, e.g., we will use a single neural network to learn both the marginal distribution $Q_Y$ and the conditional distribution $Q_{Y|X}$ while McAllester and Stratos used two separate neural networks to acheive this. 

\section{Training Classifiers via Mutual Information Learning}\label{Sec:TrainingClassifiersViaMutualInformationEstimation}

In this section, we formally present a new framework for training classifiers, i.e., via mutual information learning instead of the conditional entropy learning in existing paradigm, and we refer to classifier trained in this way as mutual information learned classifier (MILC, /milk/). 

For the mutual information learning formulation in \eqref{Eq:MutualInfoViaDoE_YEntMinusYCondEnt}, we can parameterize $Q_Y$ and $Q_{Y|X}$ via a group of parameters $\theta_{Y|X}\in\mbR^{m}$, e.g., $Q_Y(Y; \theta_{Y|X})$ is completely determined by $\theta_{Y|X}$.  More specifically, we parameterize $Q_{Y|X}(Y|X;\theta_{Y|X})$ using $\theta_{Y|X}$, and then calculate the marginal estimation ${Q}_{Y}(Y;\theta_{Y|X})$ via
\begin{align}\label{Eq:MarginalDistributionFromLearnedJointDistribution}
	{Q}_Y(Y=y; \theta_{Y|X}) 
	&:= \sum_{x\in\mcS_x} \hat{P}_X(x) Q_{Y|X}(y|x;\theta_{Y|X})\\
	& = \frac{1}{N}\sum_{(x_i,y_i)\in\mcS} 1_{\{y\}}(y_i) Q_{Y|X}(y_i|x;\theta_{Y|X})
\end{align}
where $\mcS_x=\{x_i\}_{i=1}^N$ and $ 1_{A}(y)=1$ if $y$ is in set $A$, and 0 if otherwise. Thus, the mutual information has the following form
\begin{align}\label{Eq:MIAlterEstParametericMethod_MargDistFromLearnedJointDist}
	I(X;Y)
	\approx  \inf_{\theta_{Y|X}} 
	H(\hat{P}_Y,{Q}_Y(Y;{\theta_{Y|X}})) - 
	\inf_{\theta_{Y|X}}H(\hat{P}_{Y|X}, Q_{Y|X}(Y|X; \theta_{Y|X})),
\end{align}
which is a multi-object optimization problem \cite{miettinen_nonlinear_2012}. 

An equivalent form of \eqref{Eq:MIAlterEstParametericMethod_MargDistFromLearnedJointDist} can be a regularized form as follows
\begin{align}\label{Eq:MIAlterEstParametericMethodEquivalentRegluarizedForm}
	\inf_{\theta_{Y|X}} 
	\left( H(\hat{P}_{Y|X}, Q_{Y|X}(Y|X; \theta_{Y|X})) +  \lambda_{ent}H(\hat{P}_Y, Q_Y(Y; \theta_{Y|X})) \right),
\end{align}
where $\lambda_{ent}>0$ is a regularization hyperparameter. In \eqref{Eq:MIAlterEstParametericMethodEquivalentRegluarizedForm}, the $H(\hat{P}_{Y|X}, Q_{Y|X})$ essentially corresponds to the cross entropy loss in multi-class classification while $H(\hat{P}_Y, Q_Y)$ can be treated as a regularization term. The cross entropy term guides machine learning algorithms to learn accurate estimation of condition entropy $H(Y|X)$, and the regularization term encourages the model also to learn the label entropy $H(Y)$. Intuitively, such a mutual information learning goal can guide the model to learn more accurately the dependency between input $X$ and $Y$, thus better generalization performance. To the best of our knowledge, the label entropy has never been used to train machine learning systems by the community, and we are the first to propose mutual information learning (MIL) framework for training classifiers. Though similar ideas were proposed in previous works, what considered previously is essentially the conditional entropy instead of the entropy \cite{meister_generalized_2020}. Their goal is to increase the conditional entropy of the label when the the input is given while our formulation aims at accurately characterizing the dependency between the input and ouput. Besides, in previous work, the increase of conditional entropy of label is acheived by encouraging the predicted label conditional distribution to be close to a uniform label distribution. However, our formula encourages the model to give predictions which characterize the mutual information well. 

\begin{figure*}[!htb]
	\centering
	\begin{subfigure}[b]{\textwidth}
		\centering
		\includegraphics[width=\linewidth]{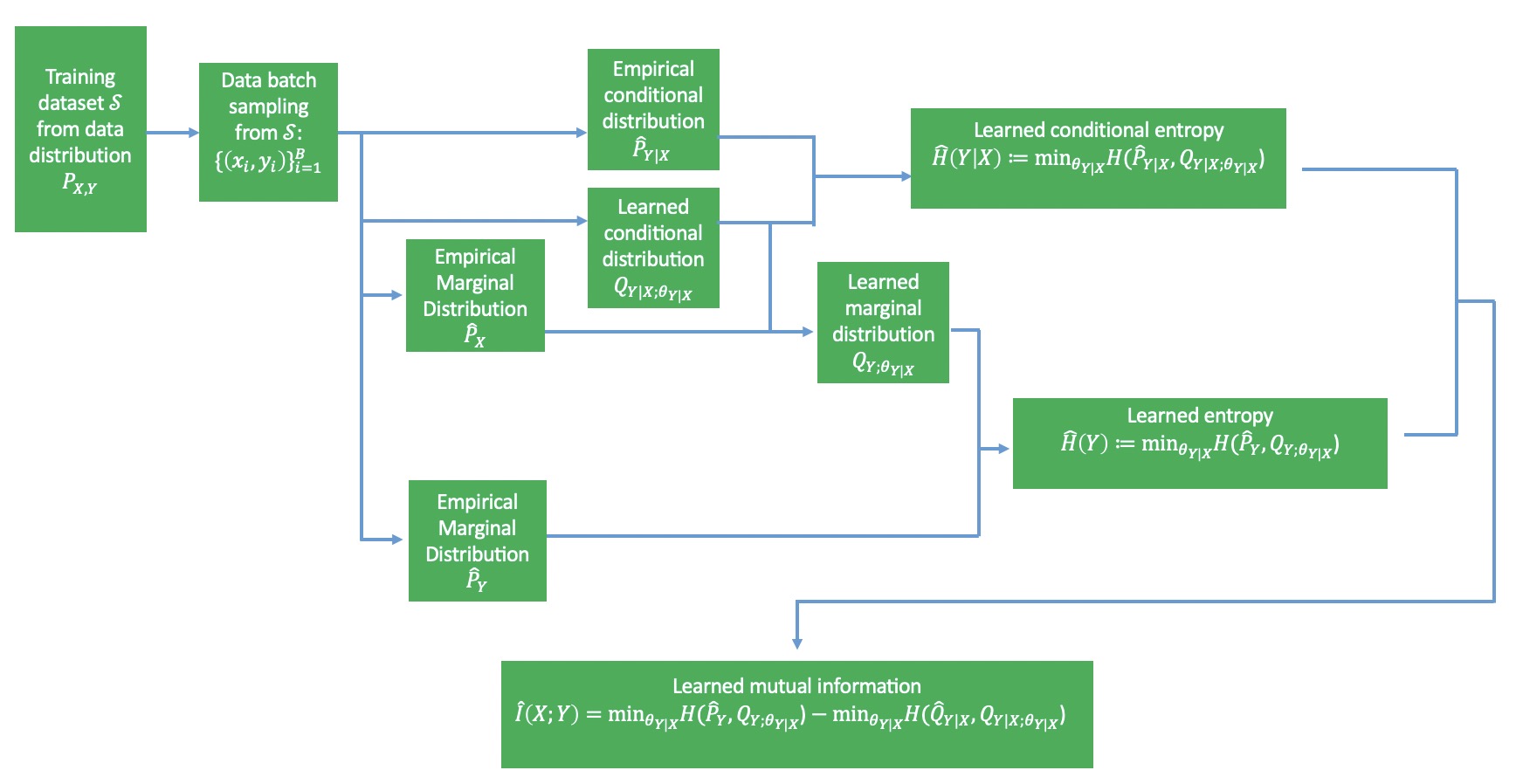}
		\caption{Training/learn pipeline}\label{Fig:MIML_TrainingPipelineV2}
	\end{subfigure}%
	\\
	\begin{subfigure}[b]{\textwidth}
		\centering
		\includegraphics[width=\linewidth]{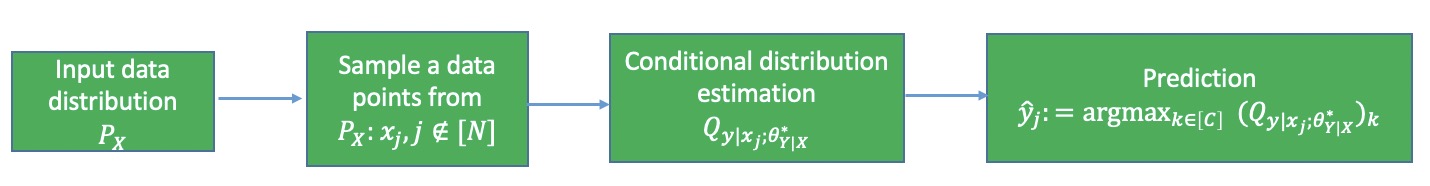}
		\caption{Inference/decision pipeline}\label{Fig:MIML_InferencePipeline}
	\end{subfigure}
	\caption{Mutual information learning classifiers (MILC): $\theta^*_{Y|X}$ is an optimal solution to \eqref{Eq:MIAlterEstParametericMethodEquivalentRegluarizedForm}}\label{Fig:MIML}
\end{figure*}

The overall DNN classifiers' trainning or learning pipeline under the MIL framework is presented in Figure \ref{Fig:MIML_TrainingPipelineV2}, and the corresponding decision or inference pipeline is presented in Figure \ref{Fig:MIML_InferencePipeline}. During the training process, we sample a data batch from training dataset $\mcS$ in each iteration, and then calculate the empirical marginal distributions $\hat{P}_X$, $\hat{P}_Y$ and $\hat{P}_{Y|X}$. The inputs $\{x_i\}_{i=1}^B$ will be fed to a machine learning system for it to learn the conditional distribution $Q_{Y|X;\theta_{Y|X}}$. We then combine the $Q_{Y|X;\theta_{Y|X}}$ with $\hat{P}_X$ and $\hat{P}_{Y|X}$ separately to calculate the learned marginal distribution $Q_{Y;\theta_{Y|X}}$ and the learned conditional entropy $\hat{H}(Y|X)$. The $Q_{Y;\theta_{Y|X}}$ is then combined with the $\hat{P}_Y$ to calculate the label entropy. We finally calculate the mutual information by subtracting conditional entropy from entropy. During inference, we feed an input $x$ to the model to get a conditional distribution, and the final class label prediction will be the one achieving the highest probability. 

One may want to estimate the MI via empirical distribution only, i.e.,
\begin{align}\label{Eq:MutualInfoNaive}
	I(X;Y) 
	& = H(Y) - H(Y|X) \nonumber\\
	& \approx \sum_{c\in[C]} \hat{P}_Y(Y=c)\log\left(\frac{1}{\hat{P}_Y(Y=c)}\right)
	-  \sum_{i\in[N]} \hat{P}_X(x_i) \times \sum_{c\in[C]} \hat{P}_{Y|X}(c|x_i) \times \log\left(\frac{1}{\hat{P}_{Y|X}(c|x_i)}\right), 
\end{align}
where 
\begin{align*}
	\hat{P}(Y=c):=\frac{\left|\{(x_i,y_i )\in\mcS: {y}_i=c\} \right|}{N}, \forall c\in [C],\\
	\hat{P}(X=x):=\frac{\left|\{(x_i,y_i )\in\mcS: {x}_i=x\} \right|}{N}, \forall x\in \mcS|_x,\\
	\hat{P}_{Y|X}(c|x_i):=\frac{ |\{(x,y)\in\mcS: x=x_i, y=c\}| }{|\{(x,y)\in\mcS: x=x_i\}|}.
\end{align*}
The problem is that the estimation of entropy or cross entropy from type only can be quite inaccurate, and this is formally presented in Theorem \ref{Thm:EntEstBound} where we give the error bound of estimating entropy of $Y$ via the empirical distribution $\hat{p}_Y$. However, in our previous formulations, we use the combination of the empirical distirbution and a learned distribution to avoid this.

\begin{thm}\label{Thm:EntEstBound}
	(Error Bound of Entropy Learning from Empirical Distribution) For two arbitrary distributions $P_Y$ and $\hat{P}_Y$ of a discrete random variable $Y$ over $[C]$, we have 
	\begin{align}
		\sum_{y\in[C]} R(y) \log\left( \frac{1}{{P}_Y(y)} \right)
		\leq 
		H_{P_Y}(Y) - H_{\hat{P}_Y}(Y) 
		\leq  \sum_{y\in[C]} R(y) \log\left( \frac{1}{\hat{P}_Y(y)} \right)
	\end{align}
	where $R(y) = P_Y(y) - \hat{P}_Y(y), \forall y \in[C]$, and $H_{P_Y}(Y)$ is the entropy of $Y$ calculated via $P_Y$. The equality holds if and only if $P_Y = \hat{P}_Y$. 
\end{thm}

Theorem \ref{Thm:EntEstBound} gives a bound for the gap between entropies calculated using different distributions, and it  applies to arbitrary distributions. In the special scenario where $P_Y$, and $\hat{P}_Y$ are the true data distribution and the empirical distribution associated with data samples,  Theorem \ref{Thm:EntEstBound} actually gives error bound of estimating entropy using empirical distribution. 

We want to point out that though our formulation in \eqref{Eq:MIAlterEstParametericMethod_MargDistFromLearnedJointDist} can be used to estimate the mutual information because $\inf_{\theta_{Y|X}} 
H(\hat{p}_Y, \hat{Q}_Y(Y;{\theta_{Y|X}}))$ and $\inf_{\theta_{Y|X}}H(\hat{p}_{Y|X}, q_{Y|X}(Y|X; \theta_{Y|X}))$ can give accurate estimate of entropy $H(Y)$ and conditional entropy $H(Y|X)$, our primary goal is to learn classifiers with excellent generalization performance. Besides, our Theorem \ref{Thm:EntEstViaCrossEnt} implies that our formulation can also be used to estimate the joint entropy of the dataset $H(X,Y)$ via
\begin{align}\label{Eq:JointEntropyEstimation}
	H(X,Y)\approx \inf_{\theta_{X,Y}} H(\hat{P}_{X,Y}, Q_{X,Y}(X,Y; \theta_{X,Y})),
\end{align}
which quantifies the amount of information contained in a dataset. The $\theta_{X,Y}$ is the model parameter. However, the joint entropy only quantifies the information contained in the dataset, but gives no characterization of the dependency between $X$ and $Y$, thus it may not help with learning classifiers with good generalization performance. 

\section{Error Probability Lower Bounds via Mutual Information}\label{Sec:ErrorProbLBound}

In this section, we establish the error probability bound of an arbitrary learning algorithm in terms of mutual information associated with the dataset used to train the models. To do this, we first follow Yi et al. to model the learning process as in Figure \ref{Fig:ErrorProbability} \cite{yi_trust_2019,xie_information-theoretic_2019,yi_derivation_2020,yi_towards_2021}. More specifically, we assume there is a label distribution $P_Y$, and based on realizations from $P_Y$, we can generate a set of observations from $p_X$. Given the observations, we want to infer the label of them. By combining the ground truth labels sampled from $P_Y$ and the predicted labels, we then calculate the error probability as
\begin{align}\label{Eq:ErrorProbability}
	P_{error} :=
	P_{Y,\hat{Y}}\left( \{(Y,\hat{Y}): Y\neq \hat{Y}\} \right).
\end{align}
This learning process is consistent with practice. For example, in a dog-cat image classification tasks, we first have the concepts of the two classes, i.e., cat and dog. Then, we can generate observations of these labels/concepts, i.e.,  images of cat and dog by taking pictures of them, or simply drawing them. We then use these observations to train models, hoping that they will finally be able to predict the correct labels. We want to mention that the community has witnessed significant progress in image generation, and it is fairly straightforward to general such dog and cat images \cite{makhzani_pixelgan_2017,lin_infinitygan_2021,preechakul_diffusion_2022}.

\begin{figure}[!htb]
	\centering
	\includegraphics[width=\linewidth]{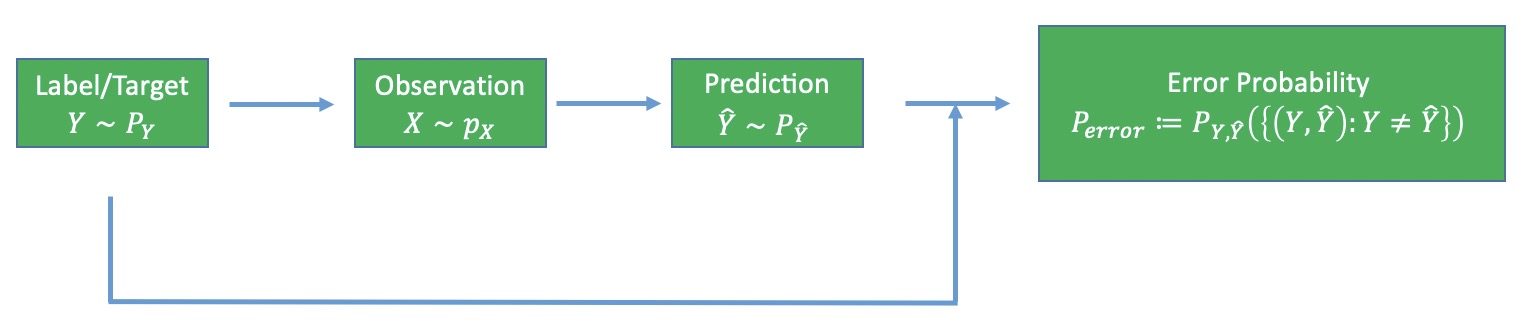}
	\caption{Information-theoretic view point of learning process.}\label{Fig:ErrorProbability}
\end{figure}

Under this framework, we can show that the error probability associated with the learning process as shown in Figure \ref{Fig:ErrorProbability} can be bounded via mutual information $I(X;Y)$. The results are formally shown in Theorem \ref{Thm:ErrorProbabilityMI_Bounds}, and we first present a useful lemma which will be used in Theorem \ref{Thm:ErrorProbabilityMI_Bounds}.

\begin{lemma}\label{Lem:ErrorEntropyUB}
	(\cite{yi_derivation_2020}) For arbitrary $x\in[0,1]$, we have
	\begin{align}\label{Eq:BinaryEntrpUpBound}
		x\log\left(\frac{1}{x}\right) + (1-x) \log\left(\frac{1}{1-x}\right)  \leq 1-2(x - 0.5)^2.
	\end{align}
\end{lemma}

Lemma \ref{Lem:ErrorEntropyUB} can be used to bound the entropy associated with a binary distribution, and a simple visual illustration of it is presented in Figure \ref{Fig:BinaryDistrEntUB} where we let $x$ be the error probability and $1-x$ be the correct probability (or accuracy).

\begin{figure}[!htb]
	\centering
	\includegraphics[width=0.5\linewidth]{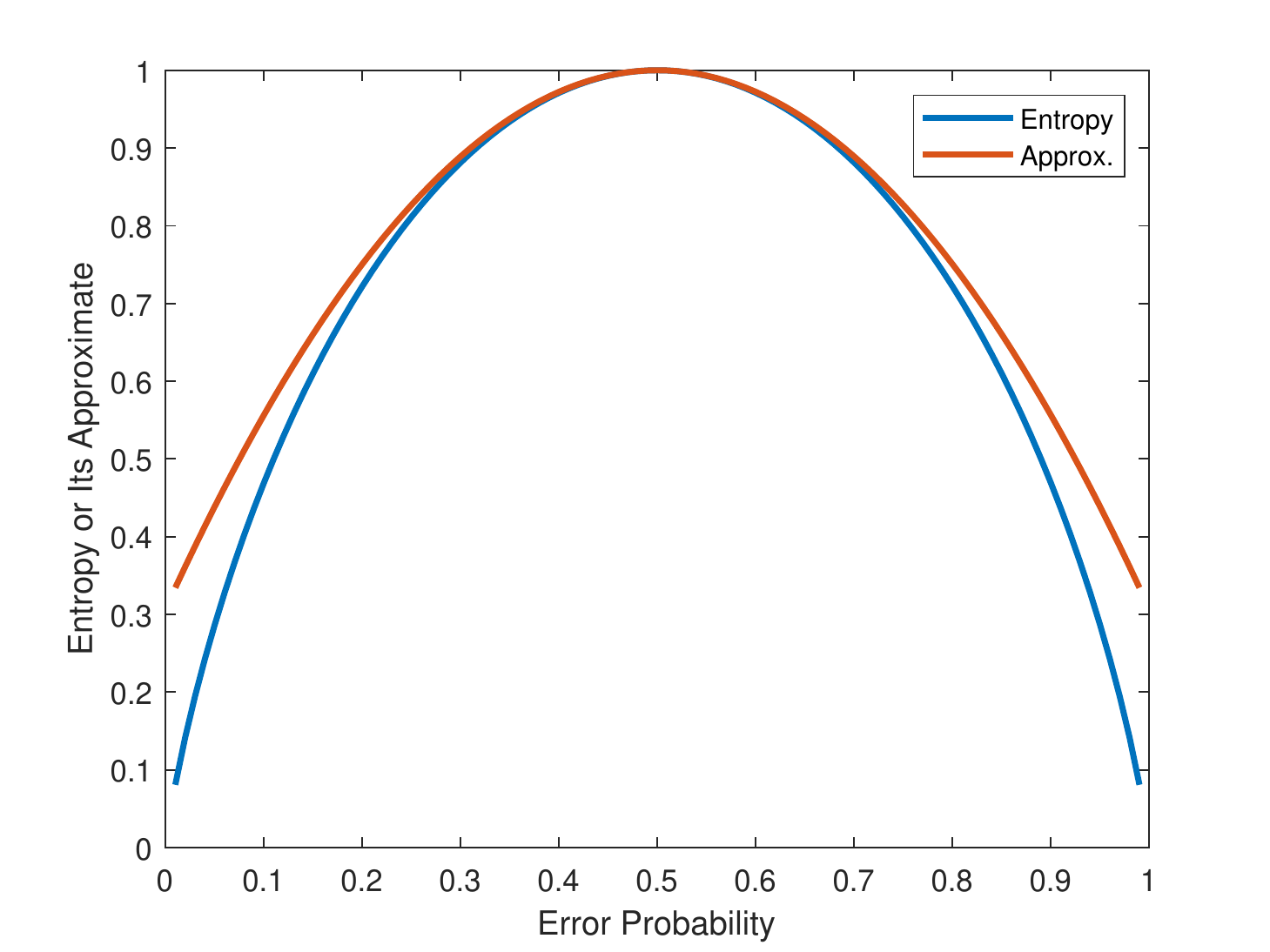}
	\caption{Upper bound of entropy assoicated with binary distribution: error probability $x$ and correct probability $1-x$. The blue curve corresponds to left hand side of \eqref{Eq:BinaryEntrpUpBound}, and the red curve corresponds to the right hand side of \eqref{Eq:BinaryEntrpUpBound}.}\label{Fig:BinaryDistrEntUB}
\end{figure}

\begin{thm}\label{Thm:ErrorProbabilityMI_Bounds}
	(Error Probability Bound via Mutual Information) Assume that the learning process $Y\to X\to\hat{Y}$ in Figure \ref{Fig:ErrorProbability} is a Markov chain where $Y\in[C]$, $X\in\mbR^n$, and $\hat{Y}\in[C]$, then for the prediction $\hat{Y}$ from an arbitrary learned model, we have
	\begin{align}\label{Eq:ErrorProbabilityBounds}
			\max\left( 0,\frac{2+H(Y) - I(X;Y) - a}{4} \right)
		\leq P_{error}
		%\leq 	\min\left( 1,\frac{2+H(Y) - I(X;Y) + a}{4}\right),
	\end{align}
where $a:=\sqrt{(H(Y) - I(X;Y) - 2)^2 + 4}$.
\end{thm}

The proof of Theorem \ref{Thm:ErrorProbabilityMI_Bounds} can be found in the appendix.  From \eqref{Eq:IntermediateBound} in the proof of Theorem \ref{Thm:ErrorProbabilityMI_Bounds}, we can see that $P_{error}  \geq 1- \frac{1}{H(Y) - I(X;Y)}$, and the $P_{error}$ lower bound will decrease when $I(X;Y)$ increases. The Theorem also implies that $I(X;Y) \geq H(Y) - \frac{H(P_{error})}{1-P_{error}}$, which means the mutual information between $X,Y$ should be at least $H(Y) - \frac{H(P_{error})}{1-P_{error}}$ so that we can achieve an error probability $P_{error}$. These are consistent with our intuitions. For example, when the dependence or MI between the observation $X$ and the label $Y$ gets weak, it will be more challenging to infer $Y$ from $X$, thus a larger error probability can occur. We also want to mention that our bound is tighter than the Fano's inequality because Fano's inequality relaxed $H(Y|\hat{Y},E=1)$ to $\log(C-1)$, while we relax $H(Y|\hat{Y},E=1)$ to $H(Y|\hat{Y})$ in \eqref{Eq:HYcYhat} \cite{cover_elements_2012}. 

In Figure \ref{Fig:ErrroProbability_and_MI}, we give illustrations of the relation between the error probability lower bound and the mutual information for both a balanced underlying data distribution and an unbalanced data distribution. For the balanced data distribution, we assume uniform marginal distribution for the label, while for the unbalanced data distribution, we assume one of the classes takes probability mass $1-\delta$ and all the other classes share the probability mass $\delta$ evenly. From the figure, we can see that when the mutual information decreases, the error probability will increase, which is consistent with our intuitions. For example, for the case with 100 classes, if the label and the input has zero mutual information, i.e., no dependency between them, we can only draw a random guess and get 0.99 error probability while the lower bound from Theorem \ref{Thm:ErrorProbabilityMI_Bounds} is about 0.9. Figure \ref{Fig:ErrroProbability_and_MI} also shows that under the same setup, the error probability associated with balanced dataset will be larger than that associated with imbalanced dataset. For example, when the mutual information is zero, the $P_{error}$ for balanced dataset is above 0.8 while the $P_{error}$ associated with imbalanced dataset is below 0.8. This is also intuitive since for an imbalanced dataset, we can set the label of the class which contains the most number of examples to all the examples. Since more samples have the correct labels, the error probability cannot be too big. However, for balanced dataset, since all classes have the same number of examples, we cannot get a too small error probability.

\begin{figure*}[!htb]
	\centering
	\begin{subfigure}[b]{0.45\textwidth}
		\centering
		\includegraphics[width=\linewidth]{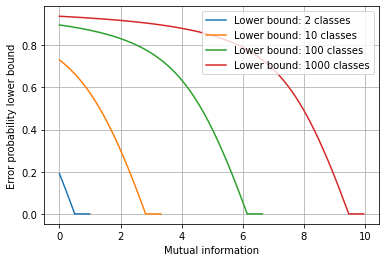}
		\caption{Balanced data distribution}\label{Fig:BalancedData}
	\end{subfigure}%
	~
	\begin{subfigure}[b]{0.45\textwidth}
		\centering
		\includegraphics[width=\linewidth]{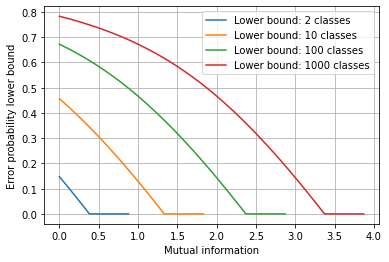}
		\caption{Unbalanced data distribution}\label{Fig:ImbalancedData}
	\end{subfigure}
	\caption{Error probability lower bound and mutual information: : uniform distirbution of labels is assumed for the balanced data distribution. For unbalanced data disrtribution, one class takes probability mass 0.7, while the other classes share the 0.3 evenly.}\label{Fig:ErrroProbability_and_MI}
\end{figure*}

\section{Guarantees for Mutual Information Learning}\label{Sec:MILGuarantees}

We have shown that the mutual information can be used to bound the error probability in Section \ref{Sec:ErrorProbLBound}. In this section, we show that the ground truth mutual information (MI) $I(X;Y)$ can be accurately learned from $\mcS\subset\mbR^n\times[C]$ under certain conditions, and we give the sample complexity for achieving this goal. 

\subsection{Weight Sharing for Learning Mutual Information via Single Neural Network}

From Theorem \ref{Thm:EntEstViaCrossEnt}, we know that
\begin{align}
	H(Y) \leq \inf_{Q_Y}  H(P_Y,Q_Y),\\
	H(Y|X) \leq \inf_{Q_{Y|X}} H(P_{Y|X},Q_{Y|X}),
\end{align}
where the equality holds if and only if $Q_Y=P_Y$ and $Q_{Y|X}=P_{Y|X}$. This allows us to estimate the mutual information defined in \eqref{Eq:MutualInfoDefn} via solving
\begin{align}\label{Eq:OrgnMIEstForm}
	\inf_{Q_Y}  H(P_Y,Q_Y)
    - \inf_{Q_{Y|X}} H(P_{Y|X},Q_{Y|X}), 
\end{align}
and the ground truth MI is
\begin{align}\label{Eq:TruthMI}
	I(X;Y) 
	& = H(P_Y,P_Y) - H(P_{Y|X},P_{Y|X}) \nonumber\\
	& = \mbE_{P_Y}[-\log(P_Y)] - \mbE_{P_{X,Y}}[- \log(P_{Y|X})] \nonumber\\
	& = \sum_{y\in[C]} -P_Y(y) \log(P_Y(y)) - \int_{\mbR^n}  p_X(x) \sum_{y\in[C]} - P_{Y|X}(y|x) \log(P_{Y|X}(y|x)) dx
\end{align}
In \eqref{Eq:OrgnMIEstForm}, we use two sets of parameters $\theta_{Y|X}$ and $\theta_Y$ to learn the conditional entropy and the entropy separately. Under the deep learning paradigma, two different neural networks can be used to acheive this in practice.

Following the weight sharing idea in deep learning community for reducing computational cost, we can use a single neural network with a single set of parameter $\theta\in\mbR^m$ to parameterize both $Q_Y$ and $Q_{Y|X}$. More specifically, we use a neural network with parameter $\theta$ to approximate the conditional distribution, i.e., $Q_{Y|X}(Y|X;\theta)$ , and define $Q_{Y}(Y;\theta)$ as
\begin{align}\label{Eq:Q_Y_via_pX}
	Q_Y(y|\theta):= \int_{\mbR^n} p_X(x) Q_{Y|X}(y|x;\theta) dx, \forall y\in[C].
\end{align} 
When the $p_X$ is not available, we can define $Q_Y(\cdot|\theta)$ by the empirical distribution $\hat{P}_X$, i.e., 
\begin{align}\label{Eq:Q_Y_via_hatpX}
	Q_Y(y|\theta):=\sum_{x\in\mcS_x} \hat{P}(x) Q_{Y|X}(y|x;\theta), \forall y\in[C].
\end{align}
In the ideal situation as we pointed out previously, we find an optimal $\theta^*$ such that $Q_Y(Y;\theta^*) =P_Y(Y)$ and $Q_{Y|X}(Y|X;\theta^*) = P_{Y|X}(Y|X)$.

\subsection{Parameterization and Empirical Minimization for Mutual Information Learned Classifiers}

To ensure efficiency for finding an optimal $\theta^*$ in practice, we search for the optimum only from a domain $\Theta\subset\mbR^m$ instead of searching it in $\mbR^m$. Thus, we calulate the MI by solving
\begin{align}\label{Eq:OrgnMIEstModApproxFormSig}
	I_\Theta 
	& := \inf_{\theta\in\Theta}  H(P_Y,Q_Y(Y;\theta))
	- \inf_{\theta\in\Theta} H(P_{Y|X},Q_{Y|X}(Y|X;\theta))\nonumber \\
	& = \inf_{\theta\in\Theta} \mbE_{P_Y}[-\log(Q_Y(Y;\theta))] 
	- \inf_{\theta\in\Theta} \mbE_{p_{X,Y}}[-\log(Q_{Y|X}(Y|X;\theta))] \nonumber\\
	& = \inf_{\theta\in\Theta} \mbE_{P_Y}\left[-\log\left( \int_{\mbR^n} p_X(x) Q_{Y|X}(Y|x;\theta) dx \right) \right] 
	- \inf_{\theta\in\Theta} \mbE_{p_{X,Y}}[-\log(Q_{Y|X}(Y|X;\theta))],
\end{align}
or 
\begin{align}
	\theta^* := \arg\min_{\theta\in\Theta} (H(P_Y,Q_Y(Y;\theta)), H(P_{Y|X},Q_{Y|X}(Y|X;\theta))),
\end{align}
where $\theta^*$ is an optimal solution to a multi-objective function \cite{miettinen_nonlinear_2012}. The price we pay for such efficiency from constraining $\theta$ is the introduction of model approximation error, i.e., the $\theta$ that corresponds to the ground truth MI can be out of $\Theta$. 

Since the ground truth distribution is not available in practice, we seek to solve an empirical form of \eqref{Eq:OrgnMIEstModApproxFormSig} using sample set $\mcS:=\{(x_i,y_i)\}_{i=1}^N$ from the distribution $p_{X,Y}$. This motivates us to solve 
\begin{align}\label{Eq:OrgnMIEstModEmpriFormSig}
	I_\Theta ^{(N)}
 &:= \inf_{\theta \in\Theta} \mbE_{\hat{P}_Y}[-\log(Q_Y(\theta))] - \inf_{\theta\in\Theta} \mbE_{\hat{p}_{X,Y}}\left[ -\log(Q_{Y|X}(\theta)) \right] \nonumber\\
 & = \inf_{\theta\in\Theta} \sum_{y\in[C]}-\hat{P}(y) \log\left( \sum_{i=1}^N \hat{P}(x_i) Q_{Y|X}(y|x_i;\theta) \right) - \inf_{\theta\in\Theta} \sum_{i=1}^N \hat{P}_X(x_i) \sum_{y\in[C]} - \hat{P}_{Y|X}(y|x_i)\log(Q_{Y|X}(y|x_i;\theta)) \nonumber\\
 &:= \inf_{\theta\in\Theta} \sum_{y\in[C]}-\hat{P}(y) \log\left( \frac{1}{N} \sum_{i=1}^N Q_{Y|X}(y|x_i;\theta) \right) 
  - \inf_{\theta\in\Theta} \frac{1}{N} \sum_{i=1}^N\sum_{y\in[C]} - \hat{P}_{Y|X}(y|x_i)\log(Q_{Y|X}(y|x_i;\theta)),
\end{align}
where we used that assumption that $\|x_i - x_j\|>0, \forall i\neq j$, i.e., uniform empirical distribution with $\hat{P}(x_i)=\frac{1}{N}$. When the empirical conditional distribution $\hat{P}_{Y|X}$ uses one-hot representation, the \eqref{Eq:OrgnMIEstModEmpriFormSig} can be further simplified. The empirical estimation error can occur due to the fact that the sample set $\mcS$ cannot completely characterize the data distribution $p_{X,Y}$.

Similarly, when we use two separate neural networks with two sets of parameters $\theta\in\mbR^m, \gamma\in\mbR^{m'}$ to parameterize $Q_Y$ and $Q_{Y|X}$, respectively, i.e., $Q_Y(Y;\gamma)$ and $Q_{Y|X}(Y|X;\theta)$, we can similarly define
\begin{align}\label{Eq:OrgnMIEstModApproxFormTwo}
	I_{\Gamma, \Theta}
	& := \inf_{\gamma\in\Gamma}  H(P_Y,Q_Y(Y;\gamma))
	- \inf_{\theta\in\Theta} H(P_{Y|X},Q_{Y|X}(Y|X;\theta))\nonumber \\
	& = \inf_{\gamma\in\Gamma} \mbE_{P_Y}[-\log(Q_Y(Y;\gamma))] 
	- \inf_{\theta\in\Theta} \mbE_{p_{X,Y}}[-\log(Q_{Y|X}(Y|X;\theta))] \nonumber\\
	&  = \inf_{\gamma\in\Gamma} \sum_{y\in[C]}- P_Y(y) \log\left( Q_{Y}(y;\gamma)\right) - \inf_{\theta\in\Theta} \int_{\mbR^n} p_X(x) \sum_{y\in[C]} - P_{Y|X}(y|x)\log(Q_{Y|X}(y|x;\theta))dx,
\end{align}
where $\Theta\subset\mbR^m$ and $\Gamma\subset\mbR^{m'}$. An empirical estimation form of $I_{\Theta,\Gamma}$ can be 
\begin{align}\label{Eq:OrgnMIEstModEmpriFormTwo}
	I_{\Gamma, \Theta} ^{(N)}
	&:= \inf_{\gamma\in\Gamma } \mbE_{\hat{P}_Y}[-\log(Q_Y(\gamma))] - \inf_{\theta\in\Theta} \mbE_{\hat{p}_{X,Y}}\left[ -\log(Q_{Y|X}(\theta)) \right] \nonumber\\
	&:= \inf_{\gamma\in\Gamma} \sum_{y\in[C]}-\hat{P}(y) \log\left( Q_{Y}(y;\gamma) \right) 
	- \inf_{\theta\in\Theta} \frac{1}{N} \sum_{i=1}^N\sum_{y\in[C]} - \hat{P}_{Y|X}(y|x_i)\log(Q_{Y|X}(y|x_i;\theta)).
\end{align}

As we can see from the above, the error for learning MI comes mainly from two sources, i.e., the model approximation error and empirical estimation error. The model approximation error can be easily bounded via the universal approximating properties of neural networks \cite{hornik_multilyaer_1989}. In this paper, we focus on the emprical estimation error. In Theorem \ref{Thm:SampleComplexity4EmpiricalEst}, we give the sample complexity for bounding the empirical estimation error by using concentration of measure arguments.

\subsection{Empirical Estimation Guarantees}\label{Sec:EmpiricalEstimationGuarantees}

Before we give the guarantees for accurate learning of MI from empirical sample set $\mcS$, we introduce some useful technical lemmas which will be used in later sections. 

\begin{lemma}\label{Lem:HoeffdingIneqSingleSide}
	(Theorem 2.8 in \cite{boucheron_concentration_2013}) Let $X_1, \cdots, X_n$ be independent random variables such that $X_i$ takes its values in $[a_i,b_i]$ almost surely for all $i\leq n$. Let $S=\sum_{i=1}^n (X_i - \mbE[X_i])$. Then, for every $t>0$, 
	\begin{align}
		P(\{X_1,\cdots,X_n: S\geq t\}) \leq \exp\left( - \frac{2t^2}{\sum_{i=1}^n (b_i - a_i)^2} \right).  
	\end{align}
\end{lemma}

Lemma \ref{Lem:HoeffdingIneqSingleSide} shows that for a sequence of bounded I.I.D. random variables, their empirical mean has a small probability to be much greater than the distribution mean. In Corollary \ref{Cor:HoeffdingIneqDouble}, we show that the empirical mean can neither be much greater nor much smaller than the distribution mean.

\begin{cor}\label{Cor:HoeffdingIneqDouble}
	(Double sided Hoeffiding inequality) Let $X_1, \cdots, X_n$ be independent random variables such that $X_i$ takes its values in $[a_i,b_i]$ almost surely for all $i\leq n$. Let $S=\sum_{i=1}^n (X_i - \mbE[X_i])$. Then, for every $t>0$, 
	\begin{align}
		P\left(\{X_1,\cdots,X_n: |S|\geq t\}\right) \leq 2\exp\left( - \frac{2t^2}{\sum_{i=1}^n (b_i - a_i)^2} \right).  
	\end{align}
\end{cor}

Corollary \ref{Cor:HoeffdingIneqDouble} generalizes the Hoeffiding' inequality from single-sided to double-sided. With Lemma \ref{Lem:HoeffdingIneqSingleSide} and Corollary \ref{Cor:HoeffdingIneqDouble}, we can derive the concentration inequalities for the contional cross entropy  random variable and the marginal distribution random variable as in Lemma \ref{Lem:ConcentrationIneq}.% and \ref{Lem:ConcentrationIneqMarginal}.

\begin{lemma}\label{Lem:ConcentrationIneq}
	(Concentration Inequality for Conditional Cross Entropy) We consider a set of random variable pairs $\mcS:\{(X_i,Y_i)\}_{i=1}^N$ with each $(X_i,Y_i)$ I.I.D. according to $p_{X,Y}$ in $\mbR^n\times[C]$, and define 
	\begin{align}
		D:= \frac{1}{N} \sum_{i=1}^N\log(Q_{Y|X}(Y_i|X_i; \theta))   - \mbE_{P_{Y,X}} \left[ \log(Q_{Y|X}(Y|X; \theta)) \right],
	\end{align}
where $\theta\in\Theta_c$, $\Theta_c$ is a countable set, and $Q_{Y|X}(Y|X;\theta)$ is a function of $X, Y$ with parameters $\theta$. Assume $Q_{Y|X}(y|x;\theta)\geq P^\#, \forall x,y,\theta$ where $P^\#>0$ is a constant. Then, we have 
	\begin{align}
		P(\{\mcS: \max_{\theta\in\Theta_c} |D| \geq t\}) \leq 2|\Theta_c| \exp\left( -\frac{2Nt^2}{(\log(P^\#))^2} \right)
	\end{align}
where $t>0$ is a constant.  
\end{lemma}

Theorem \ref{Lem:ConcentrationIneq} shows that when we randomly sample data points $(X_i,Y_i)$ from $p_{X,Y}$ to form a dataset $\mcS$, then with high probability over the dataset $\mcS$, the empirical mean $\frac{1}{N} \sum_{i=1}^N\log(Q_{Y|X}(Y_i|X_i; \theta))$ will be concentrated around the distribution mean $\mbE_{P_{Y,X}} \left[ \log(Q_{Y|X}(\theta_j)) \right]$. When $Q_{Y|X}$ becomes the $P_{Y|X}$, the negative of the distribution mean is essentially the conditional entropy $H(Y|X)$, and Theorem \ref{Lem:ConcentrationIneq} shows that the empirical estimation will be close the the truth conditional entropy for large $N$. Simlarly, we can derive the concentration inequality for cross entropy, and the results are presented in Lemma \ref{Lem:ConcentrationIneqCrossEnt}.

\begin{lemma}\label{Lem:ConcentrationIneqCrossEnt}
	(Concentration Inequality for Cross Entropy) We consider a set of random variable pairs $\mcS:\{(Y_i)\}_{i=1}^N$ with each $Y_i$ independently and identically distributed according to $p_{Y}$ in $\mbR^n\times[C]$, and define 
	\begin{align}
		D:= \frac{1}{N} \sum_{i=1}^N\log(Q_{Y}(Y_i; \gamma))   - \mbE_{P_{Y}} \left[ \log(Q_{Y}(\cdot; \gamma)) \right],
	\end{align}
	where $\gamma\in\Gamma_c$, $\Gamma_c$ is a countable set, and $Q_{Y}(Y;\gamma)$ is a function of $Y$ with parameters $\gamma$. Assume $Q_{Y}(y;\gamma)\geq P^\#, \forall y, \gamma$ where $P^\#>0$ is a constant. Then, we have 
	\begin{align}
		P(\{\mcS: \max_{\theta\in\Theta_c} |D| \geq t\}) \leq 2|\Gamma_c| \exp\left( -\frac{2Nt^2}{(\log(P^\#))^2} \right)
	\end{align}
	where $t>0$ is a constant.  
\end{lemma}

\begin{proof}\label{Proof:ConcentrationIneqCrossEnt}
	(of Lemma \ref{Lem:ConcentrationIneqCrossEnt}) We can follow similar arguments as in the proof of Lemma \ref{Lem:ConcentrationIneq} to get Lemma \ref{Lem:ConcentrationIneqCrossEnt}, and we leave it out here.
	
\end{proof}

To get the sample complexity results, we need two more technical lemma as presented in Lemma \ref{Lem:IndpdtIneq} which gives upper bound of the distance between two infimum over the same domain, and Lemma \ref{Lem:SubspaceCoveringNumber} which gives the covering number for constructing a set of balls for covering a uncountable set. Lemma \ref{Lem:SubspaceCoveringNumber} will be used to construct countable set for approximating a noncountable set.

\begin{lemma}\label{Lem:IndpdtIneq}
	(Upper Bound of Minimums Difference) For two arbitrary functions $f(x)$ and $g(x)$ defined over the same domain $\mathcal{X}$, we have 
	\begin{align}
		\left| \inf_{x\in\mathcal{X}} f(x) - \inf_{x\in\mathcal{X}} g(x) \right| \leq \sup_{x\in\mathcal{X}} |f(x) - g(x)|.
	\end{align}
\end{lemma}

\begin{proof}\label{Proof:IndpdtIneq}
	(of Lemma \ref{Lem:IndpdtIneq}) Assume $x_f:=\arg\inf_{x\in\mathcal{X}} f(x)$ and $x_g:=\arg\inf_{x\in\mathcal{X}} g(x)$, then the Lemma \ref{Lem:IndpdtIneq} is obvious since $|f(x_f) - g(x_g)| \leq \sup_{x\in\mathcal{X}} |f(x) - g(x)|.$
\end{proof}

\begin{lemma}\label{Lem:SubspaceCoveringNumber}
	(Covering Number of Subspace in $\mbR^m$, Example 27.1 in \cite{shalev-shwartz_understanding_2014}) Suppose that $A\subset\mbR^m$, let $c=\max_{a\in A}\|a\|$, and assume that $A$ lies in a $d$-dimensional subspace of $\mbR^m$. Then, $N(r,A) \leq (2c\sqrt{d}/r)^d$.
\end{lemma}

We now derive the sample complexity for achieving an accurate estimate of $I_{\Gamma,\Theta}$ in \eqref{Eq:OrgnMIEstModApproxFormTwo} by solving \eqref{Eq:OrgnMIEstModEmpriFormTwo} with the samples in the dataset $\mcS:\{(X_i,Y_i)\}_{i=1}^N$ independently and identically distributed (IID) and follow $p_{X,Y}$. 

\begin{thm}\label{Thm:SampleComplexity4EmpiricalEst}
	(Sample Complexity for Estimation Error Bound) We consider a joint distirbution $p_{X,Y}$ in $\mbR^n\times[C]$ where $X\in\mbR^n$ is a continuous random vector, and $Y\in[C]$ is a discrete random variable. We define $I_{\Gamma,\Theta}$ associated with $p_{X,Y}$ similar to \eqref{Eq:OrgnMIEstModApproxFormTwo}, i.e., 
\begin{align}
	I_{\Gamma, \Theta}
= \inf_{\gamma\in\Gamma} \mbE_{P_Y}[- \log\left( Q_{Y}(Y;\gamma)\right)] - \inf_{\theta\in\Theta} \mbE_{p_{X,Y}}[-\log(Q_{Y|X}(Y|X;\theta))],
\end{align}
	where $Q_{Y|X}(Y|X;\theta)$ is a neural network with parameters $\theta\in\Theta\subset\mbR^m$ which predicts the conditional probability of $Y$ conditioning on $X$, and $Q_Y(Y;\gamma)$ is another neural network with parameters $\gamma\in\Gamma\subset\mbR^{m'}$ which predicts the marginal probability of $Y$. Assume that we are given a set of random examples $\mcS: \{(X_i,Y_i)\}_{i=1}^N$ such that $(X_i,Y_i), i=1\cdots,N$ are I.I.D. and follow $p_{X,Y}$. Define $I_{\Gamma,\Theta}^{(N)}$ similar to that in \eqref{Eq:OrgnMIEstModEmpriFormTwo}, i.e., 
\begin{align}
	I_{\Gamma, \Theta} ^{(N)}
	= \inf_{\gamma\in\Gamma } \frac{1}{N} \sum_{i=1}^N - \log(Q_{Y}(Y_i; \gamma)) - \inf_{\theta\in\Theta} \frac{1}{N} \sum_{i=1}^N - \log(Q_{Y|X}(Y_i|X_i; \theta)).
\end{align}
Assume both $\Theta$ and $\Gamma$ are compact sets, and bounded, i.e., $\|\theta\|\leq M_\theta$ and $\|\gamma\|\leq M_\gamma$ where $M_\theta>0, M_\gamma>0$ are constants. We assume both $Q_{Y}(\cdot;\gamma)$ and $Q_{Y|X}(\cdot;\theta)$ are lower bounded by $P^\#$, and they are Lipschitz continuous with respect to $\theta$ for all $y\in[C]$ and all $(x,y)\in\mbR^n\times[C]$, and the Lipschitz constants are $L_\gamma>0$ and $L_\theta>0$, respectively. Then, when $N
\geq \frac{2\log\left(\frac{1}{\epsilon} \right) \left( P^\# \log(P^\#) \right)^2}{\left( \delta P^\# - 4^{\frac{1+m'}{m'}} L_\gamma M_\gamma \sqrt{m'} - 4^{\frac{1+m}{m}} L_\theta M_\theta \sqrt{m} \right)^2}$, we have
\begin{align}
	P\left( \left\{\mcS: \left| I_\Theta^{(N)} - I_\Theta\right| \leq \delta \right\} \right) \geq 1-\epsilon.
\end{align} 
	
\end{thm} 

Theorem \ref{Thm:SampleComplexity4EmpiricalEst} essentially gives the sample complexity for learning mutual information $I_{\Gamma,\Theta}$ from empirical samples via $I_{\Gamma,\Theta}^{(N)}$ to acheive arbitrary precision $\delta$ with arbitrary probability over a dataset $\mcS:\{(X_i,Y_i)\}_{i=1}^N$ where $(X_i,Y_i)$ are I.I.D. and follow $p_{X,Y}$. The above sample complexity is also consistent with our intuitions, e.g., if we want to achieve higher precision (smaller $\delta$) with higher probability (smaller $\epsilon$), we need higher sample complexity (lower bound of $N$ will increase).  We want to point out that Theorem \ref{Thm:SampleComplexity4EmpiricalEst} gives the sample complexity for mutual information learning using two neural networks without weight sharing, similar sample complexity bounds can also be established for the case where a single neural network is used, and we leave it for future work.

\begin{table*}[h]
	\centering
	\begin{tabular}{|l | c | c| c| c| c|}
		\hline 
		top-1 accuracy & celLoss & celLoss$+$LSR & celLoss$+$CP & celLoss$+$LC & milLoss (ours)\\
		\hline
		MLP & $0.934\pm0.000$ & 0.934$\pm$0.001 & 0.930$\pm$0.002 & 0.932$\pm$0.000 & $0.945\pm0.001$\\
		\hline 
		CNN & $0.981\pm0.000$ & 0.982$\pm$0.001 & 0.980$\pm$0.000 & 0.980$\pm$0.001 & $0.984\pm0.001$\\
		\hline 
	\end{tabular}
	\caption{Top-1 accuracy on MNIST dataset associated with different models which are trained with different loss objective function.}\label{Tab:ProofOfConcept_MNIST}
\end{table*}

\begin{table*}[h]
	\centering
	\begin{tabular}{|l | c | c| c| c| c|}
		\hline 
		top-1 accuracy & celLoss & celLoss$+$LSR & celLoss$+$CP & celLoss$+$LC & milLoss (ours) \\
		\hline
		GoogLeNet & $0.803\pm0.006$ & 0.766$\pm$0.004 & 0.784$\pm$0.006 & 0.791$\pm$0.002 & $0.866\pm0.000$\\
		\hline
		ResNet-18 & $0.732\pm0.002$ & 0.679$\pm$0.001 & 0.703$\pm$0.006 & 0.726$\pm$0.005 & $0.832\pm0.004$\\
		\hline
		MobileNetV2 &0.676$\pm$0.007 & 0.647$\pm$0.005 & 0.660$\pm$0.008 & 0.677$\pm$0.004 & 0.762$\pm$0.006\\
		\hline 
		EfficientNet-B0 & 0.524$\pm$0.013 & 0.510$\pm$0.006 & 0.503$\pm$0.008 & 0.524$\pm$0.009 & 0.682$\pm$0.006\\
		\hline 
		%		ResNeXt29\_2x64d & 0.707$\pm$0.004 & 0.664$\pm$0.001 &0.689$\pm$0.009 & 0.698$\pm$0.007 & 0.628$\pm$0.066\\
		%		\hline 
		ShuffleNetV2 & 0.604$\pm$0.003 & 0.554$\pm$0.005 & 0.578$\pm$0.005 & 0.600$\pm$0.004 & 0.677$\pm$0.003\\
		\hline
	\end{tabular}
	\caption{Top-1 accuracy on CIFAR-10 dataset associated with different models which are trained with different loss objective function.}\label{Tab:ProofOfConcept_CIFAR10}
\end{table*}

\begin{table}[h]
	\small
	\centering
	\begin{tabular}{|l | c | c| c| c| c|}
		\hline 
		top-1 (top-5) accuracy & celLoss & celLoss$+$LSR & celLoss$+$CP & celLoss$+$LC & milLoss (ours)\\
		\hline
		DenseNet-121 &0.502 (0.781) & 0.497 (0.766) & 0.485 (0.766) & 0.509 (0.788) & 0.627 (0.869)\\ % D:\RP202202-now_MIC\MILC\checkpoint\CIFAR-100\densenet121\ProOfCocept
		\hline 
		Inception-Resnet-V2 & 0.488 (0.771) & 0.451 (0.715) & 0.434 (0.711) & 0.486 (0.760) & 0.554 (0.824)\\
		\hline
		Inception-V3 &0.503 (0.770) & 0.472 (0.735) & 0.484 (0.749) & 0.510 (0.769) & 0.624 (0.859)\\
		\hline 
		PreAct-Resnet-18 &0.382 (0.684) & 0.383 (0.700) & 0.377 (0.693) & 0.367 (0.688) & 0.480 (0.766)\\
		\hline
		RreAct-Resnet-101 & 0.404 (0.689) & 0.385 (0.640) & 0.386 (0.683) & 0.404 (0.700) & 0.490 (0.766)\\
		\hline 
		ResNet-34 &0.408 (0.701) & 0.405 (0.659) & 0.400 (0.683) & 0.412 (0.697) & 0.520 (0.790)\\
		\hline
		ResNet-50 &0.365 (0.644) & 0.357 (0.611) & 0.349 (0.632) & 0.359 (0.642) & 0.492 (0.769)\\
		\hline
		VGG-16 &0.430 (0.699) & 0.400 (0.663) & 0.405 (0.6828) & 0.420 (0.694) & 0.524 (0.780)\\
		\hline
		%		Swin Transformer &XXX & XXX & XXX & XXX & XXX\\
		%		\hline
	\end{tabular}
	\caption{Top-1 and top-5 accuracy on CIFAR-100 dataset associated with different models which are trained with different loss objective function.}\label{Tab:ProofOfConcept_CIFAR100}
\end{table}

\section{Mutual Information and Error Probability Bound of Binary Classification Data Model in $\mbR^n$}\label{Sec:BinaryClassificationDataModel}

In this section, we derive the mutual information bounds for a binary classification data model $P_{X,Y}$ in $\mbR^n\times\{-1,1\}$. In the data generation process, we first sample a label $y\in\{-1,1\}$, and then a corresponding feature $x$ from a Gaussian distribution. We model the feature as a Gaussian random vector $X$ with sample space $\mbR^n$, i.e., 
\begin{align}\label{Defn:BinaryClassificationDataModel}
	P(Y=-1) = q, P(Y=1) = 1-q, \nonumber \\
	p(X=x|y) = \frac{1}{\sqrt{|2\pi\Sigma|}} \exp\left(-\frac{(x-y\mu)^T\Sigma^{-1}(x-y\mu)}{2}\right),
\end{align}
where $\mu\in\mbR^n$ is a mean vector, and $\Sigma\in\mbR^{n\times n}$ is a positive semidefinite matrix. In this data model, we can derive lower and upper bounds of the mutual information $I(X;Y)$, and the results are presented in Theorem \ref{Thm:BinaryClassificationDataModelTruthMI}. Before getting to Theorem \ref{Thm:BinaryClassificationDataModelTruthMI}, we first derive the expectation of quadratic forms of Gaussian random vector in Lemma \ref{Lem:QuadraticFormExpectation}, which will be used for deriving the bounds of mutual information in the data model \eqref{Defn:BinaryClassificationDataModel}.

\begin{lemma}\label{Lem:QuadraticFormExpectation}
	(Expectation of Quadratic Form of Gaussian Random Vector) For a Gaussian random vector $X\in\mbR^n$ following $\mcN(\mu,\Sigma)$, we have
	\begin{align}\label{Eq:QuadraticFormExpectation}
		\mbE_{X}\left[ X^TAX \right] = \tr(A\Sigma) + \mu^T\Sigma^{-1}\mu, \nonumber\\
		\mbE_{X}\left[ (X-\mu)^TA(X-\mu) \right] = \tr(A\Sigma), \nonumber\\
		\mbE_{X}\left[ (X+\mu)^T A (X+\mu) \right] = \tr(A\Sigma) + 4\mu^T\Sigma \mu,
	\end{align}
where $A\in\mbR^{n\times n}$ is a square matrix.
\end{lemma}

From Lemma \ref{Lem:QuadraticFormExpectation}, we know that when $A=\Sigma^{-1}$, we have 
\begin{align}
\mbE_{X}\left[ X^T\Sigma^{-1}X \right] = n + \mu^T\Sigma^{-1}\mu,\\
		\mbE_{X}\left[ (X-\mu)^T\Sigma^{-1}(X-\mu) \right] = n,\\
\mbE_{X}\left[ (X+\mu)^T \Sigma^{-1} (X+\mu) \right] = n + 4\mu^T\Sigma \mu.
\end{align}
When $X\sim\mcN(-\mu,\Sigma)$, $\mbE_X[X^TAX] = \tr(A\Sigma) + \mu^TA\mu$. When $X\sim\mcN(-\mu,\Sigma)$, $\mbE_{X}[(X-\mu)^TA(X-\mu)] = \tr(A\Sigma)+4\mu^T\Sigma\mu$.

\begin{thm}\label{Thm:BinaryClassificationDataModelTruthMI}
	(Mutual Information of Binary Classification Dataset Model) For the data model with distribution defined in \eqref{Defn:BinaryClassificationDataModel}, we have the mutual information $I(X;Y)$ satisfying
	\begin{align}\label{Eq:MIinBinaryClassificationModel}
2\min(q,1-q)\mu^T\Sigma^{-1}\mu \leq I(X;Y) \leq 4q(1-q)\mu^T\Sigma^{-1}\mu.
	\end{align}
\end{thm}

From Theorem \ref{Thm:BinaryClassificationDataModelTruthMI}, we can see that for the mutual information $I(X;Y)$, the maximum upper bound is acheived when $q=0.5$. For a simplified case in $\mbR$ with $\mu=1$ and variance $\sigma^2=1$, when the variance becomes bigger, the two distributions $\mcN(1,\sigma^2)$ and $\mcN(-1,\sigma^2)$ get closer to each other. Thus, conditioning on $X$ can give very little information about $Y$, making it difficult to differentiate the two class labels. In Figure \ref{Fig:XConditionalDistributionsUnderDifferentVariance}, we give illustrations for this phenomenon. As we can see in Figure \ref{Fig:XConditionalDistributionsUnderDifferentVariance}, as the $\sigma^2$ increases from 1 to 100, the two conditional distributions $\mcN(1,\sigma^2)$ and $\mcN(-1,\sigma^2)$ get closer to each other. This results in that less information is revealed about $Y$ when we condition on $X$ with larger variance $\sigma^2$. We also give illustrations of the mutual information bounds for the data distribution in $\mbR$ in Figure \ref{Fig:MutualInformationBounds}. Theorem \ref{Thm:BinaryClassificationDataModelTruthMI} can be easily generalized to multi-class classification in $\mbR^n$, and we leave this for future work.

\begin{figure*}[!htb]
	\centering
	\begin{subfigure}[b]{0.5\textwidth}
		\centering
		\includegraphics[width=\linewidth]{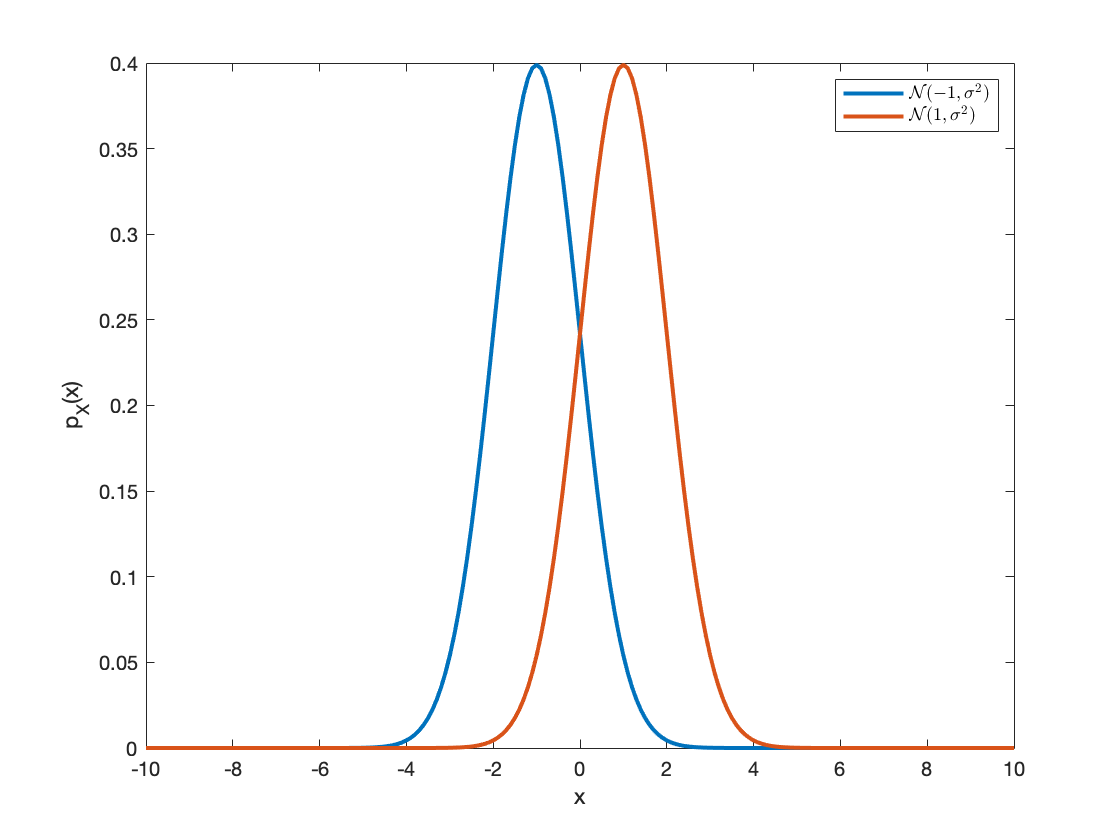}
		\caption{Variance $\sigma^2=1$}
	\end{subfigure}%
	~
	\begin{subfigure}[b]{0.5\textwidth}
		\centering
		\includegraphics[width=\linewidth]{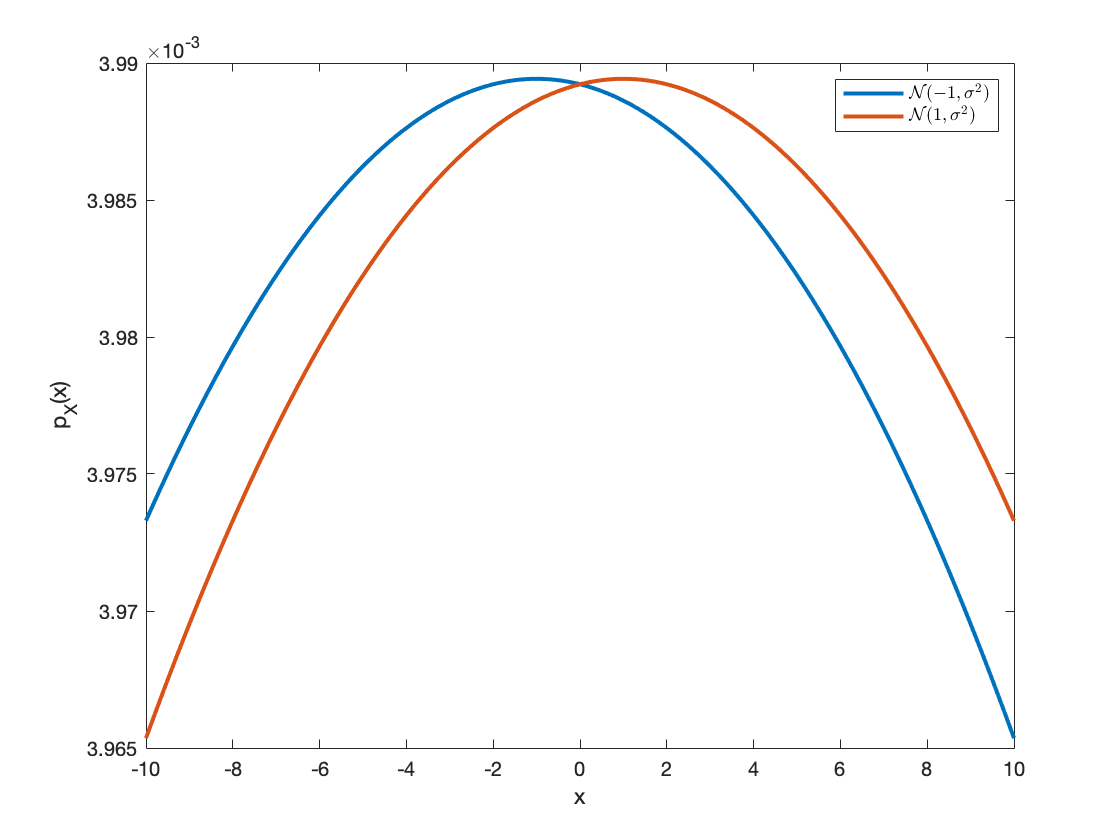}
		\caption{Variance $\sigma^2=100$}
	\end{subfigure}
	\caption{Conditional distribution $p_{X|Y}$ with different variance $\sigma^2$.}\label{Fig:XConditionalDistributionsUnderDifferentVariance}
\end{figure*}

\begin{figure*}[!htb]
	\centering
	\begin{subfigure}[b]{0.5\textwidth}
		\centering
		\includegraphics[width=\linewidth]{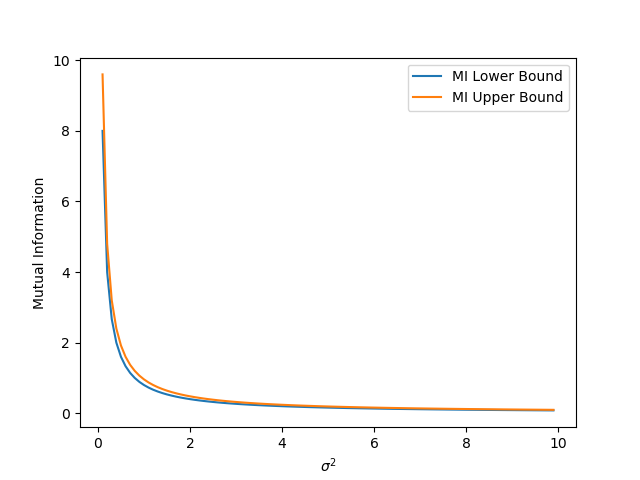}
		\caption{$q=0.4$}
	\end{subfigure}%
	~
	\begin{subfigure}[b]{0.5\textwidth}
		\centering
		\includegraphics[width=\linewidth]{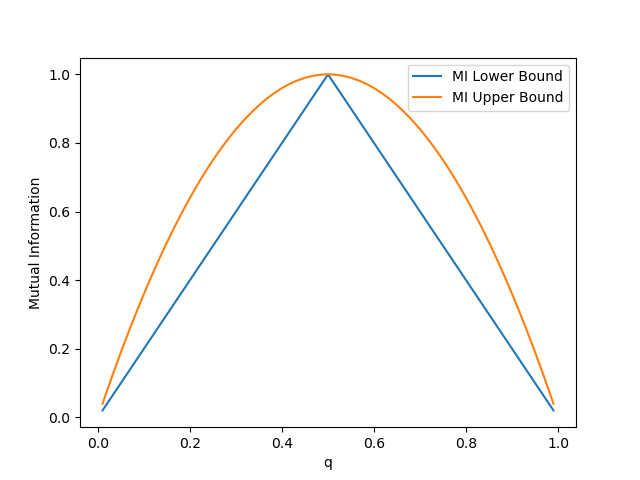}
		\caption{$\sigma^2=1$}
	\end{subfigure}
	\caption{Mutual information bounds of binary classification data model in \eqref{Defn:BinaryClassificationDataModel}.}\label{Fig:MutualInformationBounds}
\end{figure*}

Based on Theorem \ref{Thm:EntEstBound} and \ref{Thm:BinaryClassificationDataModelTruthMI}, we derive a error probability lower bound for the binary classification data model  \eqref{Defn:BinaryClassificationDataModel} in Corollary \ref{Corola:ErrorProbabilityBound}.

\begin{figure}[!htb]
	\centering
	\begin{subfigure}[b]{0.45\textwidth}
		\centering
		\includegraphics[width=\linewidth]{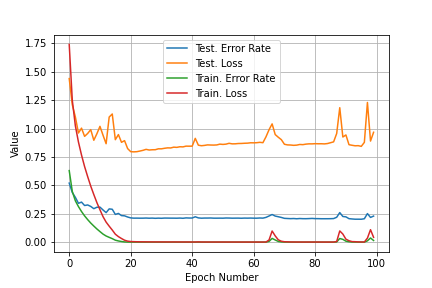}
		\caption{celLoss}\label{Fig:CIFAR10_GoogLeNet_cel}
	\end{subfigure}%
	~
	\begin{subfigure}[b]{0.45\textwidth}
		\centering
		\includegraphics[width=\linewidth]{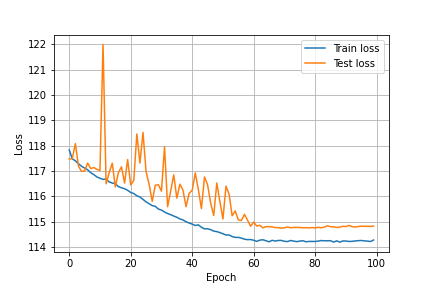}
		\caption{milLoss}\label{Fig:CIFAR10_GoogLeNet_milLoss}
	\end{subfigure}
	\begin{subfigure}[b]{0.45\textwidth}
		\centering
		\includegraphics[width=\linewidth]{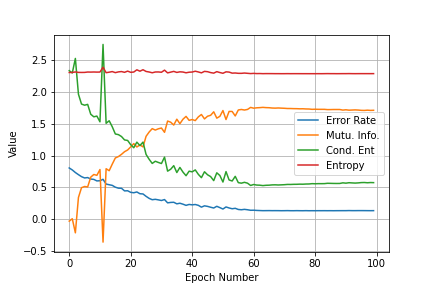}
		\caption{milLoss}\label{Fig:CIFAR10_GoogLeNet_milInformation}
	\end{subfigure}%
	~
	\begin{subfigure}[b]{0.45\textwidth}
		\centering
		\includegraphics[width=\linewidth]{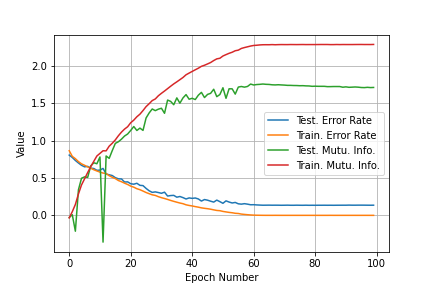}
		\caption{milLoss}\label{Fig:CIFAR10_GoogLeNet_milErrorRate_MI}
	\end{subfigure}
	\begin{subfigure}[b]{0.45\textwidth}
		\centering
		\includegraphics[width=\linewidth]{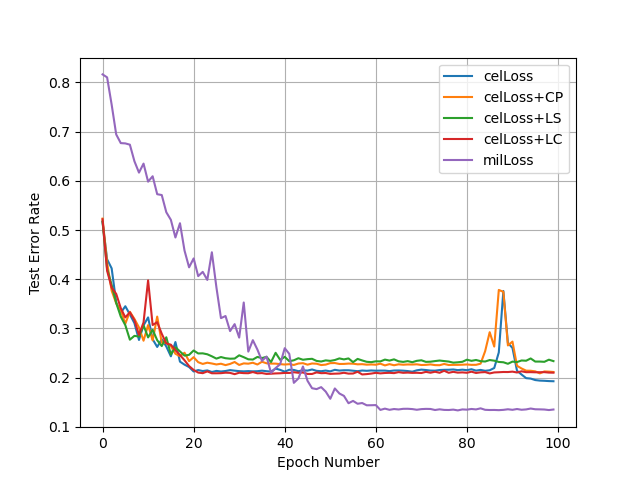}
		\caption{Error rates}\label{Fig:CIFAR10_GoogLeNet_milAcel_ErrorRate}
	\end{subfigure}
	\caption{GoogLeNet on CIFAR-10. \ref{Fig:CIFAR10_GoogLeNet_cel}: error rate and loss during training and test at different epochs. \ref{Fig:CIFAR10_GoogLeNet_milLoss}: loss during training and testing at different epochs. \ref{Fig:CIFAR10_GoogLeNet_milInformation}: error rate, mutual information, label conditional entropy, and label entropy during test at different epochs. \ref{Fig:CIFAR10_GoogLeNet_milErrorRate_MI}: mutual information and error rate during training and testing at different epoch. \ref{Fig:CIFAR10_GoogLeNet_milAcel_ErrorRate}: testing error rate curves associated with different loss functions.}\label{Fig:CIFAR10_GoogLeNet}
\end{figure}

\begin{cor}\label{Corola:ErrorProbabilityBound}
	For the data distribution defined in \eqref{Defn:BinaryClassificationDataModel}, we assume the $Y\to X\to \hat{Y}$ forms a Markov chain where $\hat{Y}$ is the prediction from a classifier, and we follow the learning process in Figure \ref{Fig:ErrorProbability} to learn the classifier. Then, the error probability for an arbitrary classifier must satisfy
	\begin{align*}
	\max\left( 0,\frac{2+H(Y) - 4q(1-q)\mu^T\Sigma^{-1}\mu - a}{4} \right)
	\leq P_{error},
\end{align*}
where $a:=\sqrt{(H(Y) - I(X;Y) - 2)^2 + 4}$.
\end{cor}

We can simply plug in the bounds of MI from Theorem \ref{Thm:BinaryClassificationDataModelTruthMI} to Theorem \ref{Thm:EntEstBound} to get Corollary \ref{Corola:ErrorProbabilityBound}. This is also intuitive. For example, when we consider the model in $\mbR$ with mean $\mu\in\mbR$ and $\sigma^2\in[0,+\infty)$, we have $\max\left( 0,\frac{2+H(Y) - 4q(1-q)\mu^T\Sigma^{-1}\mu - a}{4} \right)
\leq P_{error}$. When we increase $\mu$ (the distributions from two classes become farther from each other) and decrease $\sigma^2$ (the distributions from two classes become more concentrated), we are more likely to classify them correctly, thus a lower error probability. We also want to mention that the results from Theorem \ref{Thm:EntEstBound} applies to arbitrary data distribution and arbitrary learning algorithms, but Corollary \ref{Corola:ErrorProbabilityBound} only applies for the binary classification data model in \eqref{Defn:BinaryClassificationDataModel} with arbitrary learning algorithms.

\begin{figure}[!htb]
	\centering
	\begin{subfigure}[b]{0.45\textwidth}
		\centering
		\includegraphics[width=\linewidth]{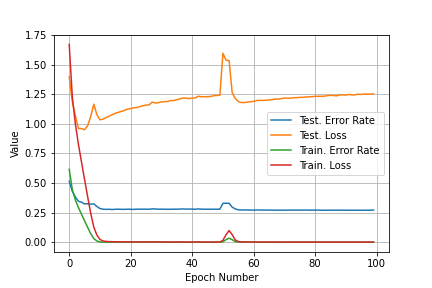}
		\caption{celLoss}\label{Fig:CIFAR10_ResNet18_cel}
	\end{subfigure}%
	~
	\begin{subfigure}[b]{0.45\textwidth}
		\centering
		\includegraphics[width=\linewidth]{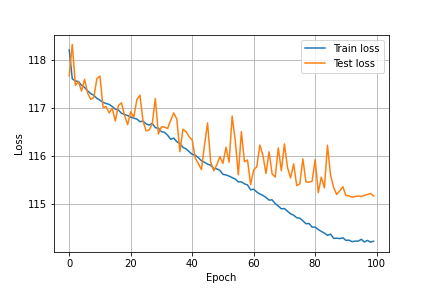}
		\caption{milLoss}\label{Fig:CIFAR10_ResNet18_milLoss}
	\end{subfigure}
	\begin{subfigure}[b]{0.45\textwidth}
		\centering
		\includegraphics[width=\linewidth]{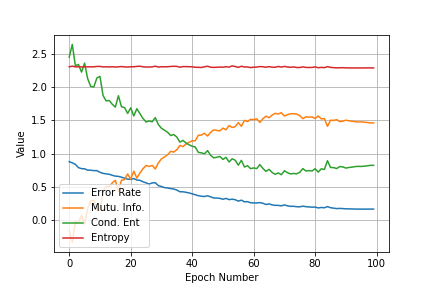}
		\caption{milLoss}\label{Fig:CIFAR10_ResNet18_milInformation}
	\end{subfigure}%
	~
	\begin{subfigure}[b]{0.45\textwidth}
		\centering
		\includegraphics[width=\linewidth]{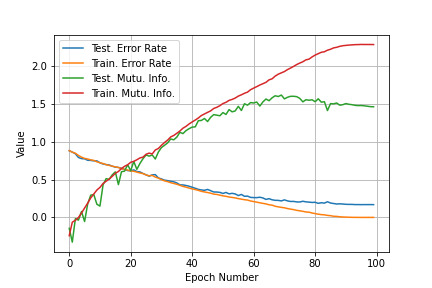}
		\caption{milLoss}\label{Fig:CIFAR10_ResNet18_milErrorRate_MI}
	\end{subfigure}
	\begin{subfigure}[b]{0.45\textwidth}
		\centering
		\includegraphics[width=\linewidth]{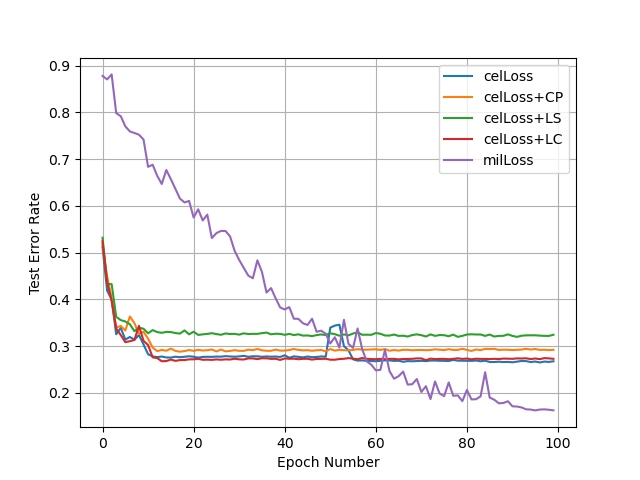}
		\caption{Error rates}\label{Fig:CIFAR10_ResNet18_milAcel_ErrorRate}
	\end{subfigure}
	\caption{ResNet18 on CIFAR-10. \ref{Fig:CIFAR10_ResNet18_cel}: error rate and loss during training and test at different epochs. \ref{Fig:CIFAR10_ResNet18_milLoss}: loss during training and testing at different epochs. \ref{Fig:CIFAR10_ResNet18_milInformation}: error rate, mutual information, label conditional entropy, and label entropy during test at different epochs. \ref{Fig:CIFAR10_ResNet18_milErrorRate_MI}: mutual information and error rate during training and testing at different epoch. \ref{Fig:CIFAR10_ResNet18_milAcel_ErrorRate}: testing error rate curves associated with different loss functions.}\label{Fig:CIFAR10_ResNet18}
\end{figure}

\section{Experimental Results}\label{Sec:ExperimentalResults}

In this section, we present experimental results from multi-class classification on the MNIST, CIFAR-10 , and CIFAR-100 to validate our theory \cite{he_deep_2015}. All our experiments are conducted on a Windows machine with Intel Core(TM) i9 CPU @ 3.7GHz, 64Gb RAM, and 1 NVIDIA RTX 3090 GPU card. %Our implementation will be released at \underline{\href{https://github.com/JAMES-YI/MILC}{MILC}} in github after the reviewing process.

{\bf Implementations and Configurations of Baseline Models} The MNIST classification task is in $[0,255]^{784}\times\{0,1,\cdots,9\}$, and our goal is to classify a given hand-written digital image into one of the 10 classes. The MNIST dataset has 60,000 examples for training, and another 10,000 examles for testing. We use both a multiple layer perceptron (MLP) and a convolutional neural network (CNN) to train two different classifiers by using the regularized form of the mutual information learning loss in  \eqref{Eq:MIAlterEstParametericMethodEquivalentRegluarizedForm}, the conditional entropy learning loss in \ref{Eq:CondEntLoss}, and also its regularized forms as discussed in previous sections, i.e., celLoss with label smoothing regularization (LSR), celLoss with confidence penalty regularization (CP), celLoss with label correction regularization (LC) \cite{szegedy_going_2015,pereyra_regularizing_2017,wang_proselflc_2021}. The regularization parameter $\lambda_{ent}$ is set to be $1e5$. The MLP is a 3-layer fully connected neural network with 64, 64, and 10 neurons in each layer. All the layers except the final layer use a relu activation function. The CNN is a 4-layer neural network with 2 convolutional layers followed by 2 fully connected layers. The first convolutional layer has 10 kernels of size $5\times5$, and the second convolutional layer has 20 kernels of size $5\times5$. A maxpooling layer with stride 2 is applied after each convolutional layer. The two fully connected layers have 50 and 10 neurons, respectively, and the first fully connected layer uses relu activation function. We use SGD optimizer with a constant learning rate 1e-3 and a momentum 0.9, and we do not use weight decay. During training, we use a batch size of 512, and no data augmentation is used. When the classifiers are trained with milLoss, the regularization parameter $\lambda_{ent}=50$. Each model is trained for 77 epochs. 

\begin{figure}[!htb]
	\centering
	\begin{subfigure}[b]{0.45\textwidth}
		\centering
		\includegraphics[width=\linewidth]{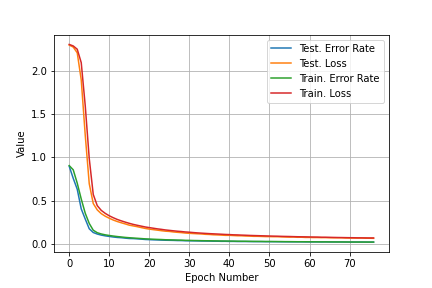}
		\caption{celLoss}\label{Fig:MNIST_CNN_cel}
	\end{subfigure}%
	~
	\begin{subfigure}[b]{0.45\textwidth}
		\centering
		\includegraphics[width=\linewidth]{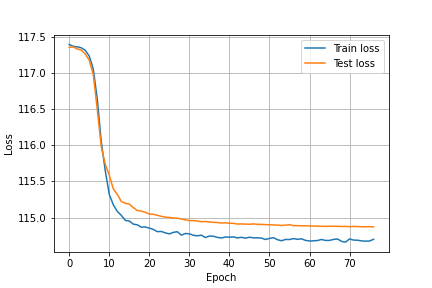}
		\caption{milLoss}\label{Fig:MNIST_CNN_milLoss}
	\end{subfigure}
	~
	\begin{subfigure}[b]{0.45\textwidth}
		\centering
		\includegraphics[width=\linewidth]{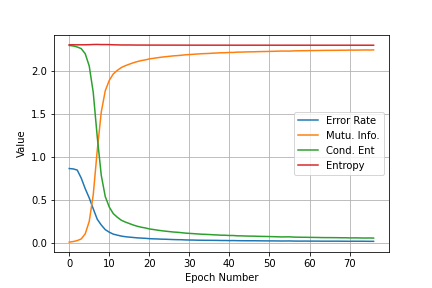}
		\caption{milLoss}\label{Fig:MNIST_CNN_milInformation}
	\end{subfigure}%
	\\
	\begin{subfigure}[b]{0.45\textwidth}
		\centering
		\includegraphics[width=\linewidth]{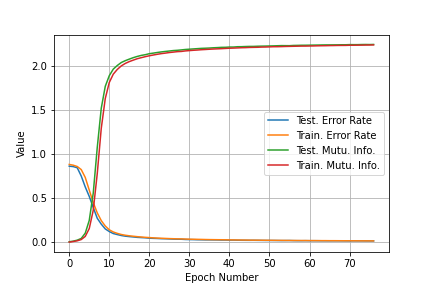}
		\caption{milLoss}\label{Fig:MNIST_CNN_milErrorRate_MI}
	\end{subfigure}
	\begin{subfigure}[b]{0.45\textwidth}
		\centering
		\includegraphics[width=\linewidth]{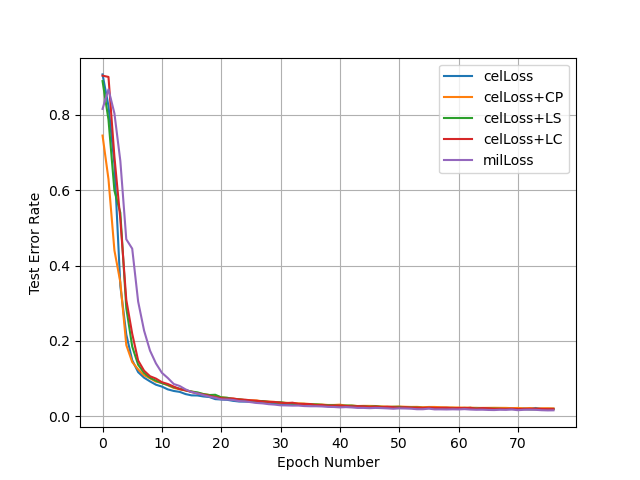}
		\caption{Error rates}\label{Fig:MNIST_CNN_milAcel_ErrorRate}
	\end{subfigure}
	\caption{CNN on MNIST. \ref{Fig:MNIST_CNN_cel}: error rate and loss during training and test at different epochs. \ref{Fig:MNIST_CNN_milLoss}: loss during training and testing at different epochs. \ref{Fig:MNIST_CNN_milInformation}: error rate, mutual information, label conditional entropy, and label entropy during test at different epochs. \ref{Fig:MNIST_CNN_milErrorRate_MI}: mutual information and error rate during training and testing at different epoch. \ref{Fig:MNIST_CNN_milAcel_ErrorRate}: testing error rate curve associated with different loss function.}\label{Fig:MNIST_CNN}
\end{figure}

Similarly, the CIFAR-10 classification task is in $[0,255]^{3072}\times\{0,1,\cdots,9\}$, and we want to assign each image a label. The dataset has 50,000 images for training and another 10,000 images for testing. We use the ResNet-18, GoogLeNet, MobileNetV2, EfficientNetB0, ResNeXt29\_2x64d, and ShuffleNetV2 \cite{he_deep_2015,szegedy_going_2015,sandler_mobilenetv2_2019,tan_efficientnet_2020,xie_aggregated_2017,ma_shufflenet_2018} to train classifiers by the conditional entropy learning loss minimization and the mutual information learning loss minimization, respectively. We use SGD optimizer with a constant learning rate 1e-3 and a momentum 0.9. The batch size is set to be 256. The $\lambda_{ent}$ takes value 5e1 when the mutual information learning loss is used to train the models. Each model is trained for 100 epochs. 

\begin{figure}[!htb]
	\centering
	\begin{subfigure}[b]{0.45\textwidth}
		\centering
		\includegraphics[width=\linewidth]{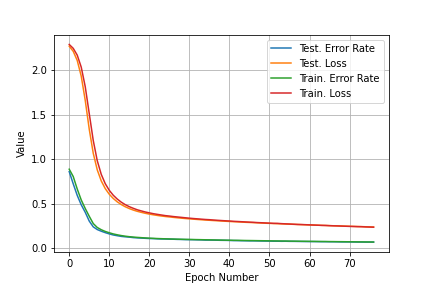}
		\caption{celLoss}\label{Fig:MNIST_MLP_cel}
	\end{subfigure}%
	~
	\begin{subfigure}[b]{0.45\textwidth}
		\centering
		\includegraphics[width=\linewidth]{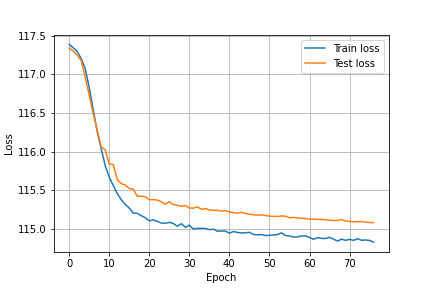}
		\caption{milLoss}\label{Fig:MNIST_MLP_milLoss}
	\end{subfigure}
	\begin{subfigure}[b]{0.45\textwidth}
		\centering
		\includegraphics[width=\linewidth]{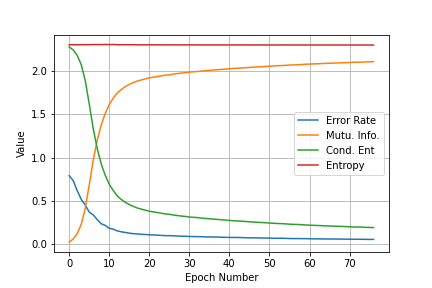}
		\caption{milLoss}\label{Fig:MNIST_MLP_milInformation}
	\end{subfigure}%
	~
	\begin{subfigure}[b]{0.45\textwidth}
		\centering
		\includegraphics[width=\linewidth]{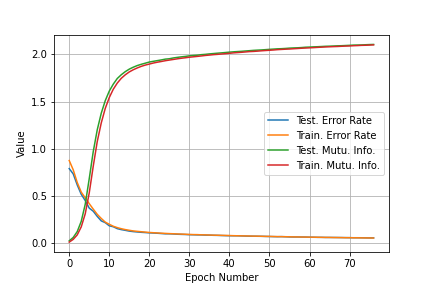}
		\caption{milLoss}\label{Fig:MNIST_MLP_milTrainTest}
	\end{subfigure}
	~
	\begin{subfigure}[b]{0.45\textwidth}
		\centering
		\includegraphics[width=\linewidth]{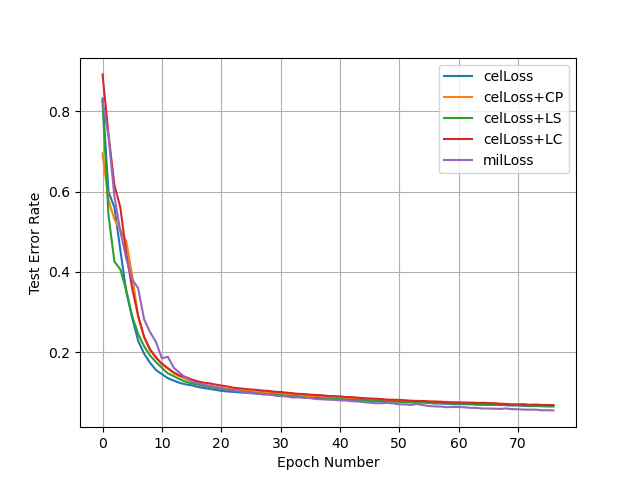}
		\caption{Error rates}\label{Fig:MNIST_MLP_testErrorRate}
	\end{subfigure}
	\caption{Multiple layer perceptron on MNIST. \ref{Fig:MNIST_MLP_cel}: error rate and loss during training and test at different epochs. \ref{Fig:MNIST_MLP_milLoss}: loss during training and testing at different epochs. \ref{Fig:MNIST_MLP_milInformation}: error rate, mutual information, label conditional entropy, and label entropy during test at different epochs. \ref{Fig:MNIST_MLP_milTrainTest}: mutual information and error rate during training and testing at different epoch. \ref{Fig:MNIST_MLP_testErrorRate}: testing error rate curves associated with different loss functions.}\label{Fig:MNIST_MLP}
\end{figure}

The CIFAR-100 dataset is similar to the CIFAR-10 dataset except that we have totally $C=100$ classes. We use DenseNet-121, Inception-ResNet-V2, Inception-V3, PreAct-ResNet-18, PreAct-ResNet-101, ResNet-34, ResNet-50, VGG-16  \cite{huang_densely_2016,szegedy_inception-v4_2016,he_identity_2016,simonyan_very_2014,he_deep_2015}. The label entropy regularization (LER) parameter $\lambda_{ent}$ associated with the mutual information learning is set to be 1e1. The batch size is fixed at 256. Each model is trained for 200 epochs. For each model under each setup for MNIST, CIFAR-10, and CIFAR-100, the regularization parameter $\epsilon$ associated with the LSR, CP, and LC is fixed at 0.1 \cite{pereyra_regularizing_2017}. We do not use any data augmentations, nor do we use the weight decays for the baseline models. Instead, we perform ablation studies over these training techniques under the mutual information learning framework. For MNIST and CIFAR-10, we conduct 3 trials for each model, and the reported results are averaged over the 3 trial. However, for CIFAR-100, we conduct single trial for each model as we do not see much variations in the results across differential trials for MNIST and CIFAR-10 datasets.  

{\bf Classification Performance} The first set of experimental results associated with the baseline models for MNIST, CIFAR-10, and CIFAR-100 datasets are presented in Table \ref{Tab:ProofOfConcept_MNIST}, \ref{Tab:ProofOfConcept_CIFAR10} and \ref{Tab:ProofOfConcept_CIFAR100} where we present the testing data accuracy. From the results, we can see the proposed mutual information learning loss (milLoss) in \eqref{Eq:MIAlterEstParametericMethod_MargDistFromLearnedJointDist} achieved improvements of large margin when compared with the conditional entropy learning loss (celLoss) and its variants in \eqref{Eq:CondEntLoss}, e.g., from 0.52 to 0.68 when the EfficientNet-B0 is used. In fact, under our experiments setup, none of LSR, CP, and LC show any improvements in accuracy.

\begin{figure}[!htb]
	\centering
	\begin{subfigure}[b]{0.45\textwidth}
		\centering
		\includegraphics[width=\linewidth]{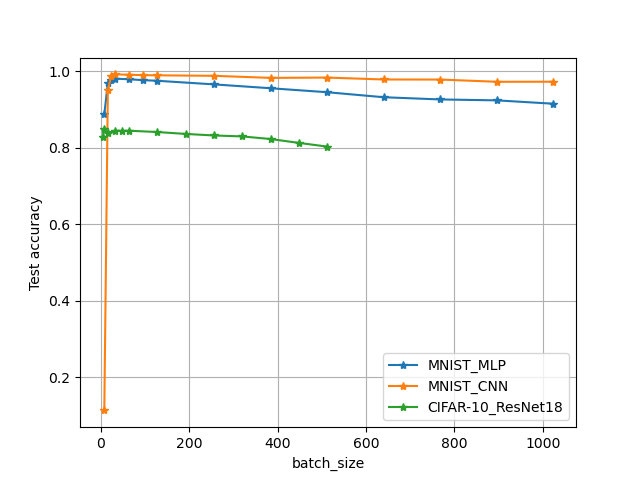}
		\caption{Batch size}
	\end{subfigure}%
	~
	\begin{subfigure}[b]{0.45\textwidth}
		\centering
		\includegraphics[width=\linewidth]{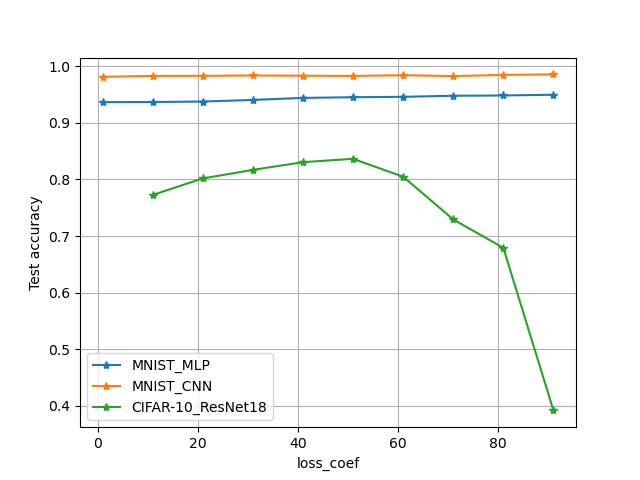}
		\caption{Loss coefficient $\lambda_{ent}$}
	\end{subfigure}
	\caption{Pinpoint batch size and $\lambda_{ent}$ parameters.}\label{Fig:ParamsPinpoint}
\end{figure} 

{\bf Learning Curves} We plot the learning curves over MNIST and CIFAR-10 datasets in Figure \ref{Fig:CIFAR10_GoogLeNet}, \ref{Fig:CIFAR10_ResNet18}, \ref{Fig:MNIST_CNN}, and \ref{Fig:MNIST_MLP} for illustrations. From the results, we can see that the proposed approach can train classifiers with much better classification performances. We can also see a very strong connection between the mutual information and the error rate. The MILCs seems to take longer time to converge than the conditional entropy learned classifiers (CELCs), and we conjecture this is because learning the joint distribution $p_{X,Y}$ is more challenging than learning the conditional distribution $p_{Y|X}$. In the celLoss training approach, we can see that the conditional entropy learned over CIFAR-10 is very small. This means that when a CIFAR-10 image is given, there will be almost no uncertainty left for the label of the CIFAR-10 sample. However, this is not case in practice. As we can see from Figure \ref{Fig:CIFAR10_samples} in the Appendix , there should be much more uncertainty left about the label associated with the CIFAR-10 sample.

We can also see a very strong connection between the mutual information and the error rate. For example, in Figure \ref{Fig:MNIST_CNN_milInformation}, the error rate decreases as the mutual information increases. When the mutual information finally converges to about 2.1, the error rate converges to about 0.06. Notice that the maximum possible entropy of the label is $\log(10) = 2.3$ nats, and the empirical results shows that the label entropy can be easily learned in very small number of epochs. The gap between the learned mutual information 2.1 and the learned conditional entropy 0.1 is about 2.0 which is 20 times larger than the conditional entropy itself. What seems to be surprising is that under our experimental settings, the existing LSR, CP, and LC regularizations for the conditional entropy learning loss (celLoss) do not bring any benefits over the celLoss itself alone. The benefits of of the proposed approach in generalization is even more obvious and significant in the CIFAR-10 classification task. For example, in Figure \ref{Fig:CIFAR10_ResNet18}, the error rate of the model trained by celLoss minimization cannot even go below 0.25, while the model trained via milLoss minimization can achieve error rate of about 0.15. 

{\bf Effects of Batch Size and Entropy Regularization Coefficient} We conduct experiments with typical neural network architectures to investigate how the batch size and the $\lambda_{ent}$ affect the performance, the experimental setup except the batch size or $\lambda_{ent}$ is exactly the same as that of the baseline models. When evaluating effect of the batch size (or the $\lambda_{ent}$), we use fixed $\lambda_{ent}=5e1$ (or fixed batch size of 512 for MNIST and 256 for CIFAR-10). The results are presented in Figure \ref{Fig:ParamsPinpoint}. From the results we can see that for both MNIST and CIFAR-10 dataset, the testing accuracy does not always go up as the batch size increases, which is quite different what is expected for CELCs. When the batch size increases, both the signal pattern and the noise pattern will become stronger. The milLoss essentially learns the mutual information associated with the joint data generation distribution, and it can overfit to the noise pattern as the batch size increases since the mutual information itself encourage the model to consider the overall data generation distribution. This then results in the classification performance degradation. However, for CELCs, despite the stronger noise pattern caused by a larger batch size, the conditional entropy learning loss can help the model avoid overfitting to the noise pattern because it encourages the model to give high confidence prediction of labels. The cost is a less accurate characterization of the joint distribution $p_{X,Y}$ by CELCs.

%\begin{figure}[!htb]
%	\centering
%	\includegraphics[width=\columnwidth]{Figures/compTraining_IUX_HUcX}\caption{Behavior of accuracy and mutual information after different number of training epochs. The $0$ horizontal coordinates corresponds to a randomized classifier before training.}\label{Fig:ClassifierAcc_Mi}
%\end{figure}
%
%\begin{figure}[!htb]
%	\centering
%	\includegraphics[width=\columnwidth]{Figures/compTraining_IUX_HUcX_1epoch}\caption{Behavior of accuracy and mutual information after different number of iterations within 1 training epoch. The 0 horizontal coordinate corresponds to a randomized classifier before training.}\label{Fig:ClassifierAcc_Mi_1epoch}
%\end{figure}

%\begin{figure}[!htb]
%	\centering
%	\includegraphics[width=\columnwidth]{Figures/dim64_N512_hidden100_layers1_rho0_5_trial0}\caption{...}\label{Fig}
% \end{figure}

\section{Conclusions}\label{Sec:Conclusions}

In this paper, we showed that the existing cross entropy loss minimization essentially learns the conditional entropy of the label when the input is revealed. We pointed out some fundamental limitations of this approach which motivate us to propose a mutual information learning framework. For the proposed learning framework, we established rigorous relation between the error probability associated with a model trained on a dataset and the mutual information associated with the distribution for generating the dataset. Besides, we derive the sample complexity for accurately training the mutual information learned classifiers. The application of our theory to a concrete binary classification data model in $\mbR^n$ was given, and we derived the bounds of the mutual information and the error probability associated with it. We also conducted extensive experiments to validate our theory, and the empirical results shows that the proposed mutual information learned classifiers (MILCs) acheive far better generalization performance than those trained via cross entropy minimization.

\section*{Acknowledgement}

We would like to thank Dr. Rui Yan from Microsoft Redmond for inspiring and constructive discussions about loss function design. We would also like to thank Dr. Praneeth Narayanamurthy and Prof. Urbashi Mitra from University of Southern California for inspiring discussions about information bottleneck and mutual information estimation.

\appendix
\section*{Appendix}

\subsection*{Appendix A1: Proof of Theorem \ref{Thm:ITQuantitiesConnections}}\label{AppSec:Proof1}

\begin{thm}\label{Thm:ITQuantitiesConnectionsApp}
	(Connections among different information-theoretic quantities) For a joint distribution $p_{X,Y}$ over a continous random vector $X\in\mbR^n$ and a discrete random variable $Y\in[C]$ where $C$ is a positive integer constant, with the definition of mutual information in \eqref{Eq:MutualInfoDefn} and Defintion \ref{Definition:Ent}-\ref{Definition:ConditionalCrossEntropy}, we have
	\begin{align}\label{Eq:ITQuantitiesConnectionsApp}
		I(X,Y) = H(Y) - H(Y|X), \nonumber\\
		I(X,Y) = h(X) - h(X|Y).
	\end{align}
\end{thm}

\begin{proof}\label{Proof:ITQuantitiesConnectionsApp}
	(of Theorem \ref{Thm:ITQuantitiesConnectionsApp}) From  \eqref{Eq:MutualInfoDefn}, we have
	\begin{align}\label{Eq:MI&ConditionalDifferentialEntrpApp}
		I(X,Y) 
		& = \int_{\mbR^n} p(x) \sum_{y\in[C]} P(y|x) \left( \log\left( \frac{p(x,y)}{P(y)} \right) -  \log\left( {p(x)} \right)\right) dx \nonumber \\
		& = \int_{\mbR^n} p(x) \sum_{y\in[C]} P(y|x) \log\left( \frac{p(x,y)}{P(y)} \right)dx 
		-
		\int_{\mbR^n} p(x)\sum_{y\in[C]} P(y|x) \log\left( {p(x)} \right) dx \nonumber \\
		& = \int_{\mbR^n} \sum_{y\in[C]} p(x,y) \log\left( {p(x|y)} \right)dx 
		-
		\int_{\mbR^n} p(x) \log\left( {p(x)} \right) dx \nonumber \\
		& = - \int_{\mbR^n} \sum_{y\in[C]} P(y)p(x|y) \log\left( \frac{1}{p(x|y)} \right)dx + h(X) \nonumber \\
		& = h(X) - \sum_{y\in[C]} P(y) \int_{\mbR^n} p(x|y) \log\left( \frac{1}{p(x|y)} \right)dx  \nonumber \\
		& = h(X) - h(X|Y).
	\end{align}
	Similarly, we have from \eqref{Eq:MutualInfoDefn}
	\begin{align}
		I(X,Y)
		& = \sum_{y\in[C]}  P(y) \int_{\mbR^n} p(x|y) \left( \log\left( \frac{p(x,y)}{p(x)} \right) -  \log\left( {P(y)} \right)\right) dx \nonumber \\
		& = \sum_{y\in[C]}  P(y) \int_{\mbR^n} p(x|y) \log\left( \frac{p(x,y)}{p(x)} \right)dx 
		- 
		\sum_{y\in[C]}  P(y) \int_{\mbR^n} p(x|y) \log\left( {P(y)} \right)dx \nonumber \\
		& = \sum_{y\in[C]}  P(y) \int_{\mbR^n} p(x|y) \log\left( \frac{p(x,y)}{p(x)} \right)dx 
		- 
		\sum_{y\in[C]}  P(y)   \log\left( {P(y)} \right)\nonumber \\
		& = H(Y) - \sum_{y\in[C]}  P(y) \int_{\mbR^n} p(x|y) \log\left( \frac{1}{P(y|x)} \right)dx \nonumber \\
		& = H(Y) - \int_{\mbR^n}  \sum_{y\in[C]}  P(y) p(x|y) \log\left( \frac{1}{P(y|x)} \right)dx \nonumber \\
		& = H(Y) - \int_{\mbR^n}  \sum_{y\in[C]}  p(x) P(y|x) \log\left( \frac{1}{P(y|x)} \right)dx \nonumber \\
		& = H(Y) - \int_{\mbR^n}  p(x)   \sum_{y\in[C]} P(y|x) \log\left( \frac{1}{P(y|x)} \right)dx \nonumber \\
		& = H(Y) - H(Y|X).
	\end{align}
\end{proof}

\subsection*{Appendix A2: Proof of Theorem \ref{Thm:EntEstViaCrossEnt}}

\begin{thm}\label{Thm:EntEstViaCrossEntApp}
	(Cross Entropy Minimization as Entropy Learning) For an arbitrary discrete distribution $P_{Y}$ in $[C]$, we have 
	\begin{align}
		H(Y) \leq \inf_{Q_Y} H(P_Y,Q_Y),
	\end{align}
	where $Q_Y$ is a distribution of $Y$, and the equality holds if and only if $P_Y=Q_Y$. When a set of $N$ data points $\mcS:=\{y_i\}_{i=1}^N$ drawn independently from $P_{Y}$ is given, by defining $R(y):= \frac{P_Y(y)}{\hat{P}_Y(y)}$ where $\hat{P}_Y$ is the empirical distribution associated with $\{y_i\}_{i=1}^N$, we have 
	\begin{align}
		H(Y) \leq \inf_{Q^g_Y} H(\hat{P}_Y^g, Q_Y^g),
	\end{align}
	where $\hat{P}_Y^g$ is defined as
	\begin{align}
		\hat{P}_Y^g(y) : = \hat{P}_Y(y) R(y), \forall {y\in[C]},
	\end{align}
	and $Q_Y^g$ is defined as 
	\begin{align}
		Q_Y^g(y) = Q_Y(y)R(y), \forall {y\in[C]},
	\end{align}
	with $Q_Y$ being a distribution of $Y$. The inequality holds if and only if $P_Y = \hat{P}_Y = Q_Y$.
\end{thm}

\begin{proof}\label{Proof:EntEstViaCrossEnt}
	(of Theorem \ref{Thm:EntEstViaCrossEnt}) From the definition of entropy, we have for an arbitrary distribution $Q_Y$ over $Y$
	\begin{align}\label{Eq:EntViaCrossEnt}
		H(Y) 
		& = \sum_{y\in[C]} P_Y(y) \log\left(\frac{1}{P_Y(y)}\right) \nonumber \\
		& = \sum_{y\in[C]} P_Y(y) \log\left(\frac{Q_Y(y)}{P_Y(y) }\frac{1}{Q_Y(y)}\right) \nonumber\\
		& = \sum_{y\in[C]} P_Y(y) \log\left(\frac{Q_Y(y)}{P_Y(y) }\right) + \sum_y P_Y(y) \log\left(\frac{1}{Q_Y(y)}\right) \nonumber\\
		& = H(P_Y,Q_Y)  - D_{KL}(P_Y||Q_Y) \nonumber\\
		& \leq H(P_Y,Q_Y),
	\end{align}
	where 
	\begin{align}\label{Defn:CrossEntropy}
		H(P_Y,Q_Y):= \sum_{y\in[C]} P_Y(y) \log\left(\frac{1}{Q_Y(y)}\right)
	\end{align} 
	is the cross entropy between $P_Y$ and $Q_Y$, and  
	\begin{align}\label{Defn:KLDivergence}
		D_{KL}(P_Y||Q_Y) :=\sum_{y\in[C]} P_Y(y) \log\left(\frac{P_Y(y)}{Q_Y(y) }\right) 
	\end{align}  
	is the KL divergence between $P_Y$ and $Q_Y$, and we used the fact that $D_{KL}(P_Y||Q_Y)\geq0$. The equality holds iff $Q_Y(y) = P_Y(y), \forall {y\in[C]}$. The \eqref{Eq:EntViaCrossEnt} holds for arbitrary $Q_Y$, thus,
	\begin{align}\label{Eq:EntViaCrossEntCont}
		H(Y) = \inf_{Q_Y} H(P_Y,Q_Y).
	\end{align}
	
	For the cross entropy $H(P_Y,Q_Y)$, we have
	\begin{align}\label{Eq:CrossEntViaType}
		H(P_Y,Q_Y) 
		& =  \sum_{y\in[C]} P_Y(y) \log\left(\frac{1}{Q_Y(y)}\right) \nonumber \\
		& =  \sum_{y\in[C]} P_Y(y) \log\left(\frac{1}{Q_Y(y)} \frac{P_Y(y)}{\hat{P}_Y(y)} \frac{\hat{P}_Y(y)}{{P}_Y(y)}\right) \nonumber  \\
		& =  \sum_{y\in[C]} P_Y(y) \log\left(\frac{1}{Q_Y(y)} \frac{P_Y(y)}{\hat{P}_Y(y)} \right) 
		+ \sum_{y\in[C]} P_Y(y) \log\left( \frac{\hat{P}_Y(y)}{{P}_Y(y)}\right) \nonumber  \\ \nonumber 
		& = - \sum_{y\in[C]} P_Y(y) \log\left(\frac{P_Y(y)}{\hat{P}_Y(y)} \right) 
		+ \sum_{y\in[C]} P_Y(y) \log\left(\frac{1}{Q_Y(y)} \frac{{P}_Y(y)}{\hat{P}_Y(y)}\right) \\
		& = -D_{KL}(P_Y||\hat{P}_Y) 
		+ \sum_{y\in[C]} P_Y(y) \log\left(\frac{1}{Q_Y(y)} \frac{{P}_Y(y)}{\hat{P}_Y(y)}\right)
	\end{align}
	Thus
	\begin{align}\label{Eq:CrossEntViaType2}
		H(P_Y,Q_Y) 
		& \leq \sum_{y\in[C]} \hat{P}_Y(y) \frac{P_Y(y)}{\hat{P}_Y(y)}  \log\left(\frac{1}{Q_Y(y)} \frac{{P}_Y(y)}{\hat{P}_Y(y)}\right) \nonumber \\
		& \leq \sum_{y\in[C]} \hat{P}_Y(y) R(y)  \log\left(\frac{1}{Q_Y(y) /R(y)}\right)\\
		&\leq \sum_{y\in[C]} \hat{P}_Y^g(y) \log\left(\frac{1}{Q_Y^g(y)}\right) \nonumber\\
		& = H(\hat{P}_Y^g,Q_Y^g),
	\end{align}
	where we used $\hat{P}_Y^g(y) = \hat{P}_Y(y) R(y)$, and $Q_Y^g(y) = Q_Y(y)/R(y)$, and $R(y) =\frac{{P}_Y(x)}{\hat{P}_Y(y)}$. The equality holds iff $P_Y(y) = \hat{P}_Y(y), \forall {y\in[C]}$ which implies that $D_{KL}(P_Y||\hat{P}_Y)=0$ and $R(y)=1$. Thus, $H(Y)=H(P_Y,Q_Y)=H(P_Y^g,Q_Y^g)$ if and only if $P_Y=Q_Y=\hat{P}_Y$. Thus, 
	\begin{align}\label{Eq:EntViaCrossEntCont1}
		H(Y)
		= \inf_{Q_Y} H(P_Y,Q_Y)
		= \inf_{Q^g_Y} H(\hat{P}_Y^g,{Q}_Y^g).
	\end{align}
	
\end{proof}

\subsection*{Appendix A3: Proof of Theorem \ref{Thm:EntEstBound}}

\begin{thm}\label{Thm:EntEstBoundApp}
	(Error Bound of Entropy Learning from Empirical Distribution) For two arbitrary distributions $P_Y$ and $\hat{P}_Y$ of a discrete random variable $Y$ over $[C]$, we have 
	\begin{align}
		\sum_{y\in[C]} R(y) \log\left( \frac{1}{{P}_Y(y)} \right)
		\leq 
		H_{P_Y}(Y) - H_{\hat{P}_Y}(Y) 
		\leq  \sum_{y\in[C]} R(y) \log\left( \frac{1}{\hat{P}_Y(y)} \right)
	\end{align}
	where $R(y) = P_Y(y) - \hat{P}_Y(y), \forall y \in[C]$, and $H_{P_Y}(Y)$ is the entropy of $Y$ calculated via $P_Y$. The equality holds if and only if $P_Y = \hat{P}_Y$. 
\end{thm}

\begin{proof}
	(of Theorem \ref{Thm:EntEstBound}) From the definition of entropy, we have 
	\begin{align*}
		H_{P_Y}(Y) - H_{\hat{P}_Y}(Y)
		& = \sum_{y\in[C]} \left( P_Y(y)\log\left(\frac{1}{{P}_Y(y)}\right) \right)
		- \sum_{y\in[C]} \left( \hat{P}_Y(y)\log\left(\frac{1}{\hat{P}_Y(y)}\right) \right) \\
		& = \sum_{y\in[C]} \left( P_Y(y)\log\left(\frac{\hat{P}_Y(y)}{{P}_Y(y)} \frac{1}{\hat{P}_Y(y)} \right) \right)
		- \sum_{y\in[C]} \left( \hat{P}_Y(y)\log\left(\frac{1}{\hat{P}_Y(y)}\right) \right) \\
		& = - \sum_{y\in[C]} \left( P_Y(y)\log\left(\frac{{P}_Y(y)}{\hat{P}_Y(y)} \right) \right)
		+ \sum_{y\in[C]} \left( P_Y(y)\log\left(\frac{1}{\hat{P}_Y(y)} \right) \right) - \sum_{y\in[C]} \left( \hat{P}_Y(y)\log\left(\frac{1}{\hat{P}_Y(y)}\right) \right) \\
		& = -D_{KL}(p_Y||\hat{P}_Y)
		+ \sum_{y\in[C]} \left( P_Y(y)\log\left(\frac{1}{\hat{P}_Y(y)} \right) \right) - \sum_{y\in[C]} \left( \hat{P}_Y(y)\log\left(\frac{1}{\hat{P}_Y(y)}\right) \right), \\
		& \leq \sum_{y\in[C]} R(y) \log\left( \frac{1}{\hat{P}_Y(y)} \right).
	\end{align*}
	The equality holds if and only if $\hat{P}_Y=P_Y$. Similarly, we can show $\sum_{y\in[C]} R(y) \log\left( \frac{1}{{P}_Y(y)} \right)
	\leq 
	H_{P_Y}(Y) - H_{\hat{P}_Y}(Y)$.
\end{proof}

\subsection*{Appendix A4: Proof of Theorem \ref{Thm:ErrorProbabilityMI_Bounds}}

\begin{thm}\label{Thm:ErrorProbabilityMI_BoundsApp}
	(Error Probability Bound via Mutual Information) Assume that the learning process $Y\to X\to\hat{Y}$ in Figure \ref{Fig:ErrorProbability} is a Markov chain where $Y\in[C]$, $X\in\mbR^n$, and $\hat{Y}\in[C]$, then for the prediction $\hat{Y}$ from an arbitrary learned model, we have
	\begin{align}\label{Eq:ErrorProbabilityBounds}
		\max\left( 0,\frac{2+H(Y) - I(X;Y) - a}{4} \right)
		\leq P_{error}
		%\leq 	\min\left( 1,\frac{2+H(Y) - I(X;Y) + a}{4}\right),
	\end{align}
	where $a:=\sqrt{(H(Y) - I(X;Y) - 2)^2 + 4}$.
\end{thm}

\begin{proof}\label{Proof:ErrorProbabilityMI_Bounds}
	(of Theorem \ref{Thm:ErrorProbabilityMI_Bounds}) We define a random variable $E$
	\begin{align}
		E = \begin{cases}
			0, \text{if } Y= \hat{Y},\\
			1, \text{if } Y\neq \hat{Y},
		\end{cases}
	\end{align}
	then the error probability will become
	\begin{align}
		P_{error} = P_{E}(E=1).
	\end{align}
	From the properties of conditional joint entropy, we have
	\begin{align}\label{Eq:UpperBound}
		H(E,Y|\hat{Y})
		& = H(E|\hat{Y})  + H(Y|E,\hat{Y}) \nonumber \\
		& = H(E|\hat{Y}) + P_E(E=0) H(Y|\hat{Y},E=0) + P_E(E=1) H(Y|\hat{Y},E=1)  \nonumber\\
		& = H(E|\hat{Y}) + P_{E}(E=1)H(Y|\hat{Y},E=1) \nonumber\\
		& \leq H(E|\hat{Y}) + P_E(E=1) H(Y|\hat{Y})  \nonumber\\
		& \leq H(E) + P_E(E=1) H(Y|\hat{Y})  \nonumber\\
		& = H(P_{error}) + P_{error}H(Y|\hat{Y}),
	\end{align}
	where we used the fact that $H(Y|\hat{Y},E=0)=0$, $H(Y|\hat{Y},E=1)\leq H(Y|\hat{Y})$, and $H(E|\hat{Y}) \leq H(E)=H(P_{error})$ with $H(P_{error})$ defined as follows
	\begin{align*}
		H(P_{error}):=	-P_{error}\log(P_{error}) - (1-P_{error})\log(1-P_{error}).
	\end{align*}
	We also have
	\begin{align}
		H(E,Y|\hat{Y})
		& = H(Y|\hat{Y}) + H(E|Y,\hat{Y}) \nonumber\\
		& = H(Y|\hat{Y}) \label{Eq:HYcYhat}\\
		& = H(Y) - I(Y;\hat{Y}) \label{Eq:MITerm}\\
		& \geq H(Y) - I(X;Y),
	\end{align}
	where we used the fact that $H(E|Y,\hat{Y})=0$, and the data processing inequality associated with Markov process $Y\to X\to \hat{Y}$, i.e., 
	\begin{align}\label{Eq:DataProcessingIneq}
		I(X;Y) \geq I(Y;\hat{Y}).
	\end{align}
	Thus, from \eqref{Eq:UpperBound} and \eqref{Eq:HYcYhat}, we have
	\begin{align}
		H(P_{error}) + P_{error}H(Y|\hat{Y}) \geq H(Y|\hat{Y}).
	\end{align}
	Combining the above and \eqref{Eq:MITerm}, we get
	\begin{align}\label{Eq:IntermediateBound}
		H(P_{error})  
		\geq (1-P_{error}) \left(H(Y) - I(X;Y) \right),
	\end{align}
	which implies $P_{error} \geq 1 - \frac{H(P_{error})}{H(Y) - I(X;Y)}$.
	
	From Lemma \ref{Lem:ErrorEntropyUB}, we have
	\begin{align*}
		(1-P_{error}) \left(H(Y) - I(X;Y) \right) \leq 1-2(P_{error} - 0.5)^2,
	\end{align*}
	or
	\begin{align}
		2P_{error}^2 - (2+H(Y) - I(X;Y))P_{error} + H(Y) - I(X;Y)-0.5 \leq 0.
	\end{align}
	By solving the above inequality for $P_{error}$, we get
	\begin{align*}
		\frac{2+H(Y) - I(X;Y) - a}{4}
		\leq P_{error}
		\leq 	\frac{2+H(Y) - I(X;Y) + a}{4},
	\end{align*}
	where $a$ is defined as 
	\begin{align*}
		a:=\sqrt{(H(Y) - I(X;Y) - 2)^2 + 4}.
	\end{align*}
	Since $P_{error}\in[0,1]$, we have 
	\begin{align}
		\max\left( 0,\frac{2+H(Y) - I(X;Y) - a}{4} \right)
		\leq P_{error}
		\leq 	\min\left( 1,\frac{2+H(Y) - I(X;Y) + a}{4}\right),
	\end{align}
	where the upper bound is trivial.

\end{proof}

\subsection*{Appendix A5: Proof of Corollary \ref{Cor:HoeffdingIneqDouble}}

\begin{cor}\label{Cor:HoeffdingIneqDoubleApp}
	(Double sided Hoeffiding inequality) Let $X_1, \cdots, X_n$ be independent random variables such that $X_i$ takes its values in $[a_i,b_i]$ almost surely for all $i\leq n$. Let $S=\sum_{i=1}^n (X_i - \mbE[X_i])$. Then, for every $t>0$, 
	\begin{align}
		P\left(\{X_1,\cdots,X_n: |S|\geq t\}\right) \leq 2\exp\left( - \frac{2t^2}{\sum_{i=1}^n (b_i - a_i)^2} \right).  
	\end{align}
\end{cor}

\begin{proof}\label{Proof:HoeffdingIneqDouble}
	(of Corollary \ref{Cor:HoeffdingIneqDouble}) We can define another set of random variables $Y_i:=-X_i, i=1,\cdots,n$, then all $i$ are independent such that $Y_i$ takes values in $[-b_i,-a_i]$. We define $S'=\sum_{i=1}^n (Y_i - \mbE[Y_i])$, then from Lemma \ref{Lem:HoeffdingIneqSingleSide}, we have 
	\begin{align}
		\exp\left( \frac{2t^2}{\sum_{i=1}^n (b_i - a_i)^2} \right)
		\geq P\left(\{X_1,\cdots,X_n: S'\geq t\}\right)
		= P\left(\{X_1,\cdots,X_n:  S\leq -t\} \right).
	\end{align}
	Thus combine the above inequality with Lemma \ref{Lem:HoeffdingIneqSingleSide}, we get the double sided Hoeffding inequality, i.e.,
	\begin{align}
		P\left(\{X_1,\cdots,X_n:  |S|\geq t\} \right)
		& = P\left(\{X_1,\cdots,X_n:  S\leq -t, \text{ or } S\geq t\} \right) \\
		& \leq P\left(\{X_1,\cdots,X_n:  S\leq -t\} \right)
		+ 	P\left(\{X_1,\cdots,X_n:  S\geq t\} \right) \\
		& \leq 2\exp\left( - \frac{2t^2}{\sum_{i=1}^n (b_i - a_i)^2} \right).
	\end{align}
\end{proof}

\subsection*{Appendix A6: Proof of Lemma \ref{Lem:ConcentrationIneq}}

\begin{lemma}\label{Lem:ConcentrationIneqApp}
	(Concentration Inequality for Conditional Cross Entropy) We consider a set of random variable pairs $\mcS:\{(X_i,Y_i)\}_{i=1}^N$ with each $(X_i,Y_i)$ I.I.D. according to $p_{X,Y}$ in $\mbR^n\times[C]$, and define 
	\begin{align}
		D:= \frac{1}{N} \sum_{i=1}^N\log(Q_{Y|X}(Y_i|X_i; \theta))   - \mbE_{P_{Y,X}} \left[ \log(Q_{Y|X}(Y|X; \theta)) \right],
	\end{align}
	where $\theta\in\Theta_c$, $\Theta_c$ is a countable set, and $Q_{Y|X}(Y|X;\theta)$ is a function of $X, Y$ with parameters $\theta$. Assume $Q_{Y|X}(Y|X;\theta)\geq P^\#, \forall X,Y,\theta$ where $P^\#>0$ is a constant. Then, we have 
	\begin{align}
		P(\{\mcS: \max_{\theta\in\Theta_c} |D| \geq t\}) \leq 2|\Theta_c| \exp\left( -\frac{2Nt^2}{(\log(P^\#))^2} \right)
	\end{align}
	where $t>0$ is a constant.  
\end{lemma}

\begin{proof}\label{Proof:ConcentrationIneq}
	(of Lemma \ref{Lem:ConcentrationIneq}) From the definitions, we have for an arbitrary $\theta_j\in\Theta_c$
	\begin{align}
		& P\left( \left\{\mcS: \max_{\theta_j\in\Theta_c} \left| \frac{1}{N} \sum_{i=1}^N \left[ \log(Q_{Y|X}(Y_i|X_i; \theta_j)) \right]  - \mbE_{P_{Y,X}} \left[ \log(Q_{Y|X}(\theta_j)) \right] \right| \geq t  \right\}\right) \\
		& = P\left( \left\{ \mcS: \max_{\theta_j\in\Theta_c} \left| \frac{1}{N} \sum_{i=1}^N \left(\log(Q_{Y|X}(Y_i|X_i; \theta_j)) - \mbE_{P_{Y,X}} \left[ \log(Q_{Y|X}(Y_i|X_i;\theta_j)) \right] \right)  \right| \geq t \right\} \right)\\
		& = P\left( \left\{ \mcS: \exists \theta_j\in\Theta_c \text{ such that }\left| \frac{1}{N} \sum_{i=1}^N \left( \log(Q_{Y|X}(Y_i|X_i; \theta_j)) - \mbE_{P_{Y,X}} \left[ \log(Q_{Y|X}(Y_i|X_i;\theta_j)) \right] \right)  \right| \geq t \right\} \right)\\
		& \leq \sum_{\theta_j\in\Theta_c} P\left( \left\{ \mcS: \left| \frac{1}{N} \sum_{i=1}^N \left[ \log(Q_{Y|X}(Y_i|X_i; \theta_j)) - \mbE_{P_{Y,X}} \left[ \log(Q_{Y|X}(Y_i|X_i;\theta_j)) \right] \right]  \right| \geq t \right\} \right),
	\end{align}
	where $\log(Q_{Y|X}(Y_i|X_i; \theta_j)), i=1,\cdots,N$ are I.I.D. because $(X_i,Y_i), i=1,\cdots,N$ are I.I.D. Since $Q_{Y|X}(y|x;\theta)\in[P^\#, 1], \forall (x,y)\in\mbR^n\times [C]$, then $\log(Q_{Y|X}(y|x;\theta))\in[\log(P^\#), 0]$. From Corollary \ref{Cor:HoeffdingIneqDouble}, we have 
	\begin{align}
		& P\left( \left\{ \mcS: \left| \frac{1}{N} \sum_{i=1}^N \left[ \log(Q_{Y|X}(Y_i|X_i; \theta_j)) - \mbE_{P_{Y,X}} \left[ \log(Q_{Y|X}(Y_i|X_i;\theta_j)) \right] \right]  \right| \geq t \right\} \right) \\
		&	\leq 2\exp\left( - \frac{2N^2 t^2}{\sum_{i=1}^N (\log(P^\#))^2} \right) \\
		& = 2\exp\left( - 2N\frac{t^2}{(\log(P^\#))^2} \right). 
	\end{align}
	Thus, 
	\begin{align}
		P(\{\mcS: \max_{\theta\in\Theta_c} |D| \geq t\})
		& = P\left( \left\{\mcS: \max_{\theta_j\in\Theta_c} \left| \frac{1}{N} \sum_{i=1}^N \left[ \log(Q_{Y|X}(Y_i|X_i; \theta_j)) \right]  - \mbE_{P_{Y,X}} \left[ \log(Q_{Y|X}(\theta_j)) \right] \right| \geq t  \right\}\right) \nonumber\\
		& \leq \sum_{\theta_j\in\Theta_c} 2\exp\left( - 2N\frac{t^2}{(\log(P^\#))^2} \right) \nonumber\\
		& = 2|\Theta_c|\exp\left( - 2N\frac{t^2}{(\log(P^\#))^2} \right).
	\end{align}
	
\end{proof}

\subsection*{Appendix A7: Proof of Theorem \ref{Thm:SampleComplexity4EmpiricalEst}}

\begin{thm}\label{Thm:SampleComplexity4EmpiricalEstApp}
	(Sample Complexity for Estimation Error Bound) We consider a joint distirbution $p_{X,Y}$ in $\mbR^n\times[C]$ where $X\in\mbR^n$ is a continuous random vector, and $Y\in[C]$ is a discrete random variable. We define $I_{\Gamma,\Theta}$ associated with $p_{X,Y}$ similar to \eqref{Eq:OrgnMIEstModApproxFormTwo}, i.e., 
	\begin{align}
		I_{\Gamma, \Theta}
		= \inf_{\gamma\in\Gamma} \mbE_{P_Y}[- \log\left( Q_{Y}(Y;\gamma)\right)] - \inf_{\theta\in\Theta} \mbE_{p_{X,Y}}[-\log(Q_{Y|X}(Y|X;\theta))],
	\end{align}
	where $Q_{Y|X}(Y|X;\theta)$ is a neural network with parameters $\theta\in\Theta\subset\mbR^m$ which predicts the conditional probability of $Y$ conditioning on $X$, and $Q_Y(Y;\gamma)$ is another neural network with parameters $\gamma\in\Gamma\subset\mbR^{m'}$ which predicts the marginal probability of $Y$. Assume that we are given a set of random examples $\mcS: \{(X_i,Y_i)\}_{i=1}^N$ such that $(X_i,Y_i), i=1\cdots,N$ are I.I.D. and follow $p_{X,Y}$. Define $I_{\Gamma,\Theta}^{(N)}$ similar to that in \eqref{Eq:OrgnMIEstModEmpriFormTwo}, i.e., 
	\begin{align}
		I_{\Gamma, \Theta} ^{(N)}
		= \inf_{\gamma\in\Gamma } \frac{1}{N} \sum_{i=1}^N - \log(Q_{Y}(Y_i; \gamma)) - \inf_{\theta\in\Theta} \frac{1}{N} \sum_{i=1}^N - \log(Q_{Y|X}(Y_i|X_i; \theta)).
	\end{align}
	Assume both $\Theta$ and $\Gamma$ are compact sets, and bounded, i.e., $\|\theta\|\leq M_\theta$ and $\|\gamma\|\leq M_\gamma$ where $M_\theta>0, M_\gamma>0$ are constants. We assume both $Q_{Y}(\cdot;\gamma)$ and $Q_{Y|X}(\cdot;\theta)$ are lower bounded by $P^\#$, and they are Lipschitz continuous with respect to $\theta$ for all $y\in[C]$ and all $(x,y)\in\mbR^n\times[C]$, and the Lipschitz constants are $L_\gamma>0$ and $L_\theta>0$, respectively. Then, when $N
	\geq \frac{2\log\left(\frac{1}{\epsilon} \right) \left( P^\# \log(P^\#) \right)^2}{\left( \delta P^\# - 4^{\frac{1+m'}{m'}} L_\gamma M_\gamma \sqrt{m'} - 4^{\frac{1+m}{m}} L_\theta M_\theta \sqrt{m} \right)^2}$, we have
	\begin{align}
		P\left( \left\{\mcS: \left| I_\Theta^{(N)} - I_\Theta\right| \leq \delta \right\} \right) \geq 1-\epsilon.
	\end{align} 
	
\end{thm} 

\begin{proof}\label{Proof:SampleComplexity4EmpiricalEst}
	(of Theorem \ref{Thm:SampleComplexity4EmpiricalEst}) From the definition of $I_{\Gamma,\Theta}$ and $I_{\Gamma,\Theta}^{(N)}$ in \eqref{Eq:OrgnMIEstModApproxFormSig} and \eqref{Eq:OrgnMIEstModEmpriFormSig}, we have 
	\begin{align}
		\left| I_{\Gamma, \Theta}^{(N)} - I_{\Gamma, \Theta} \right| 
		& = \Bigg|  \inf_{\gamma\in\Gamma } \frac{1}{N} \sum_{i=1}^N - \log(Q_{Y}(Y_i; \gamma)) - \inf_{\theta\in\Theta} \frac{1}{N} \sum_{i=1}^N - \log(Q_{Y|X}(Y_i|X_i; \theta)) \nonumber \\
		&\quad - \left(  \inf_{\gamma\in\Gamma} \mbE_{P_Y}[- \log\left( Q_{Y}(Y;\gamma)\right)] - \inf_{\theta\in\Theta} \mbE_{p_{X,Y}}[-\log(Q_{Y|X}(Y|X;\theta))] \right) \Bigg|  \nonumber \\
		& \leq \left|   \inf_{\gamma\in\Gamma } \frac{1}{N} \sum_{i=1}^N - \log(Q_{Y}(Y_i; \gamma))  - \inf_{\theta \in \Theta} \mbE_{P_Y} \left[ -\log(Q_Y(Y; \gamma)) \right]  \right|  \nonumber \\
		&\quad + \left| \inf_{\theta\in\Theta} \frac{1}{N} \sum_{i=1}^N - \log(Q_{Y|X}(Y_i|X_i; \theta)) -  \inf_{\theta \in\Theta} \mbE_{P_{Y,X}} \left[ -\log(Q_{Y|X}(Y|X; \theta)) \right]\right|  \nonumber \\
		& \leq  \sup_{\gamma\in\Gamma} \left| \frac{1}{N} \sum_{i=1}^N - \log(Q_{Y}(Y_i; \gamma)) - \mbE_{P_Y} \left[ -\log(Q_Y(\gamma)) \right]  \right|  \nonumber \\
		&\quad + \sup_{\theta\in\Theta} \left| \frac{1}{N} \sum_{i=1}^N - \log(Q_{Y|X}(Y_i|X_i; \theta))  -   \mbE_{P_{Y,X}} \left[ -\log(Q_{Y|X}(\theta)) \right]\right| \label{Eq:DueToLemma}\\
		& = \sup_{\gamma\in\Gamma} \left| \frac{1}{N} \sum_{i=1}^N \log(Q_{Y}(Y_i; \gamma))  - \mbE_{P_Y} \left[ \log(Q_Y(\gamma)) \right]  \right|  \nonumber \\
		&\quad + \sup_{\theta\in\Theta} \left| \frac{1}{N} \sum_{i=1}^N  \log(Q_{Y|X}(Y_i|X_i; \theta))  -   \mbE_{P_{Y,X}} \left[ \log(Q_{Y|X}(\theta)) \right]\right|  \nonumber \\
		& = \sup_{\gamma\in\Gamma} A+ \sup_{\theta\in\Theta} B
	\end{align}
	where \eqref{Eq:DueToLemma} is due to Lemma \ref{Lem:IndpdtIneq}, and we define 
	\begin{align}\label{Eq:A_B}
		A: = \left| \frac{1}{N} \sum_{i=1}^N \log(Q_{Y}(Y_i; \gamma))  - \mbE_{P_Y} \left[ \log(Q_Y(\gamma)) \right]  \right|, \\
		B:  =  \left| \frac{1}{N} \sum_{i=1}^N  \log(Q_{Y|X}(Y_i|X_i; \theta))  -   \mbE_{P_{Y,X}} \left[ \log(Q_{Y|X}(\theta)) \right]\right|.
	\end{align}
	
	We will bound both $A$ and $B$. Before this, we introduce the concept of covering set for $\Theta$ and $\Gamma$. For the $\Theta\subset\mbR^m$, we construct a cover, i.e., a set of balls $\left\{ B_{r_\theta}(\theta_j) \right\}_{j=1}^{N(r_\theta,\Theta)}$ with radius $r_\theta$ and centered at $\theta_j, j=1,\cdots,N(r_\theta,\Theta)$ such that $\Theta\subset \cup_{j=1}^{N(r_\theta,\Theta)} B_{r_\theta}(\theta_j)$. From Lemma \ref{Lem:SubspaceCoveringNumber}, we know that 
	\begin{align} 
		N(r_\theta,\Theta) \leq \left( \frac{2M_\theta \sqrt{m}}{r_\theta} \right)^m,
	\end{align}
	where $M_\theta$ is the upper bound of $\theta$, i.e., $\|\theta\|\leq M_\theta, \forall \theta\in\Theta$.  Similarly, for $\Gamma\subset\mbR^{m'}$, we  can construct a set of balls $\{B_{r_\gamma}(\gamma_j)\}_{j=1}^{N(r_\gamma,\Gamma)}$ which covers $\Gamma$, and 
	\begin{align}
		N(r_\gamma,\Gamma) \leq \left( \frac{2M_\gamma \sqrt{m'}}{r_\gamma} \right)^{m'}.
	\end{align}
	
	{\bf Upper bound of $A$} For $A$, we can find a $\gamma_j$ so that $B_{r_\gamma}(\gamma_j)$ covers $\gamma$, i.e., $\|\gamma - \gamma_j\|\leq r_\gamma$. Then, 
	\begin{align}
		A
		& = \left|  \frac{1}{N} \sum_{i=1}^N \log(Q_{Y}(Y_i; \gamma))  - \mbE_{P_Y} \left[ \log(Q_Y(\gamma)) \right]   \right|  \\
		& \leq  \left| \frac{1}{N} \sum_{i=1}^N \log(Q_{Y}(Y_i; \gamma)) - \frac{1}{N} \sum_{i=1}^N \log(Q_{Y}(Y_i; \gamma_j)) \right| 
		+ \left| \mbE_{{P}_Y} \left[ \log(Q_Y(\gamma)) \right] - \mbE_{P_Y} \left[ \log(Q_Y(\gamma_j)) \right]  \right| \\
		&\quad + \left| \frac{1}{N} \sum_{i=1}^N \log(Q_{Y}(Y_i; \gamma_j)) - \mbE_{P_Y} \left[ \log(Q_Y(\gamma_j)) \right]  \right| \\
		& \leq \frac{1}{N} \sum_{i=1}^N \left[  \left|  \log(Q_Y(Y_i;\gamma)) -  \log(Q_Y(Y_i;\gamma_j))  \right| \right]
		+ \mbE_{{P}_Y} \left[ \left|  \log(Q_Y(\gamma))- \log(Q_Y(\gamma_j)) \right| \right]  \\
		&\quad + \left| \frac{1}{N} \sum_{i=1}^N \log(Q_{Y}(Y_i; \gamma_j)) - \mbE_{P_Y} \left[ \log(Q_Y(\gamma_j)) \right]  \right| \\
		& \leq \frac{1}{P^\#} \frac{1}{N} \sum_{i=1}^N \left[  \left|  Q_Y(Y_i;\gamma) -  Q_Y(Y_i;\gamma_j)  \right| \right]
		+ \frac{1}{P^\#} \mbE_{{P}_Y} \left[ \left|  Q_Y(\gamma)- Q_Y(\gamma_j) \right| \right]  \\
		&\quad + \left| \frac{1}{N} \sum_{i=1}^N \log(Q_{Y}(Y_i; \gamma_j)) - \mbE_{P_Y} \left[ \log(Q_Y(\gamma_j)) \right]  \right|,
	\end{align}
	where we used the Lipschitz continuity of $\log(\cdot)$.
	
	Since 
	\begin{align}
		\frac{1}{N} \sum_{i=1}^N \left[  \left|  Q_Y(Y_i; \gamma) -  Q_Y(Y_i;\gamma_j)  \right| \right]
		\leq  L_\gamma \|\gamma - \gamma_j\|
		\leq L_\gamma r_\gamma,
	\end{align}
	and 
	\begin{align}
		\mbE_{{P}_Y} \left[ \left|  Q_Y(\gamma)- Q_Y(\gamma_j) \right| \right]
		\leq L_\gamma \|\gamma - \gamma_j\| 
		\leq L_\gamma r_\gamma,
	\end{align}
	where we used the Lipschitz continuity of $Q_Y(\cdot;\gamma)$ with respect to $\gamma$, then 
	\begin{align}
		A
		& \leq \frac{2L_\gamma r_\gamma}{P^\#}  + \left| \frac{1}{N} \sum_{i=1}^N \left[ \log(Q_Y(Y_i; \theta_j)) \right] - \mbE_{P_Y} \left[ \log(Q_Y(\theta_j)) \right]  \right|.
	\end{align}
	From Lemma \ref{Lem:ConcentrationIneqCrossEnt}, the following holds with probability at least $1-2\left( \frac{2M_\gamma\sqrt{m'}}{r_\gamma} \right)^{m'} \exp\left( -2N \frac{t^2}{(\log(P^\#))^2} \right)$ over $\{Y_i\}_{i=1}^N$, 
	\begin{align}
		\left| \frac{1}{N} \sum_{i=1}^N \left[ \log(Q_Y(Y_i; \gamma_j)) \right] - \mbE_{P_Y} \left[ \log(Q_Y(\gamma_j)) \right]  \right|
		\leq t.
	\end{align}
	Thus, with high probability,
	\begin{align}
		A \leq \frac{2L_\gamma r_\gamma}{P^\#} + t,
	\end{align} 
	where $t>0$ is a constant.
	
	{\bf Upper bound of $B$} For $B$, we can also find a $\theta_j$ such that $B_{r_\gamma}(\theta_j)$ covers $\theta$, thus, $\|\theta - \theta_j\| \leq r_\gamma$. Then, 
	\begin{align*}
		B
		& = \left| \frac{1}{N} \sum_{i=1}^N  \log(Q_{Y|X}(Y_i|X_i; \theta))  -   \mbE_{P_{Y,X}} \left[ \log(Q_{Y|X}(\theta)) \right]\right| \\
		& \leq \left| \frac{1}{N} \sum_{i=1}^N  \log(Q_{Y|X}(Y_i|X_i; \theta)) -  \frac{1}{N} \sum_{i=1}^N  \log(Q_{Y|X}(Y_i|X_i; \theta_j)) \right|  
		+ \left| \mbE_{P_{Y,X}} \left[ \log(Q_{Y|X}(\theta)) \right]
		-    \mbE_{P_{Y,X}} \left[ \log(Q_{Y|X}(\theta_j)) \right]\right| \\
		&\quad + \left| \frac{1}{N} \sum_{i=1}^N  \log(Q_{Y|X}(Y_i|X_i; \theta_j))  - \mbE_{P_{Y,X}} \left[ \log(Q_{Y|X}(\theta_j)) \right] \right| \\
		& \leq \frac{1}{N} \sum_{i=1}^N \left|  \log(Q_{Y|X}(Y_i|X_i; \theta)) -  \log(Q_{Y|X}(Y_i|X_i; \theta_j)) \right|  
		+\mbE_{P_{Y,X}} \left[ \left|   \log(Q_{Y|X}(\theta)) 
		-    \log(Q_{Y|X}(\theta_j)) \right| \right] \\
		&\quad + \left| \frac{1}{N} \sum_{i=1}^N  \log(Q_{Y|X}(Y_i|X_i; \theta_j))  - \mbE_{P_{Y,X}} \left[ \log(Q_{Y|X}(\theta_j)) \right] \right| \\
		& \leq \frac{1}{P^\#} \frac{1}{N} \sum_{i=1}^N  \left[   \left| Q_{Y|X}(Y_i; \theta) -  Q_{Y|X}(Y_i; \theta_j)  \right| \right]  
		+ \frac{1}{P^\#} \mbE_{P_{Y,X}} \left[\left|  Q_{Y|X}(\theta) 
		-   Q_{Y|X}(\theta_j) \right|  \right]\\
		&\quad + \left| \frac{1}{N} \sum_{i=1}^N  \log(Q_{Y|X}(Y_i| X_i;\theta_j))  - \mbE_{P_{Y,X}} \left[ \log(Q_{Y|X}(\theta_j)) \right] \right|  \\
		& \leq \frac{1}{P^\#}  L_\theta \| \theta -  \theta_j \| 
		+ \frac{1}{P^\#} L_\theta \|  \theta
		-   \theta_j\| 
		+ \left| \frac{1}{N} \sum_{i=1}^N  \log(Q_{Y|X}(Y_i|X_i; \theta_j)) - \mbE_{P_{Y,X}} \left[ \log(Q_{Y|X}(\theta_j)) \right] \right|  \\
		& \leq \frac{2 L_\theta r_\theta}{P^\#}  
		+ \left| \frac{1}{N} \sum_{i=1}^N  \log(Q_{Y|X}(Y_i|X_i; \theta_j)) - \mbE_{P_{Y,X}} \left[ \log(Q_{Y|X}(\theta_j)) \right] \right|, 
	\end{align*}
	where we used the Lipschitz continuity of $\log(\cdot)$ and $Q_{Y|X}(\cdot;\theta)$.
	
	From Lemma \ref{Lem:ConcentrationIneq}, with probability at least $1-2\left( \frac{2M_\theta \sqrt{m}}{ r_\theta} \right)^m \exp\left( -2N \frac{t^2}{(\log(P^\#))^2} \right)$ over $\{(X_i,Y_i)\}_{i=1}^N$, we have 
	\begin{align}
		\left| \frac{1}{N} \sum_{i=1}^N  \log(Q_{Y|X}(Y_i|X_i; \theta_j)) - \mbE_{P_{Y,X}} \left[ \log(Q_{Y|X}(\theta_j)) \right] \right| 
		\leq t,
	\end{align}
	where we used $|\Theta_c|=N(r,\Theta) \leq \left( \frac{2M_\theta \sqrt{m}}{ r} \right)^m$. Thus, with high probability
	\begin{align}
		B 
		\leq \frac{2L_\theta r_\theta}{P^\#}  
		+ t.
	\end{align}
	
	Since
	\begin{align}
		P\left( \left\{ S: A \leq \frac{2L_\gamma r_\gamma}{P^\#} + t, B \leq \frac{2L_\theta r_\theta}{P^\#} + t\right\} \right)
		& = 1 - P\left( \left\{ S: A \geq \frac{2L_\gamma r_\gamma}{P^\#} + t \text{ or }  B \geq \frac{2L_\theta r_\theta}{P^\#} + t\right\} \right) \\
		& \geq 1 - P\left( \left\{ S: A \geq \frac{2L_\gamma r_\gamma}{P^\#} + t \right\}\right) -  P\left( \left\{ S: B \geq \frac{2L_\theta r_\theta}{P^\#} + t\right\} \right) \\
		& \geq 1 - 2\left( \frac{2M_\gamma\sqrt{m'}}{r_\gamma} \right)^{m'} \exp\left( -2N \frac{t^2}{(\log(P^\#))^2} \right) \\
		&\quad - 2\left( \frac{2M_\theta \sqrt{m}}{ r_\theta} \right)^m \exp\left( -2N \frac{t^2}{(\log(P^\#))^2} \right),
	\end{align}
	then with high probability, the following holds 
	\begin{align}
		\left| I_\Theta^{(N)} - I_\Theta \right| 
		& \leq \sup_{\theta\in \Theta} A + \sup_{\theta\in\Theta_c} B \\
		& \leq \frac{2L_\gamma r_\gamma}{P^\#}+ \frac{2L_\theta r_\theta}{P^\#} + 2t
	\end{align}
	
	By taking $t=\frac{\delta P^\# - 2L_\gamma r_\gamma  - 2 L_\theta r_\theta}{2P^\#}$, we have
	\begin{align}
		\left| I_\Theta^{(N)} - I_\Theta \right|  \leq \delta,
	\end{align}
	where $\delta>0$ is a constant. Take $r_\gamma = 4^{\frac{1}{2}+\frac{1}{m'}}M_\gamma \sqrt{m'}$ and $r_\theta = 4^{\frac{1}{2}+\frac{1}{m}}M_\theta \sqrt{m}$, and then let
	\begin{align}
		2\left( \frac{2M_\gamma\sqrt{m'}}{r_\gamma} \right)^{m'} \exp\left( -2N \frac{t^2}{(\log(P^\#))^2} \right) 
		+ 2\left( \frac{2M_\theta \sqrt{m}}{ r_\theta} \right)^m \exp\left( -2N \frac{t^2}{(\log(P^\#))^2} \right) \leq \epsilon,
	\end{align}
	we have
	\begin{align}
		N 
		\geq \frac{2\log\left(\frac{1}{\epsilon} \right) \left( P^\# \log(P^\#) \right)^2}{\left( \delta P^\# - 4^{\frac{1+m'}{m'}} L_\gamma M_\gamma \sqrt{m'} - 4^{\frac{1+m}{m}} L_\theta M_\theta \sqrt{m} \right)^2}.
	\end{align}
	Thus, when $N
	\geq \frac{2\log\left(\frac{1}{\epsilon} \right) \left( P^\# \log(P^\#) \right)^2}{\left( \delta P^\# - 4^{\frac{1+m'}{m'}} L_\gamma M_\gamma \sqrt{m'} - 4^{\frac{1+m}{m}} L_\theta M_\theta \sqrt{m} \right)^2}$, we have
	\begin{align}
		P\left( \{\mcS: \left| I_\Theta^{(N)} - I_\Theta\right| \leq \delta \} \right) \geq 1-\epsilon.
	\end{align}

\end{proof}

\subsection*{Appendix A8: Proof of Lemma \ref{Lem:QuadraticFormExpectation}}

\begin{lemma}\label{Lem:QuadraticFormExpectationApp}
	(Expectation of Quadratic Form of Gaussian Random Vector) For a Gaussian random vector $X\in\mbR^n$ following $\mcN(\mu,\Sigma)$, we have
	\begin{align}\label{Eq:QuadraticFormExpectation}
		\mbE_{X}\left[ X^TAX \right] = \tr(A\Sigma) + \mu^T\Sigma^{-1}\mu, \nonumber\\
		\mbE_{X}\left[ (X-\mu)^TA(X-\mu) \right] = \tr(A\Sigma), \nonumber\\
		\mbE_{X}\left[ (X+\mu)^T A (X+\mu) \right] = \tr(A\Sigma) + 4\mu^T\Sigma \mu,
	\end{align}
	where $A\in\mbR^{n\times n}$ is a square matrix.
\end{lemma}

\begin{proof}\label{Proof:QuadraticFormExpectation}
	(of Lemma \ref{Lem:QuadraticFormExpectation}) From the definition of expectation, we have
	\begin{align*}
		\mbE_{X}\left[ X^TAX \right] 
		& = \int_{\mbR^n}  \frac{1}{\sqrt{|2\pi\Sigma|}} \exp\left(-\frac{(x-\mu)^T\Sigma^{-1}(x-\mu)}{2}\right) x^TAx  dx \\
		& = \int_{\mbR^n}  \frac{1}{\sqrt{|2\pi\Sigma|}} \exp\left(-\frac{(x-\mu)^T\Sigma^{-1}(x-\mu)}{2}\right) \tr(Axx^T)  dx \\
		& = \int_{\mbR^n}  \frac{1}{\sqrt{|2\pi\Sigma|}} \exp\left(-\frac{(x-\mu)^T\Sigma^{-1}(x-\mu)}{2}\right) \tr\left(A(x-\mu)(x-\mu)^T - A\mu\mu^T + Ax\mu^T + A\mu x^T \right)  dx \\
		& = \int_{\mbR^n}  \frac{1}{\sqrt{|2\pi\Sigma|}} \exp\left(-\frac{(x-\mu)^T\Sigma^{-1}(x-\mu)}{2}\right) \left( \tr(A(x-\mu)(x-\mu)^T) - \mu^TA\mu + \mu^TAx + x^TA\mu \right)  dx \\
		& = \int_{\mbR^n}  \frac{1}{\sqrt{|2\pi\Sigma|}} \exp\left(-\frac{(x-\mu)^T\Sigma^{-1}(x-\mu)}{2}\right) \tr(A(x-\mu)(x-\mu)^T)  dx + \mu^TA\mu\\
		& = \tr(A\Sigma) + \mu^TA\mu.
	\end{align*}
	Similarly, we can derive the other two equations in \eqref{Eq:QuadraticFormExpectation}.
\end{proof}

\subsection*{Appendix A9: Proof of Theorem \ref{Thm:BinaryClassificationDataModelTruthMI}}

\begin{thm}\label{Thm:BinaryClassificationDataModelTruthMIApp}
	(Mutual Information of Binary Classification Dataset Model) For the data model with distribution defined in \eqref{Defn:BinaryClassificationDataModel}, we have the mutual information $I(X;Y)$ satisfying
	\begin{align}\label{Eq:MIinBinaryClassificationModel}
		2\min(q,1-q)\mu^T\Sigma^{-1}\mu \leq I(X;Y) \leq 4q(1-q)\mu^T\Sigma^{-1}\mu.
	\end{align}
\end{thm}

\begin{proof}\label{Proof:BinaryClassificationDataModelTruthMI}
	(of Theorem \ref{Thm:BinaryClassificationDataModelTruthMI}) From the definition of mutual information, we have
	\begin{align}\label{Defn:MIinBinaryClassificationModel}
		I(X,Y) := h(X) - h(X|Y), 
	\end{align}
	where the differential entropy $h(X|Y=1) = h(X|Y=-1) = \frac{1}{2}\log(|2\pi e\Sigma|)$ and $e$ is the natural number. Thus 
	\begin{align}\label{Eq:ConditionalEntropy_XcY}
		h(X|Y) 
		= P(Y=1)\times h(X|Y=1) + P(Y=-1)\times h(X|Y=-1) =\frac{1}{2}\log(|2\pi e\Sigma|).
	\end{align}
	
	From the definition of data model in \eqref{Defn:BinaryClassificationDataModel}, we have the marginal distirbution $p_X$
	\begin{align}\label{Eq:XMarginalDistribution}
		p_X(x) 
		& = P_Y(Y=1) \times p_{X|Y=1}(X=x)
		+ P_Y(Y=-1) \times p_{X|Y=-1}(X=x) \nonumber\\
		%& = P_Y(Y=1) \times \frac{1}{\sqrt{2\pi}} \exp\left(-\frac{(x-1)^2}{2}\right)
		%+ P_Y(Y=-1) \times \frac{1}{\sqrt{2\pi}} \exp\left(-\frac{(x+1)^2}{2}\right) \nonumber\\
		& =  \frac{(1-q)}{\sqrt{|2\pi\Sigma|}} \exp\left(-\frac{(x-\mu)^T\Sigma^{-1}(x-\mu)}{2}\right)
		+  \frac{q}{\sqrt{|2\pi\Sigma|}} \exp\left(-\frac{(x+\mu)^T\Sigma^{-1}(x+\mu)}{2}\right).
	\end{align}
	we have
	\begin{align}\label{Eq:XEntropy}
		h(X)
		& := \int_\mbR p_X(x) \log\left( \frac{1}{p_X(x)} \right) dx \nonumber \\
		& = \int_\mbR p_X(x) \times  \log\left( \frac{1}{ \frac{(1-q)}{\sqrt{|2\pi\Sigma|}} \exp\left(-\frac{(x-\mu)^T\Sigma^{-1}(x-\mu)}{2}\right)
			+  \frac{q}{\sqrt{|2\pi\Sigma|}} \exp\left(-\frac{(x+\mu)^T\Sigma^{-1}(x+\mu)}{2}\right)} \right) dx\nonumber \\
		& =  (-1)\times \int_\mbR p_X(x)  \log\left( 
		\frac{(1-q)}{\sqrt{|2\pi\Sigma|}} \exp\left(-\frac{(x-\mu)^T\Sigma^{-1}(x-\mu)}{2}\right)
		+  \frac{q}{\sqrt{|2\pi\Sigma|}} \exp\left(-\frac{(x+\mu)^T\Sigma^{-1}(x+\mu)}{2}\right)
		\right) dx.
	\end{align}
	
	From Jensen's inequality for convex function $\exp(x)$, we have
	\begin{align*}%\label{Eq:LogOfConvexCombinationOfExp}
		& \log\left( 
		\frac{(1-q)}{\sqrt{|2\pi\Sigma|}} \exp\left(-\frac{ (x-\mu)^T\Sigma^{-1}(x-\mu) }{2}\right)
		+  \frac{q}{\sqrt{|2\pi\Sigma|}} \exp\left(-\frac{(x+\mu)^T\Sigma^{-1}(x+\mu)}{2}\right)
		\right) \\
		& =  \log\left( \frac{1}{\sqrt{|2\pi\Sigma|}} \right) 
		+ \log\left( {(1-q) \times  \exp\left(-\frac{ (x-\mu)^T\Sigma^{-1}(x-\mu) }{2}\right)
			+ q \times  \exp\left(-\frac{ (x+\mu)^T\Sigma^{-1}(x+\mu) }{2}\right)} \right) \\
		&\geq  \log\left( \frac{1}{ \sqrt{|2\pi\Sigma|} } \right) 
		+ \log\left( \exp\left((1-q)\times  \left(-\frac{ (x-\mu)^T\Sigma^{-1}(x-\mu) }{2}\right) + q\times \left(-\frac{ (x+\mu)^T\Sigma^{-1}(x+\mu) }{2}\right)\right) \right) \\ 
		& =  \log\left( \frac{1}{ \sqrt{|2\pi\Sigma|} } \right) 
		+ \left((1-q)\times  \left(-\frac{(x-\mu)^T\Sigma^{-1}(x-\mu)}{2}\right) + q\times \left(-\frac{(x+\mu)^T\Sigma^{-1}(x+\mu)}{2}\right)\right)\\
		& = \log\left( \frac{1}{ \sqrt{|2\pi\Sigma|} } \right) 
		+ \left( -\frac{1}{2}x^T\Sigma^{-1}x - \frac{1}{2}\mu^T\Sigma^{-1}\mu + (1-2q)x^T\Sigma^{-1}\mu \right).
	\end{align*}
	Since 
	\begin{align*}
		\int_\mbR p_X(x)
		\left( \log\left( \frac{1}{ \sqrt{|2\pi\Sigma|} } \right) 
		-  \frac{1}{2}\mu^T\Sigma^{-1}\mu \right) dx
		= \log\left( \frac{1}{ \sqrt{|2\pi\Sigma|} } \right) 
		-  \frac{1}{2}\mu^T\Sigma^{-1}\mu,
	\end{align*}
	\begin{align*}
		\int_{\mbR^n} p_X(x) (1-2q)x^T\Sigma^{-1}\mu dx
		& = (1-q)(1-2q)\mu^T\Sigma^{-1}\mu + q (1-2q)(-\mu)^T\Sigma^{-1}\mu \\
		& = (1-2q)^2 \mu^T\Sigma^{-1}\mu,
	\end{align*}
	and 
	\begin{align*}
		\int_{\mbR^n} p_X(x) \left(-\frac{1}{2}x^T\Sigma^{-1}x \right)dx
		= - \frac{1}{2}\left( n + \mu^T\Sigma^{-1}\mu \right),
	\end{align*}
	where we used Lemma \ref{Lem:QuadraticFormExpectation}, then we have from \eqref{Eq:XEntropy}
	\begin{align}\label{Eq:XEntropyUb}
		h(X) 
		& \leq (-1) \int_{\mbR^n} p_X(x) 
		+ \left(\log\left( \frac{1}{ \sqrt{|2\pi\Sigma|} } \right)  -\frac{1}{2}x^T\Sigma^{-1}x - \frac{1}{2}\mu^T\Sigma^{-1}\mu + (1-2q)x^T\Sigma^{-1}\mu \right) \nonumber\\
		& =(-1) \left( \log\left( \frac{1}{ \sqrt{|2\pi\Sigma|} } \right) 
		-  \frac{1}{2}\mu^T\Sigma^{-1}\mu + (1-2q)^2 \mu^T\Sigma^{-1}\mu - \frac{1}{2}\left( n + \mu^T\Sigma^{-1}\mu \right) \right) \nonumber\\
		& = \frac{1}{2}\log|2\pi e\Sigma| + 4q(1-q)\mu^T\Sigma^{-1} \mu.
	\end{align}
	Thus, combining \eqref{Defn:MIinBinaryClassificationModel}, \eqref{Eq:ConditionalEntropy_XcY}, and \eqref{Eq:XEntropyUb}, we can get 
	\begin{align}
		I(X;Y) 
		& \leq \frac{1}{2}\log|2\pi e\Sigma| + 4q(1-q)\mu^T\Sigma^{-1} \mu -  \frac{1}{2}\log|2\pi e\Sigma| \\
		& = 4q(1-q)\mu^T\Sigma^{-1} \mu.
	\end{align}
	
	We now derive the lower bound of $I(X;Y)$. Since
	\begin{align*}
		& \log\left( 
		\frac{(1-q)}{\sqrt{|2\pi\Sigma|}} \exp\left(-\frac{ (x-\mu)^T\Sigma^{-1}(x-\mu) }{2}\right)
		+  \frac{q}{\sqrt{|2\pi\Sigma|}} \exp\left(-\frac{(x+\mu)^T\Sigma^{-1}(x+\mu)}{2}\right)
		\right) \\
		& =  \log\left( \frac{1}{\sqrt{|2\pi\Sigma|}} \right) 
		+ \log\left( {(1-q) \times  \exp\left(-\frac{ (x-\mu)^T\Sigma^{-1}(x-\mu) }{2}\right)
			+ q \times  \exp\left(-\frac{ (x+\mu)^T\Sigma^{-1}(x+\mu) }{2}\right)} \right) \\
		& \leq \log\left( \frac{1}{\sqrt{|2\pi\Sigma|}} \right) 
		+ \log\left( \max\left(\exp\left(-\frac{ (x-\mu)^T\Sigma^{-1}(x-\mu) }{2}\right),  \exp\left(-\frac{ (x+\mu)^T\Sigma^{-1}(x+\mu) }{2}\right)\right) \right) \\
		& = \log\left( \frac{1}{\sqrt{|2\pi\Sigma|}} \right) 
		+ \log\left( \exp\left(\max\left(-\frac{ (x-\mu)^T\Sigma^{-1}(x-\mu) }{2},-\frac{ (x+\mu)^T\Sigma^{-1}(x+\mu) }{2}\right)\right) \right) \\
		& = \log\left( \frac{1}{\sqrt{|2\pi\Sigma|}} \right)  + \max\left(-\frac{ (x-\mu)^T\Sigma^{-1}(x-\mu) }{2},-\frac{ (x+\mu)^T\Sigma^{-1}(x+\mu) }{2}\right)
	\end{align*}
	then
	\begin{align}\label{Eq:DifferentialEntropy3}
		h(X)
		& \geq (-1)\times \int_\mbR p_X(x)\left( \log\left( \frac{1}{\sqrt{|2\pi\Sigma|}} \right)  + \max\left(-\frac{ (x-\mu)^T\Sigma^{-1}(x-\mu) }{2},-\frac{ (x+\mu)^T\Sigma^{-1}(x+\mu) }{2}\right) \right) dx \nonumber \\
		& = - \log\left( \frac{1}{\sqrt{|2\pi\Sigma|}} \right)   +  \min\left(I_-,I_+\right),
	\end{align}
	where
	\begin{align*}
		I_-
		&:=\int_{\mbR^n} p_X(x) \frac{ (x-\mu)^T\Sigma^{-1}(x-\mu) }{2} dx \\
		&=\int_\mbR 	\left(\frac{(1-q)}{\sqrt{|2\pi\Sigma|}} \exp\left(-\frac{ (x-\mu)^T\Sigma^{-1}(x-\mu) }{2}\right)
		+  \frac{q}{\sqrt{|2\pi\Sigma|}} \exp\left(-\frac{(x+\mu)^T\Sigma^{-1}(x+\mu)}{2}\right)\right)\\
		&\quad 	\times \frac{ (x-\mu)^T\Sigma^{-1}(x-\mu) }{2} dx
	\end{align*}
	and 
	\begin{align*}
		I_+
		&:=
		\int_{\mbR^n} p_X(x) \frac{ (x+\mu)^T\Sigma^{-1}(x+\mu) }{2} dx \\
		& =\int_\mbR 	\left(\frac{(1-q)}{\sqrt{|2\pi\Sigma|}} \exp\left(-\frac{ (x-\mu)^T\Sigma^{-1}(x-\mu) }{2}\right)
		+  \frac{q}{\sqrt{|2\pi\Sigma|}} \exp\left(-\frac{(x+\mu)^T\Sigma^{-1}(x+\mu)}{2}\right)\right)\\
		&\quad \times \frac{ (x+\mu)^T\Sigma^{-1}(x+\mu) }{2} dx.
	\end{align*}
	
	From Lemma \ref{Lem:QuadraticFormExpectation}, we have 
	\begin{align}\label{Eq:I_Minus}
		L_- 
		& = \frac{(1-q)n+q(n+4\mu^T\Sigma^{-1}\mu)}{2} \nonumber\\
		& = \frac{n}{2}+2q\mu^T\Sigma^{-1}\mu,
	\end{align}
	and 
	\begin{align}\label{Eq:I_Plus}
		I_+
		& = \frac{(1-q)(n+4\mu^T\Sigma^{-1}\mu) + qn}{2} \nonumber\\
		& = \frac{n}{2} + 2(1-q)\mu^T\Sigma^{-1}\mu.
	\end{align}
	
	Combining  \eqref{Eq:DifferentialEntropy3},  \eqref{Eq:I_Plus}, and \eqref{Eq:I_Minus}, we have 
	\begin{align}
		h(X) \geq  - \log\left( \frac{1}{\sqrt{|2\pi\Sigma|}} \right) + \frac{n}{2} + 2\min(q,1-q)\mu^T\Sigma^{-1}\mu.
	\end{align}
	Then from \eqref{Defn:MIinBinaryClassificationModel}, we have 
	\begin{align*}
		I(X;Y) \geq 2\min(q,1-q)\mu^T\Sigma^{-1}\mu.
	\end{align*}
	
\end{proof}

\subsection*{Appendix B1: Ignored Label Conditional Entropy in CIFAR-10 Dataset}\label{Sec:ImgNt_IgnrLblCE}

In this section, we give examples in Figure \ref{Fig:CIFAR10_samples} from CIFAR-10 to show the information loss during the annotation process \cite{he_deep_2015}.

\begin{figure*}[!htb]
	\centering
	\begin{subfigure}[b]{0.24\textwidth}
		\centering
		\includegraphics[width=\linewidth]{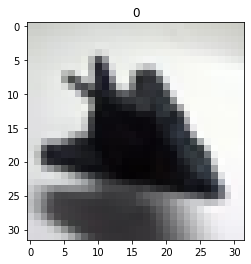}
		\caption{Truth label: class 0 airplane}
	\end{subfigure}%
	~
	\begin{subfigure}[b]{0.24\textwidth}
		\centering
		\includegraphics[width=\linewidth]{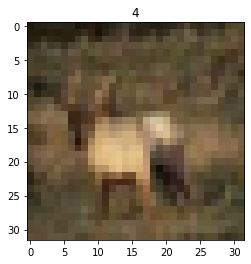}
		\caption{Truth label: class 4 dear}
	\end{subfigure}
	\begin{subfigure}[b]{0.24\textwidth}
		\centering
		\includegraphics[width=\linewidth]{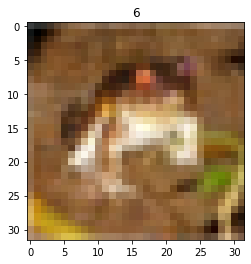}
		\caption{Truth label: class 6 frog}
	\end{subfigure}%
	~
	\begin{subfigure}[b]{0.24\textwidth}
		\centering
		\includegraphics[width=\linewidth]{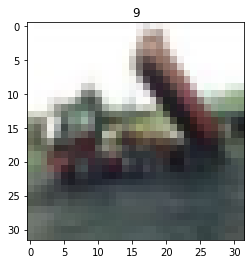}
		\caption{Truth label: class 9 truck}
	\end{subfigure}
	\caption{Examples from CIFAR-10 dataset. When such images are given, it is usually challenging to determine the label class with 100\% certainty. When we use a single label during annotation process and the one-hot encoding during the training process, the extra information left for the labels are ignord.}\label{Fig:CIFAR10_samples}
\end{figure*}

\newpage
\bibliography{202204_MILC.bib}
\bibliographystyle{unsrt}

\end{document}